\documentclass[11pt, notitlepage]{article}

\usepackage[utf8]{inputenc}

\usepackage{amsthm,graphicx,placeins}
\usepackage[linesnumbered,ruled,norelsize]{algorithm2e}

\usepackage[dvipsnames]{xcolor}
\usepackage{amsmath}
\usepackage{amssymb}
\usepackage[toc,page]{appendix}
\usepackage{array}
\usepackage[noblocks]{authblk}
\usepackage[english]{babel}
\usepackage{bbm}
\usepackage{booktabs}
\usepackage[font=small]{caption}
\usepackage{changepage}
\usepackage{color,colortbl}
\usepackage{csquotes}
\usepackage{endnotes}
\usepackage{bm}
\usepackage{enumerate}
\usepackage{tikz}
\usetikzlibrary{arrows.meta}

\usepackage{epsfig}
\usepackage{float}
\usepackage[T1]{fontenc}
\usepackage[hang,flushmargin]{footmisc}
\usepackage{graphicx}
\usepackage{subcaption} 
\usepackage{epstopdf}
\usepackage[left=1in,right=1in,top=1in,bottom=1in,footnotesep=0.8cm]{geometry}
\usepackage{hyperref}
\hypersetup{colorlinks,
	linkcolor=blue,
	filecolor=blue,
	urlcolor=black,
	citecolor=blue}
\allowdisplaybreaks
\usepackage{cleveref}

\usepackage{lipsum}
\usepackage{longtable}
\usepackage{makecell}%
\usepackage{mathrsfs}
\usepackage{mathtools}

\usepackage{multirow}
\usepackage{mwe}
\usepackage[round]{natbib}
\usepackage{relsize}

\usepackage{romanbar}
\usepackage{setspace}
\usepackage{supertabular}
\usepackage{textcomp}
\usepackage{thmtools,thm-restate}
\usepackage{titlesec}
\usepackage{enumerate}

\usepackage{makecell}

\allowdisplaybreaks

\DeclareMathOperator*{\argmax}{arg\,max}

\newcommand{\E}{\mathbb{E}}
\newcommand{\R}{\mathbb{R}}

\newtheorem{assumption}{Assumption}

\newtheorem{lemma}{Lemma}[section]
\newtheorem{proposition}{Proposition}[section]
\newtheorem{theorem}{Theorem}[section]

\newtheorem{remark}{Remark}
\newtheorem{definition}{Definition}

\titleformat*{\section}{\Large\bfseries}
\titleformat*{\subsection}{\large      \bfseries}
\titleformat*{\subsubsection}{\normalsize\bfseries}
\titleformat*{\paragraph}{\normalsize\bfseries}
\titleformat*{\subparagraph}{\normalsize\bfseries}
\titlespacing{\paragraph}{0pt}{0.3\baselineskip}{0.5em}

\allowdisplaybreaks

\title{\huge {Faster Reinforcement Learning by Freezing Slow States}}
\author{{\Large Yijia Wang$^1$ and Daniel R. Jiang$^{2,1}$\\
$^1$University of Pittsburgh, $^2$AI at Meta}}
\date{}

\begin{document}
\maketitle

\begin{abstract}
    We study infinite horizon Markov decision processes (MDPs) with \emph{fast-slow} structure, where some state variables evolve rapidly (``fast states'') while others change more gradually (``slow states''). This structure commonly arises in practice when decisions must be made at high frequencies over long horizons, and where slowly changing information still plays a critical role in determining optimal actions. Examples include inventory control under slowly changing demand indicators or dynamic pricing with gradually shifting consumer behavior. Modeling the problem at the natural decision frequency leads to MDPs with discount factors close to one, making them computationally challenging. We propose a novel approximation strategy that ``freezes'' slow states during phases of lower-level planning and subsequently applies value iteration to an auxiliary upper-level MDP that evolves on a slower timescale. Freezing states for short periods of time leads to easier-to-solve lower-level problems, while a slower upper-level timescale allows for a more favorable discount factor. On the theoretical side, we analyze the regret incurred by our frozen-state approach, which leads to simple insights on how to trade off regret versus computational cost. Empirically, we benchmark our new frozen-state methods on three domains, (i) inventory control with fixed order costs, (ii) a gridworld problem with spatial tasks, and (iii) dynamic pricing with reference-price effects. We demonstrate that the new methods produce high-quality policies with significantly less computation, and we show that simply omitting slow states is often a poor heuristic.

\end{abstract}

\section{Introduction}

Markov decision processes (MDPs) offer a foundational framework for modeling sequential decision making under uncertainty. In many real-world applications, decisions must be made at high frequencies over extended time horizons. From a modeling standpoint, this translates into MDPs with discount factors close to one (which captures the long-term nature of the objective). However, such long-horizon problems can pose computational challenges: standard algorithms like value iteration (VI), which is a core building block in many reinforcement learning (RL) algorithms, rely on the contraction property of the Bellman operator to guarantee convergence. As the discount factor approaches one, this contraction weakens, leading to slow convergence and degraded performance within a given computational budget \citep{bertsekas1996neuro, jiang2015dependence}.

In a wide range of applications, state variables naturally evolve at different timescales. For example, in inventory control problems, the inventory level may fluctuate rapidly from one period to the next, while exogenous factors that influence demand may evolve more gradually. In dynamic pricing problems, the (endogenous) price history can change frequently, but user behavior shifts more slowly. We refer to this as \emph{fast-slow} structure, where the state decomposes into fast-evolving and slow-evolving components. Although slow states may evolve gradually, they often encode critical contextual information, and omitting them can result in suboptimal policies and high regret.

Solving MDPs that fully capture both fast and slow dynamics at the natural decision frequency typically requires long effective horizons, which can pose serious computational burdens. To manage complexity, practitioners often resort to simplifying assumptions, such as fixing or omitting slow states altogether. While this reduces computational burden, it can degrade policy quality by discarding important information. This paper introduces and formalizes the notion of \emph{fast-slow MDPs}, and studies the computational implications of a novel approximation strategy called \emph{frozen-state value iteration} (FSVI).

\subsection{Main Contributions}
In this paper, we propose an approximation strategy that uses two levels of planning, aiming to provide a middle ground between solving the full MDP and completely ignoring slow states. In our approach, we first temporarily ``freeze'' the slow state during \emph{lower-level} planning. Concretely, we compute policies based on finite-horizon MDPs that are conditioned on a fixed slow state. Subsequently, we apply value iteration to an auxiliary \emph{upper-level} MDP that evolves on a slower timescale by leveraging the lower-level policies. This approach makes the lower-level problems more tractable, while the slower upper-level dynamics enable the use of a more favorable (i.e., smaller) effective discount factor.  
Specifically, we make the following contributions:
\begin{enumerate}
    \item We first formally define a fast-slow MDP and provide an (exact) reformulation into an MDP with hierarchical structure. The \emph{upper level} is a slow-timescale infinite horizon MDP and the \emph{lower level} is a fast-timescale finite horizon MDP with $T$ periods. One period of the upper-level problem is composed of a complete lower-level problem. 
    \item We propose a \emph{frozen-state approximation} to the reformulated MDP, along with an associated \emph{frozen-state value iteration} (FSVI) algorithm, where the slow state is frozen in the lower-level problem, while each period in the upper-level problem ``releases'' the slow state. Computational benefits arise in several ways: (1) frozen states simplify the dynamics of the lower-level MDP (fewer successor states), (2) the discount factor in the upper-level problem is more favorable due to the reuse of the lower-level policy, which is computed only once, (3) the lower-level MDP thus becomes separable into independent MDPs, opening the door to speedups via parallel computation. Solving the frozen-state approximation gives a policy that switches between one action from upper-level policy and $T-1$ actions from the lower-level policy. We give a theoretical analysis that upper bounds the expected regret from applying this policy compared to the optimal policy.

    \item We adapt our frozen-state approach to a more practical reinforcement learning setting where a generative model (or simulator) is available but the MDP is not fully known. Specifically, we develop frozen-state fitted value iteration (FSFVI), a version of fitted value iteration \citep{munos2008finite} that incorporates our two-level planning strategy. FSFVI can learn effective policies from sampled transitions, making our method applicable to model-free or simulation-based environments. We also provide a theoretical regret bound for a special case of FSFVI using linear function approximation (see Appendix \ref{sec:avi}).

    \item Lastly, we perform a systematic empirical study on three problem settings: (1) inventory control with fixed costs, (2) gridworld with spatial tasks, and (3) dynamic pricing with reference effects. We show that our proposed algorithms (based on the frozen-state approximation) quickly converge to good policies using significantly less computation compared to standard methods. Notably, our results show that ignoring the slow state, i.e., modeling the system as if the true state variable only includes the fast component,\footnote{See Appendix \ref{sec:app:ignore} for details on how we simulate this behavior in our experiments.} leads to poor results.
\end{enumerate}

\section{Related Work}

\looseness-1 In this section, we provide a brief review of related literature. 
First, there exists a stream of literature focused on sequential decision making problems with \emph{exact} hierarchical, multi-timescale structure.
\citet{chang2003multitime} study multi-timescale MDPs, which are composed of $M$ different decisions that are made on $M$ different discrete timescales. The authors consider the impact of upper-level states and actions on the transition of the lower levels, an idea that is also present in our fast-slow framework. Multi-timescale MDPs have often been applied in supply chain problems, including production planning in semiconductor fabrication \citep{panigrahi2004hierarchical, bhatnagar2006actor}, hydropower portfolio management \citep{zhu2006hydropower}, and strategic network growth for reverse supply chains \citep{wongthatsanekorn2010multi}.
\citet{wang2018demand} propose a row-generation-based algorithm to solve a linear programming formulation of the multi-timescale MDP. 
\citet{jacobson1999piecewise} consider ``piecewise stationary'' MDPs, where the transition and reward functions are ``renewed'' every $N+1$ periods, motivated by problems where routine decisions are periodically interrupted by higher-level decisions. For the case of large renewal periods, they propose a policy called the ``initially stationary policy'' which uses a fixed decision rule for some number of initial periods in each renewal cycle. 
Our fast-slow model focuses on a novel fast-slow structure present in many MDPs and unlike the above work, \emph{does not assume} any natural/exact hierarchical structure. Instead, we focus on how a particular type of (frozen-state) hierarchical structure can be used as an \emph{approximation} to the true MDP. However, we note that many MDPs with natural two-timescale structure can also fit into our framework, and therefore, given that perspective, our model can be viewed as a generalization.

Our proposed frozen-state algorithms are also related to literature on \emph{hierarchical reinforcement learning}, which are methods that artificially decompose a complex problem into smaller sub-problems \citep{barto2003recent}. 
Approaches include the options framework \citep{sutton1999between}, the hierarchies of abstract machines (HAMs) approach \citep{parr1998hierarchical}, and MAXQ value function decomposition \citep{dietterich2000hierarchical}. 
 
Out of these three approaches, the options framework is most closely related to this paper.
A \emph{Markov option} (also called a \emph{macro-action} or \emph{temporally extended action}) is composed of a policy,
a termination condition, %
and an initiation set %
\citep{sutton1999between, precup2000temporal}. 
One of the biggest challenges is to automatically construct options that can effectively speed up reinforcement learning. 
A large portion of work in this direction is based on \emph{subgoals}, states that might be beneficial to reach \citep{digney1998learning, mcgovern2001automatic, jonsson2005causal, ciosek2015value, wang2022subgoal}. 
The subgoals are identified by utilizing the learned model of the environment \citep{menache2002q, mannor2004dynamic, csimcsek2004using, csimcsek2005identifying}, or through trajectories without learning a model \citep{mcgovern2001automatic, stolle2002learning}.

The options (and subgoals) framework is largely motivated by robotics and navigation-related tasks, while we are particularly interested in solving problems that arise in the operations research and operations management domains. The problems that we study do not decompose naturally into ``subgoals,'' leading us to identify and focus on the fast-slow structure, which does indeed arise naturally for many problems of interest.

Another work that is related to the options framework is \citet{song2020temporal}, who divide finite-horizon MDPs into two sub-problems along the time horizon, and concatenate their optimal solutions to generate an overall solution. Our paper also, in a sense, divides MDPs along the time horizon, but our work is quite different from that of \citet{song2020temporal} in that we work on \emph{infinite horizon problems} and convert them into auxiliary problems that operate on a \emph{slower timescale}, which takes advantage of reusable lower-level policies. More importantly, the various methods we propose all build on the idea of freezing certain states to reduce computational cost, which is unique to our approach and, to our knowledge, this is a novel direction that has not been proposed before.

Finally, our work contributes to the broader literature on leveraging MDP structure to improve reinforcement learning and approximate dynamic programming algorithms. Prior research has exploited structural properties of the value function, such as convexity \citep{pereira1991multi,philpott2008convergence,nascimento2009optimal,benjaafar2022dynamic} and monotonicity \citep{papadaki2002exploiting,jiang2015approximate}, as well as structural features of the optimal policy \citep{kunnumkal2008using}. Additional examples include factored MDPs \citep{osband2014near} and weakly-coupled systems \citep{killian2021q,el2023weakly,brown2025unit,nadarajah2025self}. Our contribution to this line of work lies in identifying a new form of MDP structure, which then informs the design of a new algorithmic approach.

\section{Fast-Slow MDPs} \label{sec:GeneralModel}
In this section, we introduce the \emph{base model}, the original MDP to be solved and formally introduce the notion of a \emph{fast-slow} MDP. We then provide a hierarchical reformulation of the base model using fixed-horizon policies, and show the equivalence (in optimal value) between the two models.
Due to the considerable amount of notation in this paper, we provide a table of notation in Table \ref{tab:notation} for reference.

\begin{table}[hb!]
\centering
{\footnotesize
\vspace{10pt}
\begin{tabular}{@{}ll@{}}
\toprule
\textbf{Symbol} & \textbf{Description} \\ \midrule
$\mathcal{S}$ & State space, with generic element $s \in \mathcal S$ \\[0.5em]
$\mathcal{X}$ & Slow state space, with generic element $x \in \mathcal X$ \\[0.5em]
$\mathcal{Y}$ & Fast state space, with generic element $y \in \mathcal Y$ \\[0.5em]
$\mathcal{A}$ & Action space, with generic element $a \in \mathcal A$ \\[0.5em]
$\mathcal{W}$ & Exogenous noise space, with generic element $w \in \mathcal W$ \\[0.5em]
$f$ & Transition function: $f:\mathcal{S}\times\mathcal{A}\times\mathcal{W}\to\mathcal{S}$ \\[0.5em]
$f_\mathcal{X}$ & Slow state transition: $f_\mathcal{X}:\mathcal{S}\times\mathcal{A}\times\mathcal{W}\to\mathcal{X}$ \\[0.5em]
$f_\mathcal{Y}$ & Fast state transition: $f_\mathcal{Y}:\mathcal{S}\times\mathcal{A}\times\mathcal{W}\to\mathcal{Y}$ \\[0.5em]
$r$ & Reward function: $r:\mathcal{S}\times\mathcal{A}\to [0, r_{\max}]$ \\[0.5em]
$\gamma$ & Discount factor \\[0.5em]
$\alpha$ & Constant in the fast-slow property controlling changes in slow state \\[0.5em]
$d_\mathcal{Y}$ & Bound on the one-step change in the fast state \\[0.5em]
$T$ & Number of periods in the lower-level problem \\[0.5em]
$U^*$ & Optimal value function of the base MDP \\[0.5em]
$\nu^*$ & The optimal policy for the base MDP \\ [0.5em]
$L_r$, $L_f$, $L_U$ & Lipschitz constants for $r$, $f$, and $U^*$ \\ [0.5em]
$(\mu, \boldsymbol{\pi})$ & Generic $T$-period policy with upper-level policy $\mu$ and lower-level policy $\boldsymbol{\pi} = (\pi_1,\dots,\pi_{T-1})$ \\[0.5em]
$R(s_0, \mu(s_0), \boldsymbol{\pi})$ & $T$-period reward of hierarchical reformulation associated with $T$-period policy $(\mu, \boldsymbol{\pi})$\\[0.5em]
$(\mu^*,\boldsymbol{\pi}^*)$ & The optimal policy for the hierarchical reformulation; has same value as $\nu^*$\\ [0.5em]
$\tilde{R}(s_0,a,J_1)$ & Approximation to $T$-period reward using lower-level value function approximation $J_1$ \\[0.5em] 
$J_t^*$, $V^k$ & Lower- and upper-level value function approximation used in FSVI \\[0.5em] 
$(\tilde{\mu}^k, \tilde{\boldsymbol{\pi}}^*)$ & Output policy of FSVI after $k$ iterations\\ [0.5em]
$U^i$ & Value function approximation used in base VI \\[0.5em] 
$\nu^k$ & Output policy of base VI after $k$ iterations \\ [0.5em]

$H$ & Bellman operator for the base MDP \\ [0.5em]

$\tilde{H}$ & Bellman operator for lower-level frozen-state problem (discount factor $\gamma$) \\ [0.5em]

$F_{J_1, \boldsymbol{\pi}}$ & Bellman operator for  upper-level problem given lower-level $J_1$ and $\boldsymbol{\pi}$ (discount factor $\gamma^T$)\\ [0.5em]

\bottomrule
\end{tabular}
}
\caption{Frequently used notation in the fast-slow MDP framework.}
\label{tab:notation}
\end{table}

\subsection{Base Model}

Consider a discrete-time MDP $\langle \mathcal{S}, \mathcal{A}, \mathcal{W}, f, r, \gamma \rangle$, where $\mathcal{S}$ is the finite state space, $\mathcal{A}$ is the finite action space, $\mathcal{W}$ is the space of possible realizations of an exogenous, independent and identically distributed (i.i.d.) noise process $\{ w_t \}$ defined on a discrete probability space $(\Omega, \mathcal F, \mathbb P)$, $f: \mathcal{S}\times \mathcal{A} \times \mathcal{W} \rightarrow \mathcal{S}$ is the transition function, $r: \mathcal{S}\times\mathcal{A} \rightarrow [0,r_\textnormal{max}] $ is the bounded reward function, and $\gamma\in[0, 1)$ is the discount factor for future rewards \citep{puterman2014markov}. The objective is
\begin{equation} \label{eq:optimal_value}
    U^*(s) = \max_{\{\nu_t\}} \; \E \left[ \sum_{t=0}^{\infty} \gamma^t \, r\bigl(s_t,\nu_t(s_t)\bigr) \, \Bigr| \, s_0 = s \right],
\end{equation}
where states transition according to $s_{t+1} = f(s_t, a_t, w_{t+1})$ and we optimize over sequences of policies $\nu_t : \mathcal S \rightarrow \mathcal A$, which are deterministic mappings from states to actions. The expectation is taken over exogenous noise process $\{w_t\}_{t=1}^\infty$. We assume throughout that $\mathcal S$, $\mathcal A$, $\mathcal X$, $\mathcal Y$, $\mathcal S \times \mathcal A$, and $\mathcal X \times \mathcal Y$ are equipped with the Euclidean metric,\footnote{However, as long as the relevant spaces are metric spaces, the framework continues to hold. We choose Euclidean metrics as they are natural for our applications.} which is naturally the case for many applications. 

\begin{assumption}[Separability and the Fast-Slow Property]
Suppose the following hold:
\begin{enumerate}[(i)]
    \item The state space $\mathcal{S}$ is separable and can be written as $\mathcal S = \mathcal{X}\times\mathcal{Y}$. We call $\mathcal X$ the ``slow state space'' and $\mathcal Y$ the ``fast state space.''
    
    \item Let $s_{t} = (x_t, y_t) \in \mathcal S$, where $x_t \in \mathcal X$ is the slow state and $y_t \in \mathcal Y$ the fast state, $a_t \in \mathcal A$, and $w_{t+1} \in \mathcal W$. The transition dynamics $s_{t+1} = f(s_t, a_t, w_{t+1}) \in \mathcal S$ can be written with the notation:
    \[
    x_{t+1} = f_\mathcal{X}(x_t,y_t,a_t, w_{t+1}) \in \mathcal X \quad \text{and} \quad y_{t+1} = f_\mathcal{Y}(x_t, y_t, a_t, w_{t+1}) \in \mathcal Y,
    \]
    for some $f_\mathcal{X}: \mathcal{S}\times\mathcal{A} \times \mathcal{W} \rightarrow \mathcal{X}$ and $f_\mathcal{Y}: \mathcal{S}\times \mathcal{A}\times \mathcal{W} \rightarrow \mathcal{Y}$.
    
    \item For any state $(x,y) \in \mathcal S$, action $a \in \mathcal A$, and exogenous noise $w \in \mathcal W$, 
    suppose the one-step transitions of $x$ and $y$ satisfy:
    \begin{equation*}
        \bigl\|y-f_\mathcal{Y}(x,y,a,w)\bigr\|_2 \leq d_\mathcal{Y} \quad \text{and} \quad \bigl \|x-f_\mathcal{X}(x,y,a,w) \bigr\|_2 \leq \alpha d_\mathcal{Y},
    \end{equation*}
    for some $d_\mathcal{Y} < \infty$ and $\alpha \in [0, 1]$.

\end{enumerate}
\label{assumption.fast_slow_transition}
\end{assumption}

\begin{remark}
Note that one particularly instructive example is the case of exogenous slow states, where $x_{t+1} = f_\mathcal{X}(x_t, w_{t+1})$. Here, the transition does not depend on the action $a_t$, nor does it depend on the fast state $y_t$. Such a model is common in practice: examples of exogenous slow states include prices, weather conditions, and other environmental variables that are not influenced by the decision maker's actions or the states of the primary system. See, e.g., \cite{yu2009arbitrarily}, who study a related model called the ``arbitrarily modulated MDP.''
\end{remark}

\begin{assumption}[Lipschitz Properties]
\label{assumption:lipschitz}
Suppose that the reward function $r$, transition function $f$, and optimal value function $U^*$ are Lipschitz with respect to $\| \cdot \|_2$:
\begin{align}
|r(s,a) - r(s',a')| &\le L_r \,\bigl \|(s,a) - (s',a') \bigr\|_2, \label{assumption.lipschitz_reward}\\
 \bigl \|f(s,a,w) - f(s',a',w)\bigr\|_2 &\le L_f \,\bigl \|(s,a) - (s',a') \bigr\|_2\label{eq.lipschitz_f},\\
  \bigl \|U^*(s) - U^*(s') \bigl\|_2 &\le L_U \bigl \|s-s'\bigr\|_2\label{eq.lipschitz_Ustar},
\end{align}
for some Lipschitz constants $L_r$, $L_f$, and $L_U$. Lipschitz assumptions are common in the literature; see, for example, \cite{ok2018exploration}, \cite{domingues2021kernel}, \cite{sinclair2020adaptive}, and \cite{sinclair2022adaptive}. In Appendix \ref{appendix:bounds_on_LU}, we give bounds on $L_U$ in terms of $L_r$ and $L_f$. While we could have used those results directly and omitted the assumption on $L_U$, we opt to include \eqref{eq.lipschitz_Ustar} to increase the clarity of our results.
\end{assumption}

\begin{definition}[Fast-Slow MDP]
An MDP $\langle \mathcal{S}, \mathcal{A}, \mathcal{W}, f, r, \gamma \rangle$ is called a $(\alpha, d_\mathcal{Y})$-fast-slow MDP if Assumptions~\ref{assumption.fast_slow_transition} and \ref{assumption:lipschitz} are satisfied. 
\end{definition}

Given any state $(x,y)$, noise $w$, and policy $\nu$, we use the notation $f^\nu(x,y,w) = f(x,y,\nu(x,y), w)$, $f_\mathcal{X}^{\nu}(x,y,w) = f_\mathcal{X}\bigl(x,y,\nu(x,y),w\bigr)$, $f_\mathcal{Y}^{\nu}(x,y,w) = f_\mathcal{Y}\bigl(x,y,\nu(x,y),w\bigr)$, and $r(x,y,\nu) = r(x,y,\nu(x,y))$ throughout the paper.
The value of a stationary policy\footnote{It is well-known that there exists an optimal policy to (\ref{eq:optimal_value}) that is both stationary and deterministic. See \cite{puterman2014markov}.} $\nu$ at state $(x,y)$ is the expected cumulative reward starting from state $(x,y)$ following policy $\nu$, i.e., 
\begin{equation*}\label{eq.Bellman_original_policy}
    U^{\nu}(x,y) = \E \left[  \sum_{t=0}^{\infty} \gamma^t r\bigl(x_t,y_t,\nu \bigr) \, \Bigr| \, (x_0, y_0) = (x, y) \right] = r\bigl(x,y,\nu\bigr) + \gamma\, \E\bigl[U^{\nu}(x',y')\bigr],
\end{equation*}
where $(x', y') = f^\nu(x,y,w)$ and $(x_{t+1}, y_{t+1}) = f^{\nu}(x_t,y_t,w_t)$ for all $t$.
The optimal value function at state $U^*(x,y)$, as defined in (\ref{eq:optimal_value}), satisfies the Bellman equation, i.e., 
\begin{equation}\label{eq.Bellman_original}
    U^*(x,y) = \max_{a} \; r(x,y,a) + \gamma\, \E\bigl[U^*(x',y')\bigr].
\end{equation}
Denote by $H$ the Bellman operator of the base model; for any state $(x,y)$ and value function $U$,
\begin{equation} \label{eq.Vbase_operator}
    (H U) (x,y) = \max_{a} \, r(x,y,a) + \gamma \, \E \bigl[U(f(x,y,a,w))\bigr].
\end{equation}
A policy that is greedy with respect to the optimal value function, i.e., 
\[
\nu^*(x,y) = \argmax_{a} \; r(x,y,a) + \gamma \, \E \bigl[U^*(x',y')\bigr].
\]
is an optimal policy, and the optimal value $U^*$ and the value of the optimal policy $U^{\nu^*}$ are the same.

\subsection{Hierarchical Reformulation using Fixed-Horizon Policies} \label{sec:original_groupT}
In this section, we derive an exact hierarchical reformulation with the original timescale broken up into groups of $T$ periods each. The reformulation holds for a general MDP $\langle \mathcal{S}, \mathcal{A}, \mathcal{W}, f, r, \gamma \rangle$, but the concepts that we introduce in this section will serve as the basis for developing our frozen-state computational approach for fast-slow MDPs.

Denote $(\mu, \boldsymbol{\pi})$ a \emph{$T$-horizon policy}, which is a sequence of $T$ policies $(\mu,\pi_1,\ldots,\pi_{T-1})$, $\mu:\mathcal{S} \rightarrow \mathcal{A}$, $\pi_t:\mathcal{S} \rightarrow \mathcal{A}$ and $\boldsymbol \pi = (\pi_1,\ldots,\pi_{T-1})$.
Following $(\mu, \boldsymbol{\pi})$ means that we take the first action according to $\mu$ and then next $T-1$ actions according to $\boldsymbol{\pi}$. Given any state $s_0$, the $T$-period reward function (of the base model) associated with $(\mu, \boldsymbol{\pi})$ is written as:
\begin{equation} \label{eq.r_T_original}
    R(s_0, \mu(s_0), \boldsymbol{\pi}) = r(s_0,\mu) + \sum_{t=1}^{T-1} \gamma^t \, r(s_t,\pi_t),
\end{equation}
where $s_1 = f^{\mu}(s_0,w_1)$ and $s_{t+1} = f^{\pi_{t}}(s_t,w_{t+1})$ for $t > 0$.

A \emph{$T$-periodic policy} $(\mu, \boldsymbol{\pi})$ refers to the infinite sequence that repeatedly applies the $T$-horizon policy $(\mu, \boldsymbol{\pi})$, i.e., $(\mu, \boldsymbol{\pi}, \mu, \boldsymbol{\pi}, \ldots)$. Note that despite it not being a stationary policy, the $T$-periodic policy $(\mu, \boldsymbol{\pi})$ can be implemented in the infinite horizon problem defined in (\ref{eq:optimal_value}). The value of the $T$-periodic policy $(\mu, \boldsymbol{\pi})$ at state $s_0$ is
\begin{equation*} \label{eq.Bellman_T_original_policy}
    \bar{U}^{\mu}(s_0, \boldsymbol{\pi}) 
    =\E\left[ \sum_{k=0}^{\infty} \gamma^{kT} R(s_k, \mu(s_k), \boldsymbol{\pi}) \, \Bigr| \, s_0=s \right]
    =\E\bigl[ R(s_0, \mu(s_0), \boldsymbol{\pi}) + \gamma^T \, \bar{U}^{\mu}(s_T, \boldsymbol{\pi}) \bigr],
\end{equation*}
where, again, $s_1 = f^{\mu}(s_0,w_1)$ and $s_{t+1} = f^{\pi_{t}}(s_t,w_{t+1})$ for $t>0$ within each cycle of $T$ periods. Figure \ref{fig.T-periodic} compares stationary policy $\nu$ and a $T$-periodic policy $(\mu, \boldsymbol{\pi})$ for the case of $T=4$. In the figure, we also illustrate how rewards can be written in an ``aggregated'' fashion over the $T$ periods using \eqref{eq.r_T_original}.
\begin{figure}[ht]
	\centering
	\includegraphics[width=0.8\textwidth]{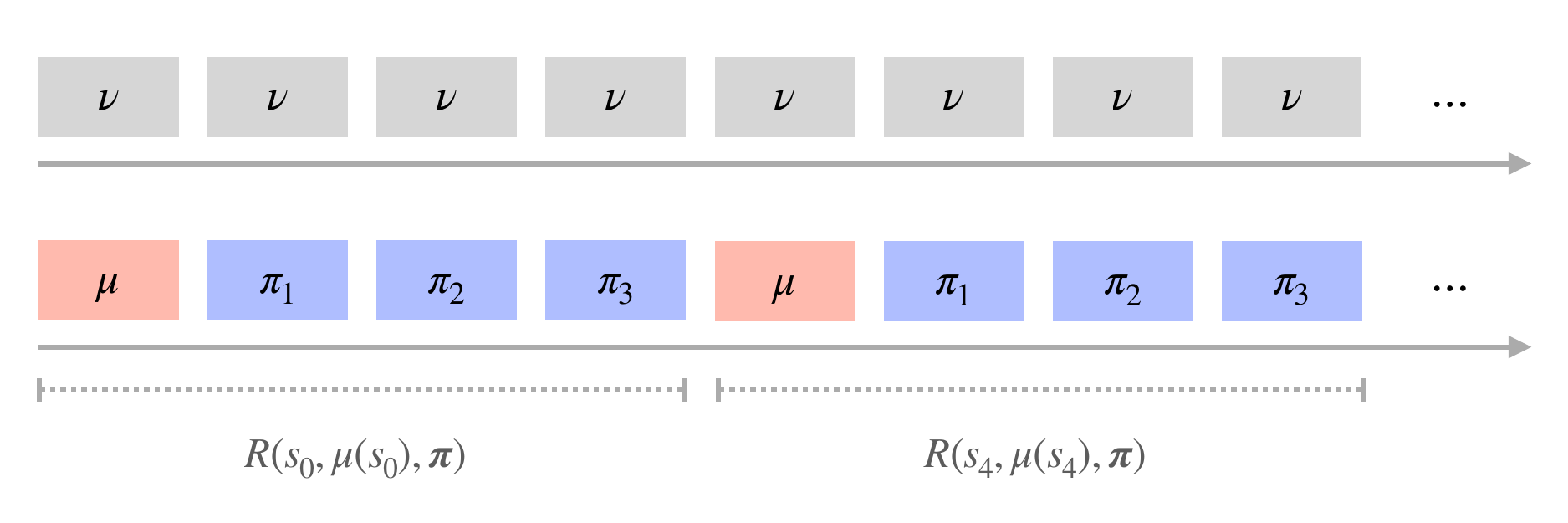}
	\caption{Illustration of a stationary policy $\mu$ (upper timeline) and a $T$-periodic policy $(\mu, \boldsymbol{\pi})$ (lower timeline) for $T=4$. The periods covered by the $T$-period reward associated with $(\mu, \boldsymbol{\pi})$ is visualized with brackets in the lower timeline.}
	\label{fig.T-periodic}
\end{figure}

The optimal value function satisfies the following Bellman equation:
\begin{equation} \label{eq.Bellman_T_original}
    \bar{U}^*(s_0) 
    =\max_{(\mu, \boldsymbol{\pi})} \, \E \bigl[ R(s_0, \mu(s_0), \boldsymbol{\pi}) + \gamma^T \bar{U}^*(s_T) \bigr],
\end{equation}
where the ``action'' now involves selecting the $\boldsymbol{\pi}$ as well.
Denote $(\mu^*,\boldsymbol{\pi}^*)$ an optimal $T$-periodic policy, which solves (\ref{eq.Bellman_T_original}).
In Proposition \ref{thm:stationary_opt_policy}, we prove that the base model \eqref{eq.Bellman_original} and the hierarchical reformulation \eqref{eq.Bellman_T_original} are equivalent in a certain sense.

\begin{proposition}
\label{thm:stationary_opt_policy}
    Given an MDP $\langle \mathcal{S}, \mathcal{A}, \mathcal{W}, f, r, \gamma \rangle$, the following hold:
    \begin{enumerate}[(i)]
        \item The optimal value of the base model \eqref{eq.Bellman_original} is equal to the optimal value of the hierarchical reformulation \eqref{eq.Bellman_T_original}, i.e., $U^* = \bar{U}^*$.
        \item An optimal stationary policy $\nu^*$ with respect to the base model \eqref{eq.Bellman_original} is also an optimal policy for the hierarchical reformulation \eqref{eq.Bellman_T_original}, i.e., $\bar{U}^* = \bar{U}^{\nu^*}$.
    \end{enumerate}
\end{proposition}
\begin{proof}
See Appendix~\ref{sec:stationary_opt_policy_proof}.
\end{proof}
 Part (i) of Proposition~\ref{thm:stationary_opt_policy} is most relevant to our situation in the sense that the optimal $T$-periodic policy $(\mu^*,\boldsymbol{\pi}^*)$ is no better than the stationary optimal policy $\nu^*$. Therefore, solving the hierarchical reformulation \eqref{eq.Bellman_T_original} allows us to achieve the same value as the $\nu^*$, the optimal policy to the original base model \eqref{eq.Bellman_original}.

Note that, at this point, we have simply reformulated the problem, but \eqref{eq.Bellman_T_original} is no easier to solve than \eqref{eq.Bellman_original}. Despite the more favorable discount factor $\gamma^T$ in \eqref{eq.Bellman_T_original}, its action space is now effectively the space of $T$-horizon policies, rather than a single action $a$. In the next section, we propose an approximation that allows us to \emph{fix} a lower-level policy $\boldsymbol{\pi}$ and only optimize $\mu$. This allows us to enjoy the $\gamma^T$ discount factor while maintaining the same action space.

\section{The Frozen-State Approximation} \label{sec:hierarchical_approx}
We propose a \emph{frozen-state approximation}, where we make two simplifications to the $T$-period finite-horizon problem with terminal value $U^*$ that is embedded in each $T$-period ``time step'' of \eqref{eq.Bellman_T_original}, termed the \emph{lower-level problem}. First, motivated by the slow transitions of $x$ given in Assumption \ref{assumption.fast_slow_transition}, we ``freeze'' slow states for all $T$ periods of the lower-level problem, and second, we decouple the problem from the main MDP by solving an approximation with zero terminal value instead of $U^*$.

The first simplification reduces the computation needed to solve the finite-horizon MDP. The second simplification, due to the decoupling from the main problem, allows us to \emph{pre-compute} an approximation to $\boldsymbol{\pi}^*$, which we denote $\tilde{\boldsymbol{\pi}}^*$. By then fixing $\tilde{\boldsymbol{\pi}}^*$, we are able to construct an auxiliary problem that proceeds at a timescale that is a factor of $T$ \emph{slower} than the MDP of the base model (equivalently, the discount factor becomes $\gamma^T$ instead of $\gamma$), yet optimizing over the same action space. This naturally leads to algorithms with computational benefits (see Section~\ref{sec:frozen_state_error}). The number of periods $T$ to freeze the state is a parameter to the approach.  See Figure \ref{fig.hierarchical} for a high-level illustration; we provide a detailed description of the approach in the next few sections.

\begin{figure}[ht]
	\begin{subfigure}[ht]{0.48\textwidth}
		\includegraphics[width=\textwidth]{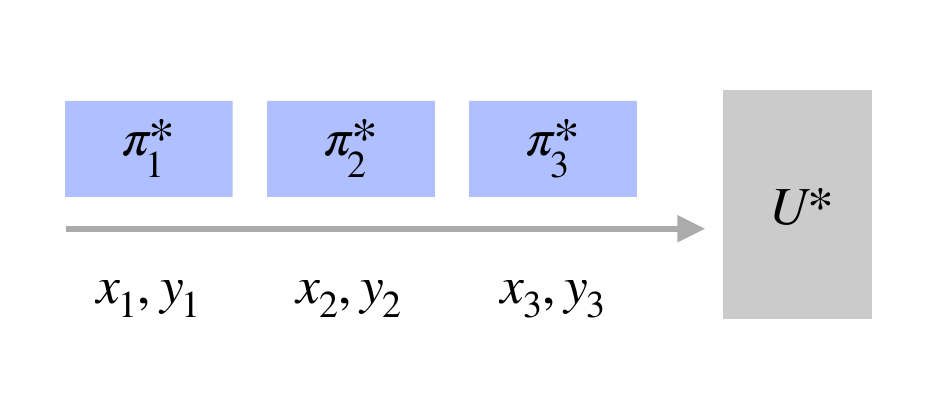}
		\caption{The lower-level problem (i.e., optimizing over $\boldsymbol{\pi}$) embedded in \eqref{eq.Bellman_T_original}.}
	\end{subfigure}
	\begin{subfigure}[ht]{0.48\textwidth}
	    \includegraphics[width=\textwidth]{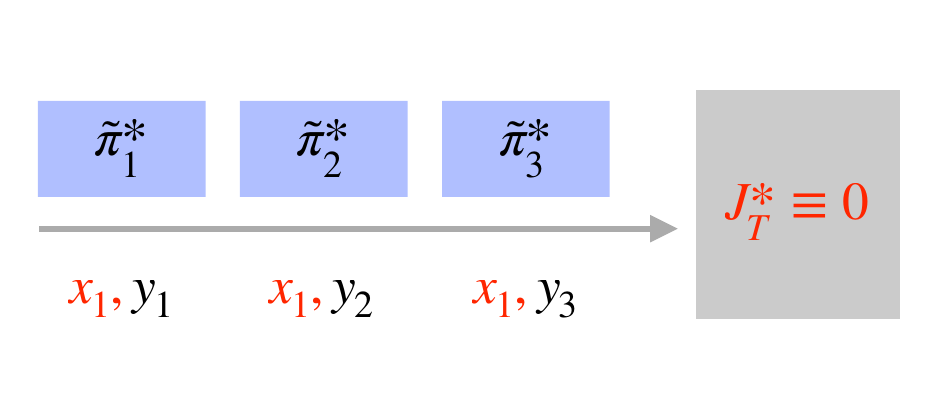}
	    \caption{The lower-level problem of the frozen-state approximation, with frozen $x_1$ and $J_T^*$ instead of $U^*$.}
	\end{subfigure}	
	\caption{A comparison of the lower-level problem of the hierarchical reformulation vs the lower-level problem of the frozen-state approximation.}
	\label{fig.hierarchical}
\end{figure}

\begin{remark}
It is important to note that the freezing of states only occurs ``within the algorithm'' as a step toward more efficient computation of policies. Our resulting policies are then implemented in the underlying base model MDP, which proceeds naturally according to its true dynamics. Our theoretical and empirical results always attempt to answer the question: how well does an approximate policy, which is computed by pretending certain states are frozen, perform in the true model?
\label{rem:pretend}
\end{remark}

\subsection{The Lower-Level MDP (Frozen Slow States)} 
\label{sec:lowerlevel}
We view the problem from period $1$ to period $T$ as the ``lower level'' of the frozen-state approximation.\footnote{This corresponds to the periods relevant to $\boldsymbol{\pi}$ from $(\mu, \boldsymbol{\pi})$ in the hierarchical reformulation \eqref{eq.Bellman_T_original}, whose structure the frozen-state approximation mimics.}
To form the lower-level problem of the frozen-state approximation, we consider this $T-1$ period problem in isolation:
\begin{equation}
J_1^{\tilde{\boldsymbol{\pi}}}(x, y) = \E \left[ \sum_{t=1}^{T-1} \gamma^{t-1} \, r(x_1, y_t, \tilde{\pi}_t) \,\Bigr| \, (x_1, y_1) = (x,y) \right] \quad \text{and} \quad \max_{\tilde{\boldsymbol{\pi}}} \; J_1^{\tilde{\boldsymbol{\pi}}}(x, y)
\label{eq.lower_objective}
\end{equation}
where $x_{t+1} = x_t = x$ remains frozen, $y_{t+1} = f^{\tilde{\pi}_t}_\mathcal{Y}(x, y_t, w_{t+1})$,
and $\tilde{\boldsymbol{\pi}} = (\tilde{\pi}_1, \ldots, \tilde{\pi}_{T-1})$. The problem \eqref{eq.lower_objective} can be solved using backward dynamic programming: accordingly, let the terminal $J^*_T \equiv 0$ and for $t=1,2,\ldots,T-1$, let
\begin{equation} \label{eq.Vlower}
    J^*_t(x,y) = \max_{a} \,  r(x,y,a) + \gamma \, \E \bigl[ J^*_{t+1}(x,y') \bigr],
\end{equation}
where $y' = f_\mathcal{Y}(x,y,a,w)$. We also have the standard recursion for the performance of a policy: \begin{equation}
    J^{\tilde{\boldsymbol{\pi}}}_t(x,y) = r(x,y,\tilde{\pi}_t(x,y)) + \gamma \, \E \bigl[ J^{\tilde{\boldsymbol{\pi}}}_{t}(x,f^{\tilde{\pi}_t}_\mathcal{Y}(x, y, w_{t+1})) \bigr],
    \label{eq.bellmaneval}
\end{equation}
with $J_T^{\tilde{\boldsymbol{\pi}}} \equiv 0$. We denote by $\tilde{H}$ the Bellman operator of the lower-level problem, which is on the same timescale as the base model (hence, the discount factor is $\gamma$) and looks similar to the Bellman operator $H$ defined in (\ref{eq.Vbase_operator}), but the transition of the slow-state $x$ is frozen. For any state $(x,y)$ and lower-level value function $J_{t+1}$,\footnote{We include time indexing on the value function to emphasize that this Bellman operator is used in a finite-horizon (i.e., non-stationary) setting, but the definition of $\tilde{H}$ itself does not depend on $t$.} define:
\begin{equation} \label{eq.Vlower_operator}
    \bigl (\tilde{H} J_{t+1} \bigr) (x,y) = \max_{a} \, r(x,y,a) + \gamma \, \E \bigl[J_{t+1}(x, f_{\mathcal{Y}}(x,y,a,w))\bigr].
\end{equation}
Note that (\ref{eq.Vlower_operator}) can be viewed as an approximation to (\ref{eq.Vbase_operator}). Analogously, let $\bar{H}^{\tilde{\boldsymbol{\pi}}}$ be the Bellman operator associated with \eqref{eq.bellmaneval}.

Also, let $\tilde{\boldsymbol{\pi}}^* = (\tilde{\pi}^*_1, \ldots, \tilde{\pi}^*_{T-1})$ be the finite-horizon policy that is greedy with respect to $J_t^*$:
\[
\tilde{\pi}^*_t(x, y) = \argmax_{a} \; r(x,y,a) + \gamma \, \E \bigl[J_{t+1}^*(x,y')\bigr].
\]
Note that using $J_T^* \equiv 0$ is aligned with the fact that the lower-level problem \eqref{eq.lower_objective} contains $T-1$ periods. It thus follows that $J_1^{\tilde{\boldsymbol{\pi}}^*} = J_1^*$, meaning that we can use $J_1^*$ to represent the value of the lower-level problem. We will see in the next section that the upper-level problem depends on an approximation of $J_1^{\tilde{\boldsymbol{\pi}}^*}$, which is the value of the lower-level policy in $T-1$ periods (and zero terminal value after that).

\begin{remark}[Non-zero terminal value] 
What about when one wants to solve the lower-level problem with some $J_T^* \neq 0$? In some applications, it may be natural or desirable to assign a non-zero terminal value at time $T$. This would allow for a heuristic continuation value based on the algorithm designer's domain knowledge. In this case, the lower-level policy $\tilde{\boldsymbol{\pi}}^*$ can still be computed using finite-horizon dynamic programming, starting from the non-zero $J_T^*$ as the terminal condition in the recursion \eqref{eq.Vlower}. However, for consistency of the upper-level problem, it is important to evaluate the performance of the resulting policy $\tilde{\boldsymbol{\pi}}^*$ using the original objective \eqref{eq.lower_objective} (without the terminal value). In other words, when $J_T^* \neq 0$, $J_1^{\tilde{\boldsymbol{\pi}}^*}$ is no longer guaranteed to be the same as $J_1^*$. Therefore, to use a non-zero terminal value, we need to perform an additional policy evaluation step for $\tilde{\boldsymbol{\pi}}^*$ to obtain $J_1^{\tilde{\boldsymbol{\pi}}^*}$. This is a relatively simple one-time calculation that can be performed concurrently with the computation of $\tilde{\boldsymbol{\pi}}^*$. We assume throughout the paper that $J_T^* \equiv 0$ for simplicity.
\end{remark}

It may not immediately be clear why freezing slow states is desired. There are two main computational benefits to solving \eqref{eq.Vlower} instead of an analogous version of \eqref{eq.Vlower} \emph{without} freezing $x$:
\begin{itemize}
    \item In algorithms like value iteration \citep{puterman2014markov}, each update requires computing expectations over successor states, and therefore the number of successor states impacts the number of operations for each step of value iteration. When $x$ is frozen, the number of successor states is much smaller since we only have successor fast states ($y'$): in other words, we only need to compute $\E \bigl[J_{t+1}^*(x,y')\bigr]$ instead of $\E \bigl[J_{t+1}^*(x',y')\bigr]$.\footnote{Even in the case that the expectation is approximated via sampling, the former requires sampling from a lower-dimensional successor state distribution.}
        
    \item  Second, \eqref{eq.Vlower} can effectively be viewed as $|\mathcal X|$ independent MDPs, one for each $x\in\mathcal X$, allowing for the possibility of computing the policy with additional parallelism. 
\end{itemize}

\subsection{The Upper-Level MDP (True State Dynamics)}
Let us now consider the upper-level problem of the frozen-state approximation, which is an infinite horizon problem with groups of $T$ periods aggregated. Denote the stationary upper-level policy by $\mu:\mathcal{S} \rightarrow \mathcal{A}$, which is the policy that we are attempting to optimize in the upper-level problem. The upper-level problem takes two ``inputs'' related to the lower-level problem: (1) $J_1$, an approximation of the optimal lower-level value $J_1^*$, (2) $\boldsymbol{\pi}$, a lower-level finite-horizon policy. Fixing these inputs, the value at state $s_0=(x_0,y_0)$ by executing policy $\mu$ is
\begin{equation*} \label{eq.Vupper_policy}
    \begin{aligned}
    V^{\mu}(s_0,J_1,\boldsymbol{\pi})
    = \E\bigl[\tilde{R}(s_0, \mu(s_0), J_1) + \gamma^T V^{\mu}(s_T(\mu, \boldsymbol{\pi}),J_1,\boldsymbol{\pi})\bigr],
    \end{aligned}
\end{equation*}
where $s_T(\mu, \boldsymbol{\pi})$ is the state reached according to the true system dynamics by following $(\mu, \boldsymbol{\pi})$, starting from $s_0$ and
\begin{equation}
\tilde{R}(s_0,a,J_1) = r(s_0,a) + \gamma \, J_1 \bigl(f(s_0,a,w)\bigr)
\label{eq:R_tilde_defn}
\end{equation}
is a one-step approximation to the $T$-period reward function $R$, defined in \eqref{eq.r_T_original}. Figure \ref{fig.upper_problem} helps to visualize the upper-level MDP.

\begin{figure}[h]
	\centering
	\includegraphics[width=0.9\textwidth]{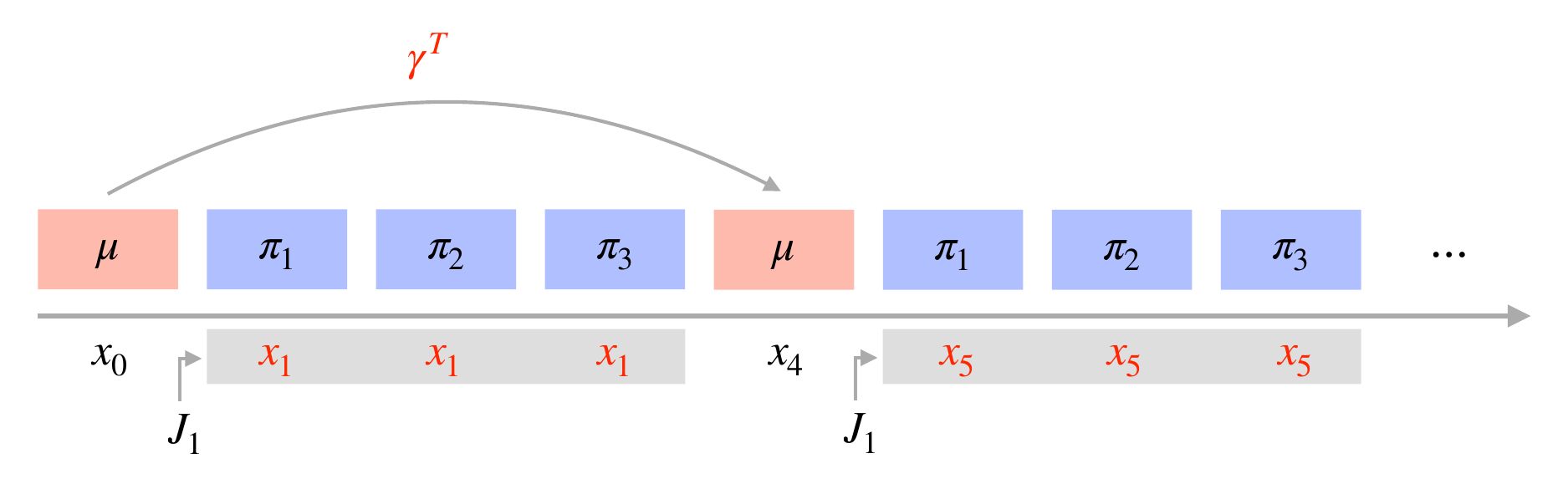}
	\caption{Illustration of the upper-level problem. Notably, the discount factor is $\gamma^T$ and the reward function, from the point of view of $\mu$, depends on the lower-level value function $J_1$. This value function is computed by freezing states, as visualized by the grey box.}
	\label{fig.upper_problem}
\end{figure}

The optimal value (for this approximation) at state $s_0$ can be written as
\begin{equation} \label{eq.Vupper}
    V^*(s_0,J_1,\boldsymbol{\pi})
    = \max_{a} \, \E\bigl[\tilde{R}(s_0, a, J_1) + \gamma^T V^*(s_T(a, \boldsymbol{\pi}),J_1,\boldsymbol{\pi})\bigr],
\end{equation}
where $s_T(a, \boldsymbol{\pi})$ is the state reached according to the true system dynamics by first taking action $a$ and then following $\boldsymbol{\pi}$, starting from $s_0$.

Throughout the paper, we use the notations $V^\mu(J_1, \boldsymbol{\pi}) : \mathcal S \rightarrow \mathbb R$ and $V^*(J_1, \boldsymbol{\pi}):\mathcal S \rightarrow \mathbb R$ to refer to the value function obtained when the MDP is evaluated or solved for a fixed $J_1$ and $\boldsymbol{\pi}$, respectively. We also define the Bellman operator associated with (\ref{eq.Vupper}):
\begin{equation} \label{eq.Vupper_operator}
    \bigl(F_{J_1, \boldsymbol{\pi}} V\bigr)(s_0)
    = \max_{a} \, \E\bigl[\tilde{R}(s_0, a, J_1) + \gamma^T V(s_T(a, \boldsymbol{\pi}))\bigr],
\end{equation}
which will become useful later on. 

\looseness-1 Recall that the optimal lower-level policy (for the frozen-state model) is denoted $\tilde{\boldsymbol{\pi}}^*$ and its optimal value is $J_1^*$. Let $\tilde{\mu}^*$ be the optimal upper-level policy corresponding to these inputs, i.e., the policy that is greedy with respect to $V^*(s_0,J_1^*,\tilde{\boldsymbol{\pi}}^*)$. Thus, $(\tilde{\mu}^*, \tilde{\boldsymbol{\pi}}^*)$ is the resulting $T$-periodic policy from the frozen-state hierarchical approximation; we refer to it as the \emph{$T$-periodic frozen-state policy}.

\section{Frozen-State Value Iteration} \label{sec:frozen_state_error}
In this section, we introduce the our new approach: the \emph{frozen-state value iteration} (FSVI) algorithm. We start by mentioning that the naive approach to solving the base model MDP (\ref{eq.Bellman_original}) is to directly apply standard VI (see, e.g., \cite{bertsekas1996neuro}). For completeness, we provide the full description of standard VI in Algorithm \ref{alg:exact_vi}. 

Our new approach, FSVI, is given in Algorithm~\ref{alg:vi_appr}. The main ideas of the algorithm are:
\begin{enumerate}
\item Solve the lower-level MDP with frozen states to obtain a policy $\tilde{\boldsymbol{\pi}}^*$ and its value $J_1^*$. Since the lower-level problem is a finite horizon MDP, it can be solved exactly using $T-1$ steps of VI. An alternative way of viewing the lower-level problem is that it is $|\mathcal X|$ independent MDPs, which we can potentially solve in parallel. Note that Algorithm \ref{alg:vi_appr} emphasizes this by looping over the slow state $x \in \mathcal X$ in Line 1.
\item Apply value iteration (VI) to the upper-level problem starting with some initial value function $V^0$, while using $J_1^*$ to approximate the $T$-horizon reward and $\tilde{\boldsymbol{\pi}}^*$ for $T$-step transitions. Note that this is an infinite-horizon MDP that operates at a slower timescale and enjoys a more favorable discount factor of $\gamma^T$.
\end{enumerate}

We denote the resulting value function approximation after $k$ iterations of value iteration as $V^k$, from which we obtain a policy $\tilde{\mu}^k$. The $T$-periodic policy output by FSVI is formed by combining $\tilde{\mu}^k$ with the optimal finite-horizon policy $\tilde{\boldsymbol{\pi}}^*$ from the lower-level MDP: $(\tilde{\mu}^k, \tilde{\boldsymbol{\pi}}^*)$.

\begin{algorithm}
	\SetKwInput{Input}{Input}
	\SetKwInput{Output}{Output}
	\BlankLine
	\Input{Initial values $U^0$, number of iterations $k$.}
		\medskip
	\Output{Approximation to the optimal policy $\nu^k$.}
	\BlankLine
	\For{$k' = 1, 2, \ldots, k $}{
	    \vspace{0.5em}
	    \For{$s$ in the state space $\mathcal{S}$}{
	        \vspace{0.5em}
    $U^{k'}(s) = \max_{a} r(s,a) + \gamma\, \E\bigl[U^{k'-1}(f(s,a,w))\bigr].$ 
        	\vspace{0.5em}
    	}
	}
	\vspace{0.5em}
	\For{$s$ in the state space $\mathcal{S}$}{
        \vspace{0.5em}
    	$\nu^{k}(s) = \argmax_{a} r(s,a) + \gamma\, \E\bigl[U^{k}(f(s,a,w))\bigr]$. 
    	\vspace{0.5em}
	}
	\caption{Exact VI for the Base Model (Base VI)}
	\label{alg:exact_vi}
\end{algorithm}%

\begin{algorithm}
	\SetKwInput{Input}{Input}
	\SetKwInput{Output}{Output}
	\BlankLine
	\Input{Initial values $J_T^* \equiv 0$ and $V^0$, number of iterations $k$.}
		\medskip
	\Output{Approximation of the $T$-periodic frozen-state policy $(\tilde{\mu}^k, \tilde{\boldsymbol{\pi}}^*)$ and $J_1^*$.}
	\BlankLine

    	\For{each slow state $x \in \mathcal{X}$}{
		\vspace{0.5em}
        	\For{$t = T-1, T-2, \ldots, 1$}{
		\vspace{0.5em}
		    \For{each fast state $y \in \mathcal{Y}$}{
        		\vspace{0.5em}
        		$J^*_t(x,y) = \max_a r(x,y,a) + \gamma  \, \E \bigl[ J^*_{t+1}(x,f_\mathcal{Y}(x,y,a,w)) \bigr]$. \\
        		\vspace{0.5em}
        		$\tilde{\pi}_t^*(x,y) = \argmax_a r(x,y,a) + \gamma \, \E \bigl[ J^*_{t+1}(x,f_\mathcal{Y}(x,y,a,w))\bigr]$.
        		\vspace{0.5em}
		    }
		}
	}
	\BlankLine
	\For{$k' = 1, 2, \ldots, k $}{
	    \vspace{0.5em}
	    \For{$s_0=(x_0,y_0)$ in the state space $\mathcal{X}\times \mathcal{Y}$}{
	        \vspace{0.5em}
        	$V^{k'}(x_0,y_0,J^*_1,\tilde{\boldsymbol{\pi}}^*) = \max_{a} \E\bigl[\tilde{R}(s_0, a, J^*_1) + \gamma^T V^{i-1}(x_T,y_T,J^*_1,\tilde{\boldsymbol{\pi}}^*)\bigr]$. 
        	\vspace{0.5em}
    	}
	}
	\vspace{0.5em}
	\For{$s_0=(x_0,y_0)$ in the state space $\mathcal{X}\times \mathcal{Y}$}{
        \vspace{0.5em}
    	$\tilde{\mu}^k(x_0,y_0) = \argmax_{a} \E\bigl[\tilde{R}(s_0, a, J^*_1) + \gamma^T V^k(x_T,y_T,J^*_1,\tilde{\boldsymbol{\pi}}^*)\bigr]$. 
    	\vspace{0.5em}
	}
	\caption{Frozen-State Value Iteration (FSVI)}
	\label{alg:vi_appr}
\end{algorithm}%

\subsection{Natural Application Domains for FSVI}

While the FSVI algorithm is general and can be applied to any MDP, we do not necessarily expect it to perform well in any arbitrary setting. Its approximation relies heavily on a separation of timescales and on the idea that some form of structured behavior can be captured by freezing slow variables over short horizons. That said, there are two broad classes of problems where we expect FSVI to work particularly well:

\begin{enumerate}
    \item \textbf{Problems with implicit hierarchical structure.}
Many real-world applications do not exhibit explicit hierarchy in their MDP formulations but nevertheless contain natural two-level behavioral structure. We illustrate this type of problem in Section \ref{sec:experiment} through our gridworld with spatial tasks domain: the agent first ``accepts'' a task (each task comes with a starting and ending location) to undertake and then executes a sequence of navigation actions to complete it. This problem was motivated by operations domains such as warehouse robotics \citep{yang2020novel} and on-demand spatial services from the driver's perspective \citep{chung2018bike,jiang2020optimistic,ong2021driver,ashkrof2024relocation}, where high-level task acceptance decisions are followed by fast-paced and relatively more myopic execution.
\item \textbf{Problems with natural cyclical behavior.} In some domains, optimal or near-optimal behavior follows a cyclic structure. A classic example is inventory control, where the policy repeatedly builds up and depletes inventory. We show results on inventory control with fixed costs in Section \ref{sec:experiment}. Related problems, such as energy or commodity trading \citep{kim2011optimal,carmona2010valuation,lohndorf2010optimal,jiang2015optimal} or inventory repositioning \citep{he2020robust,benjaafar2022dynamic}, may also exhibit similar cyclical behavior. Another domain we consider in Section \ref{sec:experiment} is a dynamic pricing problem with reference effects, where cyclic pricing strategies have been shown to perform well under various assumptions \citep{wang2016intertemporal,chen2017efficient}. In all of the above settings, freezing the slow state over a short horizon can capture the cyclical dynamics of the optimal policy.
\end{enumerate}
A key advantage of FSVI is that it can discover both forms of structure automatically through its two-level decomposition. Therefore, the user of the algorithm need not encode any concrete hierarchical or cyclical behavior and only needs to set the parameter $T$, which can often be selected using domain expertise. In Section \ref{sec:interpreting}, we also show how our theoretical analyses may guide the selection of $T$.

\subsection{Computational Cost of FSVI}
\label{sec:runningtime}
Let us now discuss the computational cost of FSVI. Here, we take a dynamic programming perspective,\footnote{In our numerical results of Section \ref{sec:experiment}, we will take a different, more practical perspective, where we work with a generative model or simulator of the environment.} where the model is known and expectations can be computed. It is well known that each iteration of standard VI, which provides a contraction factor of $\gamma$, has time complexity $\mathcal O\bigl(|\mathcal S|^2|\mathcal A|\bigr)$ \citep{littman1995complexity}. The upper level of FSVI, on the other hand, enjoys an improved contraction factor $\gamma^T$ with the same per-iteration running time of $\mathcal O\bigl(|\mathcal S|^2|\mathcal A|\bigr)$, given that we pay a \emph{one-time fixed cost} of solving the lower level. The quantity $\mathcal O\bigl(|\mathcal S|^2|\mathcal A|\bigr)$ consists of a factor of $|\mathcal S ||\mathcal A|$ due to the number of state-action pairs at which to compute the Bellman update and another factor of $|\mathcal S|$ due to the number of successor states. Since freezing slow states restricts the successor states to $\mathcal Y$, each iteration of the lower-level VI (first part of Algorithm \ref{alg:vi_appr}) has running time $\mathcal O\bigl(|\mathcal X| |\mathcal Y|^2|\mathcal A|\bigr)$. An additional $\mathcal O\bigl(|\mathcal S|^2 \, T\bigr)$ is required to compute the $T$-step transition probabilities of following $\tilde{\boldsymbol{\pi}}^*$, to be used in the upper-level VI steps, resulting in a one-time fixed cost of $\mathcal O ( |\mathcal X| |\mathcal Y|^2|\mathcal A| \, T + |\mathcal S|^2\, T)$. Particularly when $|\mathcal X|$ is large, this can be a reasonable fixed cost to pay in order to get the much improved discount factor of $\gamma^T$ going forward. The equations below give a direct comparison between the computational cost of VI vs FSVI:
\[
\text{Cost}_{\text{VI}}(k)
=\;
\mathcal O\!\bigl(k\,|\mathcal S|^{2}\,|\mathcal A|\bigr),
\tag{VI after \(k\) iterations}
\]
\[
\text{Cost}_{\text{FSVI}}(k)
=\;
\mathcal O\!\Bigl(\,
\underbrace{|\mathcal X|\,|\mathcal Y|^{2}\,|\mathcal A|\,T
\;+\;
|\mathcal S|^{2}\,T}_{\substack{\text{one-time lower-level}\\\text{pre–processing cost}}}
\;+\;
\underbrace{k\,|\mathcal S|^{2}\,|\mathcal A|}_{\substack{\text{upper-level}\\\text{iterations}}}
\, \Bigr).
\tag{FSVI after \(k\) iterations}
\]

\section{Theoretical Analysis}
\label{sec:theory}

In this section, we develop the theory to understand the regret of FSVI. First, in Section \ref{sec:reward_approx_error}, we derive the reward approximation error between the hierarchical reformulation and the frozen-state approximation. This result becomes needed in later sections. In Section \ref{sec:defn_regret}, we define the notion of regret that we use in the paper. Then in Section \ref{sec:general_lemma}, we give an important lemma that illustrates the trade-offs between the various sources of error and how they contribute to the overall regret. Finally in Section \ref{sec:regret_fsvi}, we give the main result of the section, which characterizes the regret of FSVI for a particular $T$ and $k$.

\subsection{Reward Approximation Error}
\label{sec:reward_approx_error}

Recall that $(\mu^*,\boldsymbol{\pi}^*)$ is an optimal $T$-periodic policy of the base model's hierarchical reformulation \eqref{eq.Bellman_T_original}. Suppose $\boldsymbol{\pi}^*$ is available. Then, the Bellman equation of the base model reformulation is
\begin{align}
    U^*(x_0,y_0) 
    &= \bar{U}^*(x_0,y_0) \nonumber\\
    &= \max_{a} \, \E \bigl[ R(x_0,y_0, a, \boldsymbol{\pi}^*) + \gamma^T \, \bar{U}^*(x_T,y_T) \bigr] \nonumber\\
    &= \max_{a}  \, \E \left[ r(x_0,y_0,a) + \sum_{t=1}^{T-1} \gamma^t \, r(x_t,y_t,\pi^*_t) + \gamma^T \, U^*(x_T,y_T) \right] \nonumber\\
    &= \max_{a} \, \E \Bigl[ r(x_0,y_0,a) + \gamma \, \bigl (H^{T-1} U^* \bigr) (x_1,y_1) \Bigr], \label{eq.base_model_reform2}
\end{align}
where the notation $H^k$ is shorthand for $k$ applications of the operator $H$, i.e., $H^k U = H(H^{k-1} U)$ and $H^1 U = H U$.
Therefore, the expected $T$-horizon reward can be written as
\begin{equation} \label{eq.R1_a}
    \E \bigl [R(x_0,y_0, a, \boldsymbol{\pi}^*) \bigr] = \E \Bigl [r(x_0,y_0,a) + \gamma \, \bigl (H^{T-1} U^* \bigr) (x_1,y_1) - \gamma^T \, U^*(x_T,y_T) \Bigr].
\end{equation}

Given the optimal value $J_1^*$ of the lower level \eqref{eq.Vlower}, the $T$-horizon reward of the upper level \eqref{eq.Vupper} can be written as
\begin{align} 
    \E \bigl [\tilde{R}(x_0, y_0 ,a,J_1^*) \bigr] &= r(x_0,y_0,a) + \gamma \, \E \bigl[J_1^*(x_1, y_1)\bigr] \nonumber\\
    &= r(x_0,y_0,a) + \gamma  \bigl (\tilde{H}^{T-1} J_T^* \bigr) (x_1,y_1), \nonumber\\ &= r(x_0,y_0,a) + \gamma  \bigl (\tilde{H}^{T-1} \, \boldsymbol{0} \bigr) (x_1,y_1),\label{eq.R2_a}
\end{align}
where $\boldsymbol{0}$ is the all-zero value function. %
The difference between \eqref{eq.R1_a} and \eqref{eq.R2_a} can be interpreted as follows: in the former, we follow a lower-level policy that is \emph{aware} of a terminal value $U^*$ (but exclude that value when defining the $T$-horizon reward), while in the latter, we follow a lower-level policy that sees zero terminal reward at the end of the $T-1$ periods.

The first step to understanding the performance of the frozen-state policy is to analyze the reward approximation $\E[\tilde{R}(s_0, a, J_1^*)]$ compared to the true reward $\E[R(s_0, a, \boldsymbol{\pi}^*)]$. Proposition~\ref{prop:cumreward_diff} shows how the difference between two reward functions is dependent on the number of frozen periods $T$, along with the problem parameters. 
\begin{proposition}[Reward Approximation Error]
    \label{prop:cumreward_diff}
    Let $\langle \mathcal{S}, \mathcal{A}, \mathcal{W}, f, r, \gamma \rangle$ be a $(\alpha, d_\mathcal{Y})$-fast-slow MDP satisfying Assumption \ref{assumption:lipschitz}. Let $\boldsymbol{\pi}^*$ be the optimal lower-level policy for the base model reformulation \eqref{eq.Bellman_T_original} and $J_1^*$ be the optimal (first-stage) value of the lower-level problem in the frozen-state approximation \eqref{eq.Vlower}. For any state $s_0=(x_0,y_0)$ and action $a$, the approximation error between the $T$-horizon reward of hierarchical reformulation and the frozen-state approximation, i.e., the discrepancy between \eqref{eq.R1_a} and \eqref{eq.R2_a}, can be bounded as:
\begin{align}
    \bigl| \E\bigl[R(&s_0, a, \boldsymbol{\pi}^*)\bigr]- \E\bigl[\tilde{R}(s_0, a, J_1^*)\bigr]\bigr|\nonumber\\
    &\leq 
    \underbrace{\alpha d_\mathcal{Y} \biggl( L_r \sum_{i=1}^{T-2} \gamma^i \, \sum_{j=0}^{i-1} L_f^j \biggr)}_{\textnormal{freeze error}} 
    + 
    \underbrace{\gamma^{T-1} L_U \Biggl[ \alpha d_\mathcal{Y} \sum_{j=0}^{T-2} L_f^j + \gamma d_\mathcal{Y} (\alpha + 2) (T-1)\Biggr]}_{\textnormal{end-of-horizon error}},
    \label{eq:R_prop}
\end{align}

\end{proposition}
\begin{proof}
The detailed proof is in Appendix~\ref{sec:proof_cumreward_diff}.
\end{proof}

For more convenient notation, we define $\epsilon_r(\gamma,\alpha,d_\mathcal{Y},L_r,L_f,T)$ to be the right-hand-side of (\ref{eq:R_prop}):
\[
\epsilon_r(\gamma,\alpha,d_\mathcal{Y},\mathbf{L},T) =\alpha d_\mathcal{Y} \biggl( L_r \sum_{i=1}^{T-2} \gamma^i \, \sum_{j=0}^{i-1} L_f^j \biggr) + \gamma^{T-1} L_U \Biggl[ \alpha d_\mathcal{Y} \sum_{j=0}^{T-2} L_f^j + \gamma d_\mathcal{Y} (\alpha + 2) (T-1)\Biggr],
\]
where $\mathbf{L} = (L_r, L_f, L_U)$ emphasizes the dependence on the various Lipschitz constants.
In subsequent sections, we use $\epsilon_r$ as an ingredient in analyzing the regret of various frozen-state policies.

\subsection{Defining Regret}
\label{sec:defn_regret}
We start with definitions of the regret of both stationary and $T$-periodic policies. %

\begin{definition}[Regret]
Consider a fast-slow MDP with initial state $s_0$ and optimal policy $\nu^*$. The regret of a stationary policy $\nu$ is defined as  
\[
\mathcal{R}(s_0, \nu) = U^{\nu^*}(s_0) - U^{\nu}(s_0) \quad \text{and} \quad \mathcal R(\nu) = \max_{s_0} \, \mathcal{R}(s_0, \nu).
\] The regret of the $T$-periodic policy $(\mu, \boldsymbol{\pi})$ is defined as:
\begin{align*}
\mathcal{R}(s_0, \mu, \boldsymbol{\pi}) = U^{\nu^*}(s_0) - \bar{U}^{\mu}(s_0, \boldsymbol{\pi}) = \bar{U}^{*}(s_0) - \bar{U}^{\mu}(s_0, \boldsymbol{\pi}) \quad \text{and} \quad \mathcal R(\mu, \boldsymbol{\pi}) = \max_{s_0} \, \mathcal{R}(s_0, \mu, \boldsymbol{\pi}).
\end{align*}
The second equality in the definition of $\mathcal{R}(s_0, \mu, \boldsymbol{\pi})$ uses the value equivalence between the base model and its hierarchical reformulation (Proposition \ref{thm:stationary_opt_policy}).
\end{definition}

\begin{remark}
As a follow-up comment to Remark \ref{rem:pretend}, notice that $V^*(s_0,J_1^*,\tilde{\boldsymbol{\pi}}^*)$ does not directly enter the regret definition, as $V^*(s_0,J_1^*,\tilde{\boldsymbol{\pi}}^*)$ is just the optimal value of the frozen-state approximation, not the value of its implied greedy policy $\tilde{\mu}^*$ when evaluated in the base model. However, the regret of course depends on $V^*(s_0,J_1^*,\tilde{\boldsymbol{\pi}}^*)$ indirectly, because $\tilde{\mu}^*$ depends on $V^*(s_0,J_1^*,\tilde{\boldsymbol{\pi}}^*)$.
\end{remark}

\subsection{A General Lemma}
\label{sec:general_lemma}
In this section, we first prove a general lemma that will be used to analyze FSVI.

\begin{lemma}
\label{lemma:main}
Suppose we have an approximation $(\boldsymbol{\pi}, J_1)$ to the lower-level solution $(\boldsymbol{\pi}^*, U^*)$. Further, suppose we have an approximation $V$ to the upper-level solution $V^*(J_1, \boldsymbol{\pi})$. Consider a $T$-periodic policy $(\mu, \boldsymbol{\pi})$, where
\begin{equation}
\label{eq.mu_approx_opt}
\mu(s_0) = \argmax_{a\in\mathcal A} \, \E\bigl[\tilde{R}(s_0, a, J_1) + \gamma^T V(s_T(a, \boldsymbol{\pi}))\bigr].
\end{equation}
Then, the regret of $(\mu, \boldsymbol{\pi})$ can be bounded as follows:
\begin{align*}
\mathcal R(\mu, \boldsymbol{\pi}) \leq \biggl(&\frac{2\gamma^T}{(1-\gamma^T)^2} + \frac{2}{1-\gamma^T}\biggr) \epsilon_r(\boldsymbol{\pi}^*, J_1) \\
&+ \biggl(\frac{2\gamma^{2T}}{(1-\gamma^T)^2} + \frac{2\gamma^T}{1-\gamma^T}\biggr) L_U \, d(\alpha,d_\mathcal{Y},T) + \frac{2\gamma^T}{1-\gamma^T} \bigl \| V^*(J_1, \boldsymbol{\pi}) - V \bigr\|_\infty,
\end{align*}
where $\epsilon_r(\boldsymbol{\pi}^*, J_1) = \max_{s,a} \, |\E [R(s, a, \boldsymbol{\pi}^*)] - \E [\tilde{R}(s, a, J_1)]|$ and $d(\alpha,d_\mathcal{Y},T) = 2 d_\mathcal{Y} (\alpha+1) (T-1)$.
\end{lemma}
\begin{proof}
See Appendix \ref{sec:proof_lemma_main}.
\end{proof}

The result of Lemma \ref{lemma:main} above can be interpreted as the regret being bounded by
\[
\textnormal{reward error (\& end-of-horizon effect)} \; +\; \textnormal{upper-level freeze error} \; +\; \textnormal{$V$-approximation error},
\]
which directly corresponds to the three terms in the bound. The reward error is due to freezing the slow state in the lower level along with using zero terminal value; the upper-level freeze error propagates the lower-level freeze error to the upper level's infinite horizon; and the $V$-approximation error is due to not solving the upper-level problem exactly.

\subsection{Regret of FSVI}
\label{sec:regret_fsvi}
An instance of FSVI is associated with two primary quantities: $k$, the number of VI iterations, and $T$, the number of periods the slow state is frozen in the frozen-state approximation. The next theorem gives a bound on the regret of the policy obtained for a particular $k$ and $T$.
\begin{theorem} \label{thm:fsvi_regret}
    Let $(\tilde{\mu}^k, \tilde{\boldsymbol{\pi}}^*)$ be the resulting $T$-periodic policy after running FSVI for $k$ iterations. The regret incurred when running $(\tilde{\mu}^k, \tilde{\boldsymbol{\pi}}^*)$ in the base model satisfies
    \begin{align*}
        \mathcal R(\tilde{\mu}^{k}, \tilde{\boldsymbol{\pi}}^*) \leq &\biggl(\frac{2\gamma^T}{(1-\gamma^T)^2} + \frac{2}{1-\gamma^T}\biggr) \epsilon_r(\gamma,\alpha,d_\mathcal{Y},\mathbf{L},T)\\
&+ \biggl(\frac{2\gamma^{2T}}{(1-\gamma^T)^2} + \frac{2\gamma^T}{1-\gamma^T}\biggr) L_U \, d(\alpha,d_\mathcal{Y},T)+ \frac{2 r_\textnormal{max} \gamma^{(k+1)T} }{(1-\gamma)(1-\gamma^T)},
    \end{align*}
    where the last term, which depends on $k$, accounts for the error due to value iteration.
\end{theorem}
\begin{proof}
See Appendix \ref{sec:proof_thm:fsvi_regret}.
\end{proof}

\begin{figure}[bh!]
  \centering
  \begin{subfigure}[b]{0.49\textwidth}
    \includegraphics[width=\textwidth]{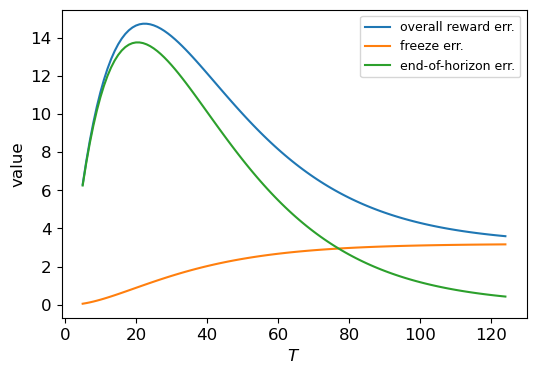}
    \caption{$L_f = 0.99$ (small)}
    \label{fig:freeze_bounded}
  \end{subfigure}
  \hfill
  \begin{subfigure}[b]{0.49\textwidth}
    \includegraphics[width=\textwidth]{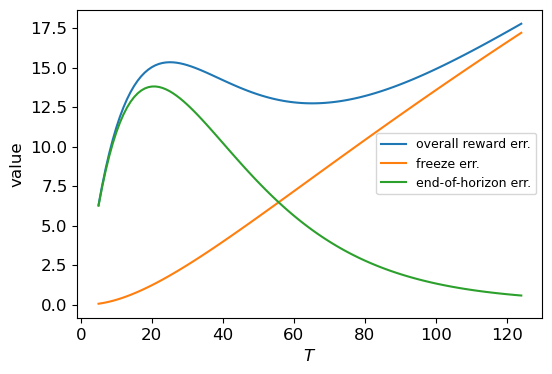}
    \caption{$L_f = 1.05$ (large)}
    \label{fig:freeze_unbounded}
  \end{subfigure}
  \caption{Reward approximation error comparison for two values of $L_f$.}
  \label{fig:reward_approx_comparison}
\end{figure}

\section{Interpreting the Regret Analysis}
\label{sec:interpreting}
It is common for theoretical analyses in reinforcement learning to depend on Lipschitz constants of the underlying MDP (see, e.g., \cite{pirotta2015policy}, \cite{asadi2018lipschitz}, \cite{gelada2019deepmdp}, \cite{sinclair2022adaptive}, or \cite{maran2024no}). Because these Lipschitz constants are typically conservative approximations of the true MDP dynamics, our regret bounds may not always align closely with the empirical regret observed in real-world settings. This is consistent with much of the literature where bounds often serve more as qualitative guides than precise predictors.

Regardless, such regret analysis is valuable, as it can help us understand and clarify the trade-offs involved when selecting various parameters, primarily $T$, the horizon of the frozen lower-level problem. This is particularly true in our setting, where we are balancing several sources of error with computational improvements. In this section, we numerically compute the regret bounds for various values of $T$ in order to obtain insights.

\subsection{Illustrating Reward Error}

First, we examine the reward approximation error analyzed in Proposition \ref{prop:cumreward_diff}. Depending on how the ``freeze error'' and the ``end-of-horizon error'' interact, we see different behaviors of the overall reward approximation error, as illustrated in Figure \ref{fig:reward_approx_comparison}. For both examples, the end-of-horizon error increases at first, but then decreases toward zero as $T$ becomes large. In the example with $L_f < 1$ (left), the freeze error converges as $T$ increases; however, if $L_f > 1$, the freeze error grows without bound. This leads to two distinct patterns in the overall reward approximation error: (1) when the freeze error is well-controlled, the total reward approximation error (blue) generally follows the unimodal shape of the end-of-horizon error (green); and (2) when the freeze error (orange) grows without bound, the reward approximation error (blue) follows a unimodal shape at first, but eventually grows again.

\subsection{Regret of Base VI}
We pause here briefly to state Proposition~\ref{prop:standard_VI_error}, a well-known property that gives the regret of exact VI on the base model. We use this result as a comparison in to the regret of FSVI in the next section.

\begin{proposition}
    \label{prop:standard_VI_error}
    Let $\nu^k$ be the result of running Algorithm \ref{alg:exact_vi} on the base model (\ref{eq.Bellman_original}). Then,
    \begin{equation*}
    \mathcal R(\nu^k) = \|U^{\nu_{k}} - U^*\|_\infty \leq \frac{2 r_\textnormal{max} \gamma^{k+1}}{(1-\gamma)^2},
    \end{equation*}
    we recall that $r_\textnormal{max}$ is an upper bound on the reward.
\end{proposition}
\begin{proof}
See Appendix \ref{proof:standard_VI_error}.
\end{proof}

\begin{figure}[t]
  \centering    \includegraphics[width=0.7\textwidth]{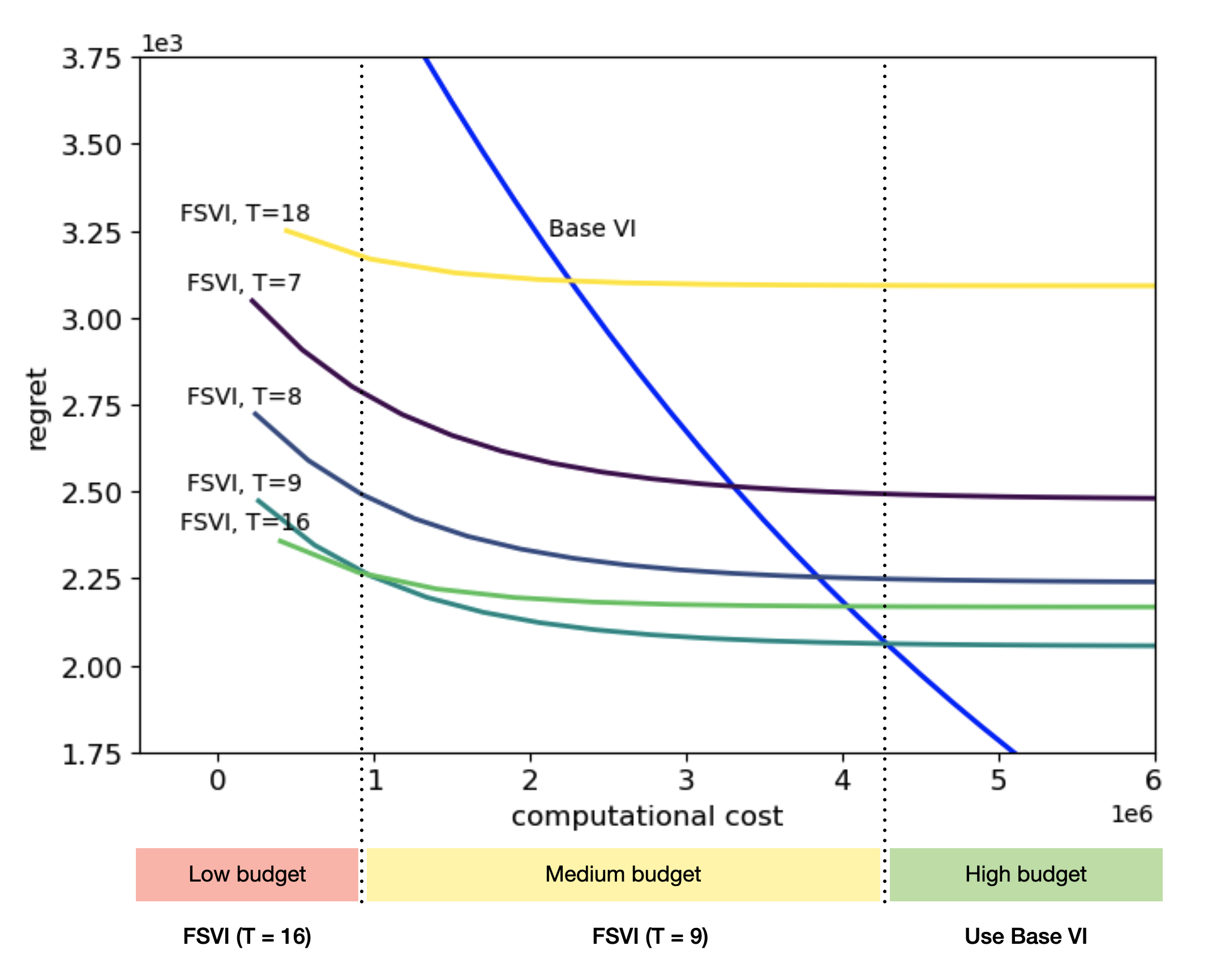}
  \caption{Regret versus computational cost for Base VI and FSVI as the number of value iteration steps increases. In this example, we see three regimes depending on the computational budget: in the low budget regime, the lowest regret algorithm is FSVI with $T=16$; in the medium budget regime, FSVI with $T=9$ achieves lowest regret; and in the high budget regime, not surprisingly, Base VI achieves lowest regret.}
  \label{fig:fsvi_vs_base_vi}
\end{figure}

\subsection{Regret versus Computational Cost for FSVI and VI: Selecting $T$}

Next, we show how Proposition \ref{prop:standard_VI_error} (the regret of Base VI) and Theorem \ref{thm:fsvi_regret} (the regret of FSVI) can be combined to provide intuition about whether to use Base VI or FSVI (and if so, which $T$ and $k$?). We consider the trade-off between regret and computational cost, using the definition of cost given in Section \ref{sec:runningtime}. For FSVI, the concrete problem facing a practitioner is as follows: given a computational cost budget $B$, what combination of $(T, k)$ should be used? We make two observations: first, for a fixed $T$, regret decreases with $k$ (see Theorem \ref{thm:fsvi_regret}), and second, for a fixed $T$, cost increases with $k$ (see Section \ref{sec:runningtime}). Therefore, for a fixed $T$, one should run FSVI for $k_\text{max}(T, B)$ iterations, where $k_\text{max}(T, B)$ is the largest $k$ such that the computational cost of running FSVI with $T$ is less than $B$. It thus follows that to select $T$, one should simply examine the regret for different values of $T$ at $k=k_\text{max}(T, B)$ and select the value of $T$ with the lowest regret.

The above procedure can be converted into a simple visualization. In Figure \ref{fig:fsvi_vs_base_vi}, we plot regret versus computational cost. This is done for both FSVI for different values of $T$ and also for standard value iteration (Base VI).
Each colored curve traces the regret-cost trade‑off generated by varying $k$ for a distinct $T$ (though $k$ is not shown). Imposing the computational budget corresponds to drawing the vertical line $x=B$ and where this line intersects a curve, one can read off the best attainable regret given budget $B$ for the algorithm represented by that curve. By scanning the intersection points and selecting the one with the lowest regret, one can identify the best algorithm to use, whether it is FSVI with a particular $T$ or simply Base VI. Hence, the plot allows us to visually solve the constrained optimization problem of minimizing regret under a given computational budget. This exercise leads us to interesting takeaways illustrated in Figure \ref{fig:fsvi_vs_base_vi}:
\begin{enumerate}
    \item If the available computational budget is low, one should use FSVI with a large value of $T$ (in this case, $T=16$). Despite introducing more error due to freezing, this allows us to get the quickest decrease in regret initially due to the lower discount factor, but as we can see from the plot, the decrease does not sustain as we increase computation.
    \item If the computational budget is high, one should simply use Base VI, as it will eventually achieve zero regret. This is natural and intuitive because FSVI introduces unavoidable error that does not vanish as the number of value iteration steps increases.
    \item With a medium computational budget, FSVI with a moderate value of $T=9$ leads to the best regret guarantee. This regime is where we expect many real-world problems to fall.
\end{enumerate}

\begin{algorithm}[h!]
\SetKwInput{Input}{Input}
\SetKwInput{Output}{Output}
\BlankLine
	\Input{State distribution $\rho$, number of iterations $k$, number of state samples $N$, number of next-state samples $M$, and a regression procedure $\texttt{Regress}$.}
    \medskip
\Output{Approximate value function $U^k$, and policy $\nu^k$.}
\BlankLine
Initialize $U^0$ arbitrarily. \\
\For{$k' = 1, 2, \ldots, k$}{
    \vspace{0.5em}
        Sample a dataset $\mathcal{D} = \{s^{(i)}\}_{i=1}^N \sim \rho$\\
    \vspace{0.5em}
    \For{$s^{(i)} \in \mathcal{D}$}{
        \vspace{0.5em}
        $\displaystyle u^{(i)} = \max_{a}  \,  r(s^{(i)}, a) + \gamma\,  \frac{1}{M} \sum_{j=1}^M \, U^{k'-1}\bigl(s^{(i,a,j)}\bigr) $, where $s^{(i,a,j)}$ is sampled from the generative model, conditional on state $s^{(i)}$ and action $a$.
        \vspace{0.5em}
    }
    \vspace{0.5em}
    $U^{k'} = \texttt{Regress}(\{(s^{(i)}, u^{(i)})\}_{i=1}^N)$
    \vspace{0.5em}
}
\vspace{0.5em}
\For{$s \in \mathcal{S}$}{
    \vspace{0.5em}
    $\displaystyle \nu^k(s) = \argmax_{a}  \,  r(s, a) + \gamma\,  \frac{1}{M} \sum_{j=1}^M \, U^{k}\bigl(s^{(a,j)}\bigr)$, where $s^{(a,j)}$ is sampled from the generative model, conditional on state $s$ and action $a$.
    \vspace{0.5em}
}
\caption{Fitted Value Iteration (FVI) \citep{munos2008finite}}
\label{alg:fvi}
\end{algorithm}

\begin{algorithm}[h!]
	\SetKwInput{Input}{Input}
	\SetKwInput{Output}{Output}
	\BlankLine
	\Input{State distribution $\rho$, number of iterations $k$, number of state samples $N$, number of next-state samples $M_l$ (lower level), $M_u$ (upper level), and a regression procedure $\texttt{Regress}$.}
	\medskip
	\Output{Approximation of the $T$-periodic frozen-state policy $(\tilde{\mu}^k, \tilde{\boldsymbol{\pi}}^*)$ and $\tilde{J}_1$.}
	\BlankLine
	Initialize $\tilde{J}_T \equiv 0$. \\
        \vspace{0.5em}

	\For{$t = T-1, T-2, \ldots, 1$}{
        \vspace{0.5em}
        Sample a dataset $\mathcal{D} = \{(x^{(i)}, y^{(i)})\}_{i=1}^N \sim \rho$\\
        \vspace{0.5em}
		\For{each $(x^{(i)}, y^{(i)}) \in \mathcal{D}$}{
			\vspace{0.5em}
			$\displaystyle J_t^{(i)} = \max_{a \in \mathcal{A}}  \, r\bigl(x^{(i)}, y^{(i)},a\bigr) + \gamma \, \frac{1}{M_l} \sum_{j=1}^{M_l} \tilde{J}_{t+1}\bigl(x^{(i)}, y^{(i,a,j)}\bigr)$
			where $y^{(i,a,j)}$ is sampled from the generative model, conditional on state $(x^{(i)}, y^{(i)})$ and action $a$.
		}
		\vspace{0.5em}
		$\tilde{J}_t = \texttt{Regress}\Bigl(\bigl\{\bigl(  (x^{(i)}, y^{(i)}), J_t^{(i)} \bigr)\bigr\}_{i=1}^N\Bigr)$
	}
        \BlankLine
        Let $\tilde{\boldsymbol{\pi}}^* = (\pi_1^*, \ldots \pi^*_{T-1})$, with $\displaystyle \pi_t^*(s) = \argmax_{a}  \,  r(s, a) + \gamma\,  \frac{1}{M_l} \sum_{j=1}^{M_l} \, \tilde{J}_{t+1}\bigl(s^{(a,j)}\bigr)$, where $s^{(a,j)}$ is sampled from the generative model, conditional on state $s$ and action $a$.    
	\BlankLine
	Initialize $\hat{V}^0$ arbitrarily.
        \BlankLine
	\For{$k' = 1, 2, \ldots, k$}{
		\vspace{0.5em}
		Sample a dataset $\mathcal{D} = \{(x^{(i)}, y^{(i)})\}_{i=1}^N \sim \rho$ \\
		\vspace{0.5em}
		\For{each $(x^{(i)}, y^{(i)}) \in \mathcal{D}$}{
			\vspace{0.5em}
			$ \displaystyle {V}^{(k')} = \max_{a \in \mathcal{A}} \frac{1}{M_u} \sum_{j=1}^{M_u} \left[ \tilde{R}^{(i,a,j)} + \gamma^T {V}^{k'-1}\left(x^{(i,a,j)}, y^{(i,a,j)}\right) \right]$ where $x_T^{(i,a,j)}$, $y_T^{(i,a,j)}$, and $\tilde{R}^{(i,a,j)}$ are sampled using a $T$-step transition conditioned on $x^{(i)}$, $a$, $\tilde{\boldsymbol{\pi}}^*$ and $\tilde{J}_1$.
		}
		\vspace{0.5em}
		$\tilde{V}^{k'} = \texttt{Regress}\Bigl( \bigl \{ \bigl(  (x^{(i)}, y^{(i)}), {V}^{(i)} \bigr) \bigr\}_{i=1}^N\Bigr)$ 
	}
    \BlankLine
	\For{$s_0 = (x_0, y_0) \in \mathcal{S}$}{
	    \vspace{0.5em}
	    $\displaystyle \tilde{\mu}^k(x_0, y_0) = \argmax_{a \in \mathcal{A}} \frac{1}{M_u} \sum_{j=1}^{M_u} \left[ \tilde{R}^{(i,a,j)} + \gamma^T \tilde{V}^k \left(x_T^{(i,a, j)}, y_T^{(i,a, j)}\right)\right]$
	    where $x_T^{(i,a, j)}$, $y_T^{(i,a, j)}$, and $\tilde{R}^{(i,a,j)}$ are sampled from a $T$-step transition as above.
	}
	\caption{Frozen-State Fitted Value Iteration (FSFVI)}
	\label{alg:fsfvi}
\end{algorithm}

\section{Frozen-State Fitted Value Iteration under a Generative Model}
So far, we have operated in the dynamic programming setting, where the model (rewards and transitions) is assumed to be known. We now move on to a more practical and realistic setting where we assume access to a \emph{generative model} or \emph{simulator} of the environment (also referred to as a \emph{restart distribution}), an assumption frequently used in the RL literature \citep{kakade2002approximately,kearns2002sparse,munos2008finite,li2020breaking,agarwal2020model}. This setting is widely studied and captures many real-world use cases where it is possible to simulate environment transitions \citep{bellemare2012investigating,todorov2012mujoco,brockman2016openai,tassa2018deepmind}, without having full access to closed-form model dynamics. This is an important setting that occupies a middle ground between dynamic programming and traditional model-free RL.

Fitted value iteration (FVI) is a simple method that approximates the value iteration backup step using regression \citep{munos2008finite}, which we keep as an abstract subroutine using the notation $\texttt{Regress}\bigl(\{(s^{(i)}, y^{(i)})\}_{i=1}^N\bigr)$, which refers to the problem of regressing $\{y^{(i)}\}$ onto $\{s^{(i)} \}$. The $\texttt{Regress}$ procedure is meant to encode all of the information needed to perform the regression: basis functions, approximation architecture (e.g., linear model or neural network), optimization procedure (e.g., stochastic gradient descent), etc. In Algorithm \ref{alg:fvi}, we show the abstract version of FVI and in Algorithm \ref{alg:fsfvi}, we give the frozen-state extension, \emph{Frozen-state FVI} (FSFVI), which can be viewed as a practical extension of FSVI for the RL setting.

Additionally, in Appendix \ref{sec:avi}, we present an alternative version of FVI that is less practical, but more amenable to theoretical analysis, leading to a regret bound that we provide in Theorem \ref{thm:fsavi_regret}. These results are motivated by the approximate value iteration method first presented in \cite{tsitsiklis1996feature}. Given that both the method and analysis are technical and notation-heavy, we relegated the discussion to the Appendix.

\section{Numerical Experiments}
\label{sec:experiment}
For our numerical experiments, we operate in the simulator-based RL setting described above, where a generative model or simulator is available but the MDP is not explicitly known. Note that both FVI and FSFVI naturally apply in this setting, as expectations are approximated via sample averages using either $M$, $M_l$, or $M_u$ next-state samples. In addition, we introduce empirical counterparts to Base VI (Algorithm~\ref{alg:exact_vi}) and FSVI (Algorithm~\ref{alg:avi_appr}), which we refer to as Base Empirical-VI (Base E-VI) and Empirical-FSVI (E-FSVI), following the naming convention of \citet{haskell2016empirical}. These algorithms are identical to their exact counterparts, except that all expectations are replaced with sample averages. For completeness, we provide full specifications in Appendix~\ref{app:empirical_vi}.

In Section~\ref{sec:runningtime}, we analyzed the computational cost of FSVI under the assumption of a known model with exact expectation computations. In the RL setting, however, this definition requires adjustment. As a natural analogue, we use the number of \emph{value function evaluations} as a proxy for computational effort. This makes sense in the RL context because value-based methods, whether value iteration, Q-iteration, or their fitted variants, all rely on repeatedly evaluating a value function (or Q-function). These evaluations typically dominate the overall computation time, especially when using function approximation. Thus, counting value function evaluations provides a measure of computational cost that aligns closely with RL practice, while being closely related to the exact results from Section \ref{sec:runningtime}.

We compare to a range of algorithms. For purely tabular methods, we consider:
\begin{itemize}
    \item E-FSVI with $T\in \{3, 6, 12\}$: Empirical-FSVI, as described in Algorithm \ref{alg:e-vi_appr}. For each domain, we run E-FSVI for each $T$ value in order to illustrate the trade-offs.
    \item Base E-VI: Empirical-VI, as described in Algorithm \ref{alg:e-vi}, applied to the ``base'', non-hierarchical problem. We also use Base E-VI to compute an approximation of the optimal value.
    \item Base E-QI: Empirical Q-iteration is the analogue of E-VI using a Q-function representation. 
    \item Slow-Agn Base E-VI: Slow-agnostic Base E-VI, a simple baseline that pretends the slow state does not exist and assumes the value function only depends on the fast state.\footnote{To implement this method, we convert the ``full'' state to a ``partial'' state by deleting the slow state, and when needing to interact with the simulator, we convert from a ``partial'' state back to a ``full'' state by injecting a random value for the slow state. See Appendix \ref{sec:app:ignore} for further details.}
\end{itemize}
With function approximation, we test:
\begin{itemize}
    \item \{Lin, Deep\}-FSFVI: Our proposed frozen-state fitted value iteration, as introduced in Algorithm \ref{alg:fsfvi}, using either linear function approximator (Lin-FSFVI) or a three-layer neural network function approximator (Deep-FSVI).
    \item \{Lin, Deep\}-FVI: Fitted value iteration \citep{munos2008finite}  using either linear function approximator (Lin-FVI) or a three-layer neural network function approximator (Deep-FVI). This algorithm is also known as least-squares Monte Carlo, popularized by \cite{longstaff2001valuing} and frequently used in the operations literature (see, e.g., \cite{nadarajah2017comparison}). 
    \item \{Lin, Deep\}-FQI: Fitted Q-iteration \citep{ernst2005tree} using either linear function approximator (Lin-FQI) or a three-layer neural network function approximator (Deep-FQI).
\end{itemize}
For Lin-FSFVI and Lin-FVI, we use the raw state feature vectors as input to a linear function approximator, without applying any feature transformation or basis functions. For Deep-{FSFVI, FVI}, we input the raw state vector into a fully connected feedforward neural network with two hidden layers, each comprising 4 units and ReLU activations, followed by a final linear output layer. For Lin-FQI and Deep-FQI, the input is formed by concatenating the raw state vector with a one-hot encoding of the discrete action (a standard practice in RL for handling discrete action spaces). The architecture of the Deep-FQI network is otherwise identical to that of Deep-FVI. All models are trained using the mean squared error loss and optimized with the Adam optimizer \citep{kingma2014adam}. A detailed open-source implementation is also provided with this paper.

We test these algorithms across three domains, summarized in Table \ref{tab:domains}. Each domain uses a high discount factor of $\gamma = 0.995$, corresponding to an effective planning horizon of approximately $1/(1 - \gamma) = 200$ steps. We describe each experiment below; in each section, we use self-contained notation. Detailed experimental settings are given in Appendix \ref{app:exp_details}.

\renewcommand{\arraystretch}{1.5}  %

\begin{table}[h]
\small
\vspace{10pt}
\centering
\caption{Summary of Experimental Domains}
\label{tab:domains}
\begin{tabular}{@{}lrrrrl@{}}
\toprule
Domain & \multicolumn{1}{c}{$|\mathcal{S}|$} & \multicolumn{1}{c}{$|\mathcal{A}|$} & \multicolumn{1}{c}{$|\mathcal{S} \times \mathcal{A}|$} & Algorithms Tested \\
\midrule
Inventory control        & 561     & 11  & 6{,}171      & Base E-VI, E-FSVI, Base E-QI, Slow-Agn Base E-VI \vspace{5pt}\\
Spatial-task gridworld   & 4{,}356 & 32  & 139{,}392    & Base E-VI, E-FSVI, Base E-QI, Slow-Agn Base E-VI \vspace{5pt}\\ 
Dynamic pricing          & 15{,}150 & 21  & 318{,}150    & \makecell[r]{Lin-FSFVI, Lin-FVI, Lin-FQI,\\Deep-FSFVI, Deep-FVI, Deep-FQI} \\
\bottomrule
\end{tabular}
\end{table}

\subsection{Inventory Control with Fixed Order Costs}
We consider an infinite horizon inventory control problem \citep{porteus1990stochastic} with fixed order costs, lost sales, and limited storage \citep{zhao2007storage}. The state of the system is $s_t = (y_t, d_t)$, where $y_t \in \{0, 1, \ldots, y_\textnormal{max}\}$ (the fast state) is the inventory level and $d_t \in \mathcal D$ is the current demand (the slow state), which evolves exogenously over time. At the beginning of each period $t$, the decision maker observes the current state $s_t$ and chooses an order quantity $a_t \in \mathcal A_\textnormal{order} = \{0, a_1, \ldots, a_\textnormal{max}\}$. The cost incurred includes a fixed cost $K$ if $a_t > 0$, a per-unit ordering cost $c$, and a holding cost $h$ per unit of end-of-period inventory. Unsatisfied demand is lost. The transition equation is:
\[
y_{t+1} = \min\left\{ y_t + a_t - \min\{ y_t, d_{t+1} \},\; y_\textnormal{max} \right\}, \quad
d_{t+1} \sim P_d(\cdot \mid d_t),
\]
where $d_{t+1}$ is drawn from a distribution $P_d$ conditional on $d_t$. The (expected) reward in each period is given by:
\[
r(s_t, a_t) =  p \, \mathbb{E} \bigl(\min\{ y_t, d_{t+1} \} \bigr)- c \, a_t - K \, \mathbf{1}_{\{a_t \, > \, 0\}}, 
\]
which represents the revenue, variable order cost, and fixed order cost. The objective is to maximize infinite horizon discounted reward, with $\gamma = 0.995$. In our experiment, we set $y_\textnormal{max} = 50$, $\mathcal A_\textnormal{order} = \{0, 5, \ldots, 50\}$, $d_{t+1} = \textnormal{clip}_{[0, 50]}(d_t + \epsilon_{t+1})$, and $\{\epsilon_t\}$ are i.i.d. random variables that take value 0 with probability $0.8$ and $\pm 1$ with probability 0.1 each.\footnote{We define $\textnormal{clip}_{[a,b]}(x) = \max\{\min\{x, b\},a\}$.} We use $M = M_u = 50$ for all methods, and $M_l=1$ for E-FSVI because the lower-level problem is deterministic.

\begin{figure}[h!]
  \centering    
  \includegraphics[width=0.7\textwidth]{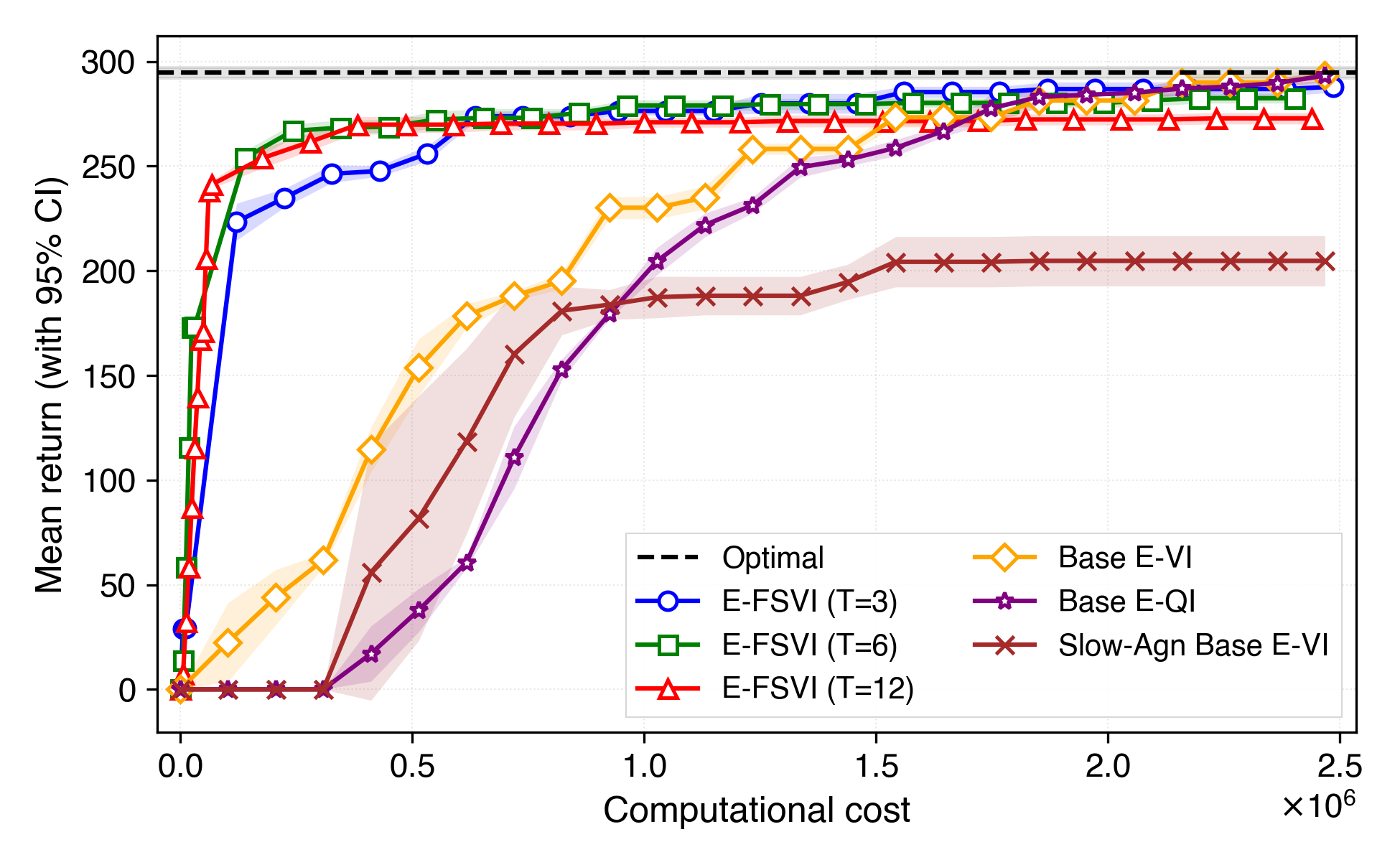}
  \caption{Mean return (with 95\% CI) versus computational cost for the inventory problem across five independent replications. For E-FSVI, we plot the performance of the policy after each lower-level value iteration and 3 times per upper-level value iteration. For the other methods, we plot 3 times per value iteration. To compute the optimal value, we ran Base E-VI for 50 iterations with $M=100$ samples per backup.}
  \label{fig:inventory}
\end{figure}

The results, shown in Figure \ref{fig:inventory}, demonstrate that frozen-state methods substantially outperform standard approaches in terms of computational efficiency. Specifically, E-FSVI with moderate values of \( T \), such as \( T = 6 \) and \( T = 12 \), achieves near-optimal return with significantly fewer value function evaluations compared to Base E-VI and E-QI. This supports the intuition that freezing the slow-moving demand state \( d_t \) over short planning horizons can preserve essential structure while simplifying the lower-level planning problem.

Among the E-FSVI variants, we find that \( T = 6 \) offers the best overall performance when considering both reward and computational cost. Larger values of \( T \), such as \( T = 12 \), initially converge faster but eventually suffer from increased approximation error due to longer freezing horizons. Conversely, if we look at the earlier part of the curves, E-FSVI with \( T = 3 \) seems to benefit less from the improved discount factor.

While Base E-VI and E-QI eventually reach optimal performance, they require significantly more computation. In contrast, the slow-agnostic baseline performs poorly, confirming that ignoring the slow state leads to suboptimal policies. These findings reinforce both our theoretical claims from Section \ref{sec:theory} and the trade-offs we observed in Figure \ref{fig:fsvi_vs_base_vi} from Section \ref{sec:interpreting}.

\subsection{Gridworld with Spatial Tasks}

\begin{figure}[htbp]
\centering
\begin{tikzpicture}[scale=0.7, >=Stealth]

  \draw[step=1cm, gray, very thin] (0,0) grid (11,11);

  \foreach \x in {0,...,10} {
    \node[below] at (\x+0.5, 0) {\x};
  }
  \foreach \y in {0,...,10} {
    \node[left] at (0, \y+0.5) {\y};
  }

  \newcommand{\drawtask}[8]{%
    \draw[->, ultra thick, #6]
      (#1+0.5,#2+0.5) to[bend left=#5] (#3+0.5,#4+0.5);

    \node[font=\small, fill=white, inner sep=1pt] at (#1+0.5, #2+0.5) [#8] {#7};
  }

  \drawtask{1}{7}{4}{10}{20}{black}{Task 1}{below};
  \drawtask{7}{9}{10}{6}{20}{black}{Task 2}{above};
  \drawtask{9}{3}{6}{0}{20}{black}{Task 3}{above};
  \drawtask{3}{1}{0}{4}{20}{black}{Task 4}{below};

  \drawtask{4}{9}{4}{7}{-20}{black}{Task 5}{right};
  \drawtask{6}{1}{6}{3}{-20}{black}{Task 6}{left};

  \drawtask{4}{6}{6}{6}{20}{black}{Task 7}{below left};
  \drawtask{6}{4}{4}{4}{20}{black}{Task 8}{above right};

\end{tikzpicture}
\caption{Spatial tasks on an 11$\times$11 gridworld. Start and end locations of each task are represented by the arrows. In order to receive a reward, the agent must pick up an object from the start location of the task and move it to the end location.}
\label{fig:gridworld-manual-labels}
\end{figure}

The main idea behind our second experiment domain is illustrated in Figure \ref{fig:gridworld-manual-labels} and is straightforward: an RL agent must navigate the gridworld and complete tasks. There are unlimited ``quantities'' of each task, which are labeled using arrows in the figure. The agent first must ``accept'' a task, after which they begin ``working'' on that task. To earn the ``completion'' reward, the agent must pick up an object from the start location of the task (beginning of the arrow) and bring it to the end location of the task. The agent also earns a small reward for picking up the object. The state of the agent is $s_t = (x_t, y_t, i_t, o_t, w_t)$, where $(x_t, y_t)$ are the coordinates of the agent, $i_t \in \{0, 1, \ldots, 8\}$ is the identity of the agent's current task (where $i_t=0$ means no current accepted task), $o_t \in \{0, 1\}$ is a binary state indicating whether the agent has picked up the object associated with the current task, and finally, $w_t \in \{0, 1\}$ is a reward signal indicating whether the high rewards are in the outer ring ($w_t=0$) or the inner ring ($w_t=1$). The agent's action is $a_t = (i'_t, d_t)$, where $i'_t \in \{1, 2, \ldots, 8\}$ is the identity of the task that the agent wishes to accept and $d_t$ is a cardinal direction to move. The rest of the environment works as follows:
\begin{itemize}
    \item If $i_t = 0$ (no current task), then $i_t = i'_t$. Otherwise, $i'_t$ takes no effect.
    \item If the agent reaches the start location of the current task $i_t$, then $o_t$ transitions from $0$ to $1$.
    \item  If the agent reaches the end location of the current task $i_t$ \emph{and} $i_t=1$, then the agent receives a reward $r(i_t, w_t)$.
\end{itemize}
The agent wishes to maximize infinite horizon discounted reward, with $\gamma = 0.995$. Again, we use $M = M_u = 50$ for all methods, and $M_l=1$ for E-FSVI.
We omit a detailed formulation of the MDP for succinctness, but give the full reward specification in Appendix \ref{app:gridworld}.

\begin{figure}[h!]
  \centering    
  \includegraphics[width=0.7\textwidth]{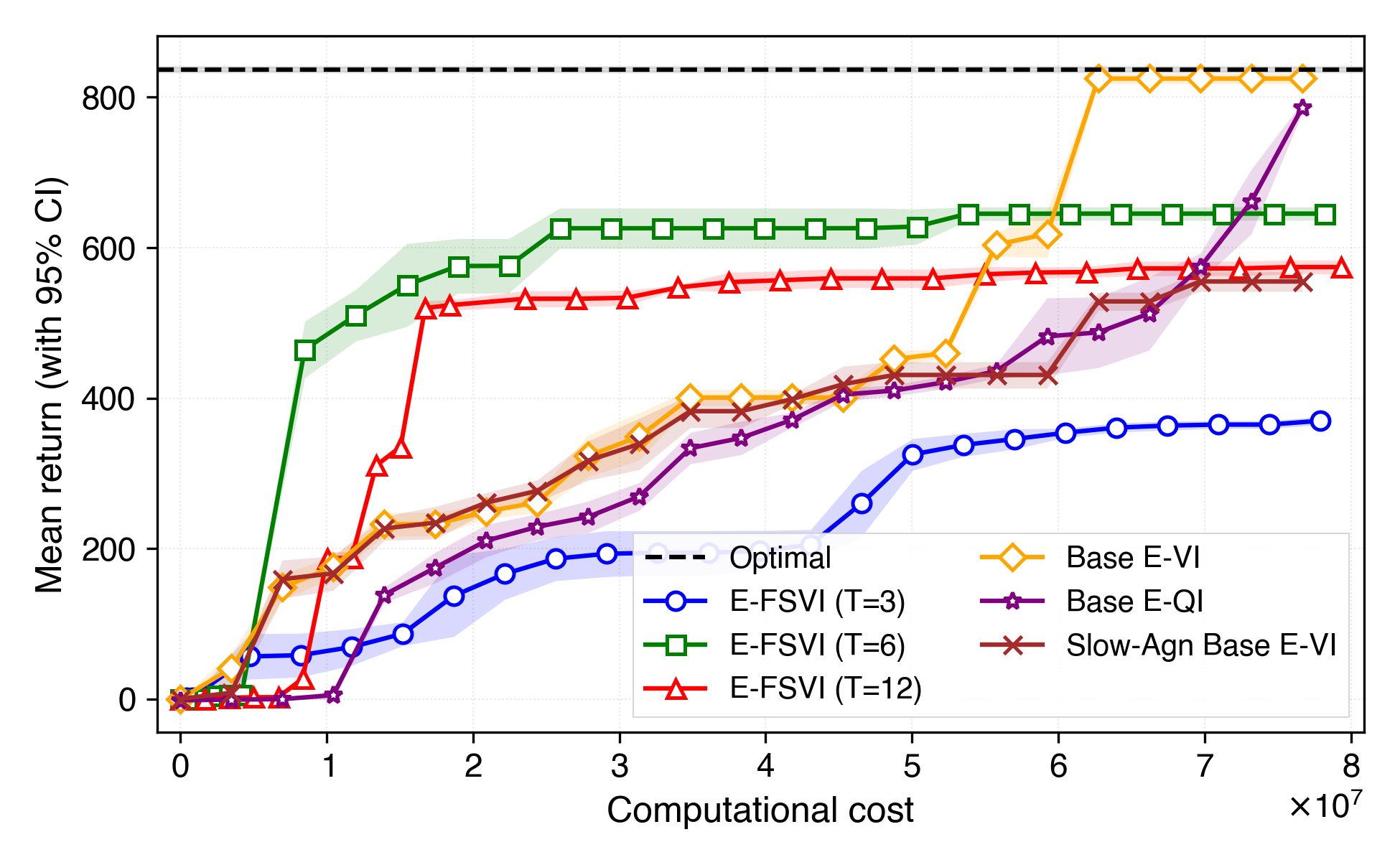}
  \caption{\looseness-1 Mean return (with 95\% CI) versus computational cost for the spatial gridworld problem across five independent replications. For E-FSVI, we plot the performance of the policy after each lower-level value iteration and 2 times per upper-level value iteration. For the other methods, we plot 2 times per value iteration. To compute the optimal value, we ran Base E-VI for 50 iterations with $M=100$ samples per backup.}
  \label{fig:spatial}
\end{figure}

The results are shown in Figure \ref{fig:spatial}. In contrast to the inventory control problem, the performance differences between E-FSVI algorithms with different values of $T$ are far more pronounced in this domain: the smallest $T$ value of $T = 3$ improves gradually and consistently, but is never quite able to achieve good performance. This effect is likely due to the nature of the tasks and the need to coordinate navigation across multiple steps. When $T$ is too small, the frozen lower-level policy is unable to fully capture meaningful navigation plans, leading to suboptimal behavior at the upper level. Interestingly, we observe that an intermediate value such as $T = 6$ provides a strong compromise: it allows the lower-level policy to encode useful subtask behavior while retaining enough flexibility for the upper level to guide high-level decisions. We also notice that Base VI and QI slowly improve and eventually reach optimality, but requires nearly 3x the computation to reach 75\% optimality that E-FSVI with $T=6$ is able to achieve very quickly.

This experiment underscores that in environments with complex dynamics, the choice of $T$ can be quite consequential.

\subsection{Dynamic Pricing with Reference Effects and Stochastic Consumer Behavior}
We consider a dynamic pricing problem where a seller interacts with a market of consumers who exhibit reference effects: that is, their purchase behavior depends not only on the current price but also on a history of past prices.  Reference prices represent consumers' internal price expectations for the product, based on historical prices. When the firm sets prices below (above) the consumer's reference price, there is a positive (negative) impact on demand. This is motivated by real-world e-commerce and retail environments where consumers are sensitive to perceived fairness or price changes over time. The model we consider is a slight extension of the one in \cite{chen2017efficient}. 

The state variable is $s_t = (p^\textnormal{ref}_t, \alpha_t, \eta_t^+, \eta_t^-)$, where $p^\textnormal{ref}_t$ is the reference price, $\alpha_t$ is the market's memory factor, and $(\eta_t^+, \eta_t^-)$ are the marginal reference price effect for gains and losses. The action is the price in the current period $p_t \in \mathcal P_t$, which induces a change in the reference price $p^\textnormal{ref}_{t+1} = \alpha_t \, p^\textnormal{ref}_{t} + (1-\alpha_t) \, p_t$. We model the demand as 
\[
D_t(p_t, p_t^\textnormal{ref}) = b_t - a_t \, p_t + \eta_t^+ \max \{p^\textnormal{ref}_{t} - p_t, 0\} + \eta_t^- \min\{p^\textnormal{ref}_{t} - p_t, 0\}.
\]
Consumer behavior is jointly described by $(\alpha_t, \eta_t^+, \eta_t^-)$, which evolve exogenously according to:
\[
(\alpha_{t+1}, \eta_{t+1}^+, \eta_{t+1}^-)  = \textnormal{clip}_{[(0.4,3,8), (0.9,7,12)]}((\alpha_t, \eta_t^+, \eta_t^-) + \epsilon_{t+1}),
\]
where $\textnormal{clip}$ is applied componentwise and $\epsilon_{t+1} \in \mathbb R^3$ and each component is i.i.d., taking value $0$ with 0.9 probability, $\pm 0.1$ with 0.05 probability each for the $\alpha_t$ dimension, and $\pm 1$ with $0.05$ probability each for the $\eta_t^+, \eta_t^-$ dimensions.
The reward function is $r(s_t, p_t) = p_t \, D_t(p_t, p_t^\textnormal{ref}).$

This problem is discretized as follows. The possible values of the action $p_t$ are  $\{0, 0.5, 1.0, \ldots, 10.0\}$, and the possible values of the reference price are $p_t^\textnormal{ref}$ is $\{0, 0.1, 0.2, \ldots, 10.0\}$. For $\alpha_t, \eta_t^+, \eta_t^-$, the possible values are $\{0.4, 0.5, 0.6, 0.7\}$, $\{3, 4, 5, 6,7\}$, and $\{8, 9, 10, 11, 12\}$.

\begin{figure}[h!]
  \centering    
  \includegraphics[width=0.7\textwidth]{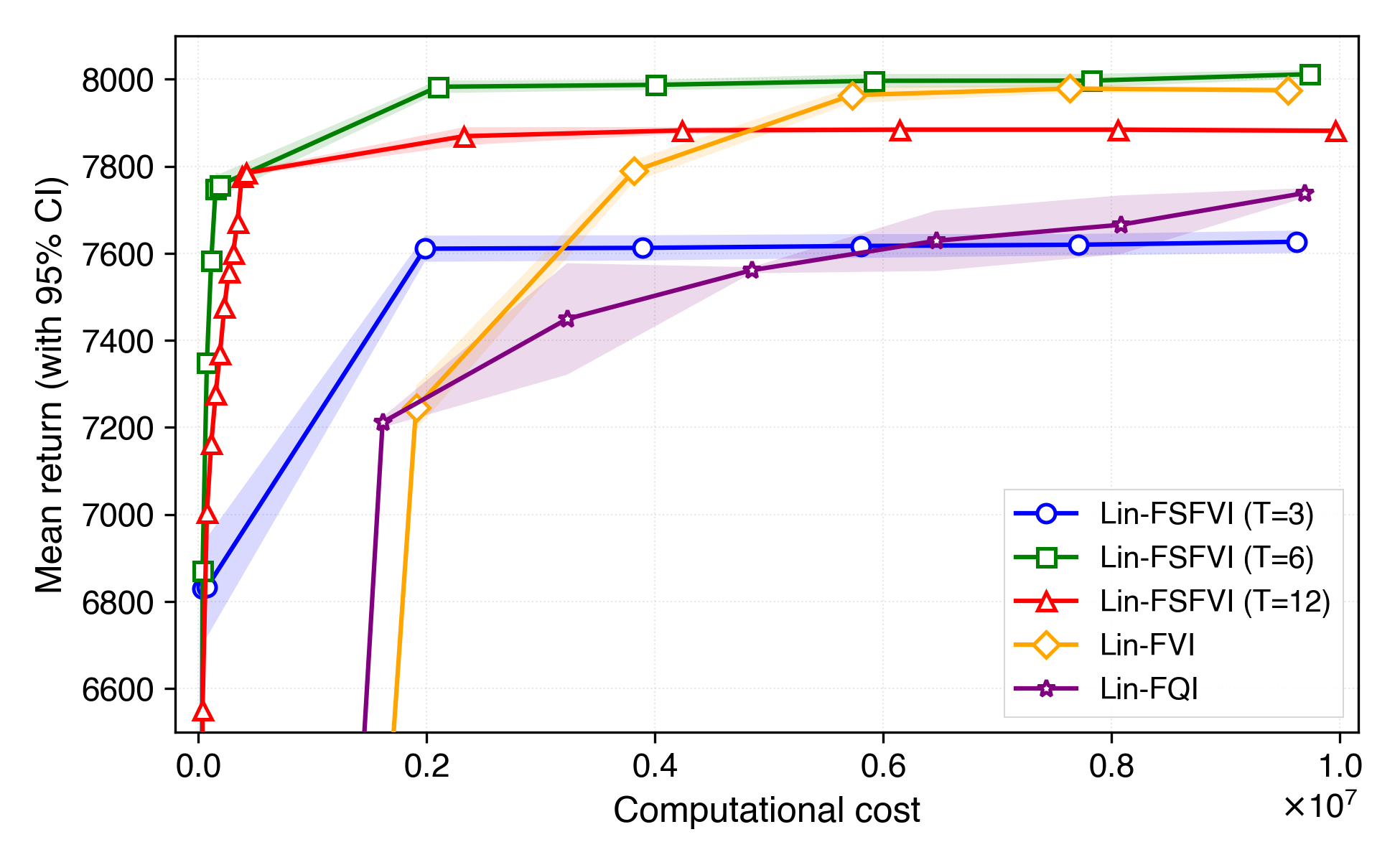}
  \caption{Mean return (with 95\% CI) versus computational cost for the dynamic pricing problem using linear architectures across five independent replications. For E-FSVI, we plot the performance of the policy after each lower-level value iteration and once per upper-level value iteration. For the other methods, we plot once per value iteration.}
  \label{fig:pricing_lin}
\end{figure}

\begin{figure}[h!]
  \centering    
  \includegraphics[width=0.7\textwidth]{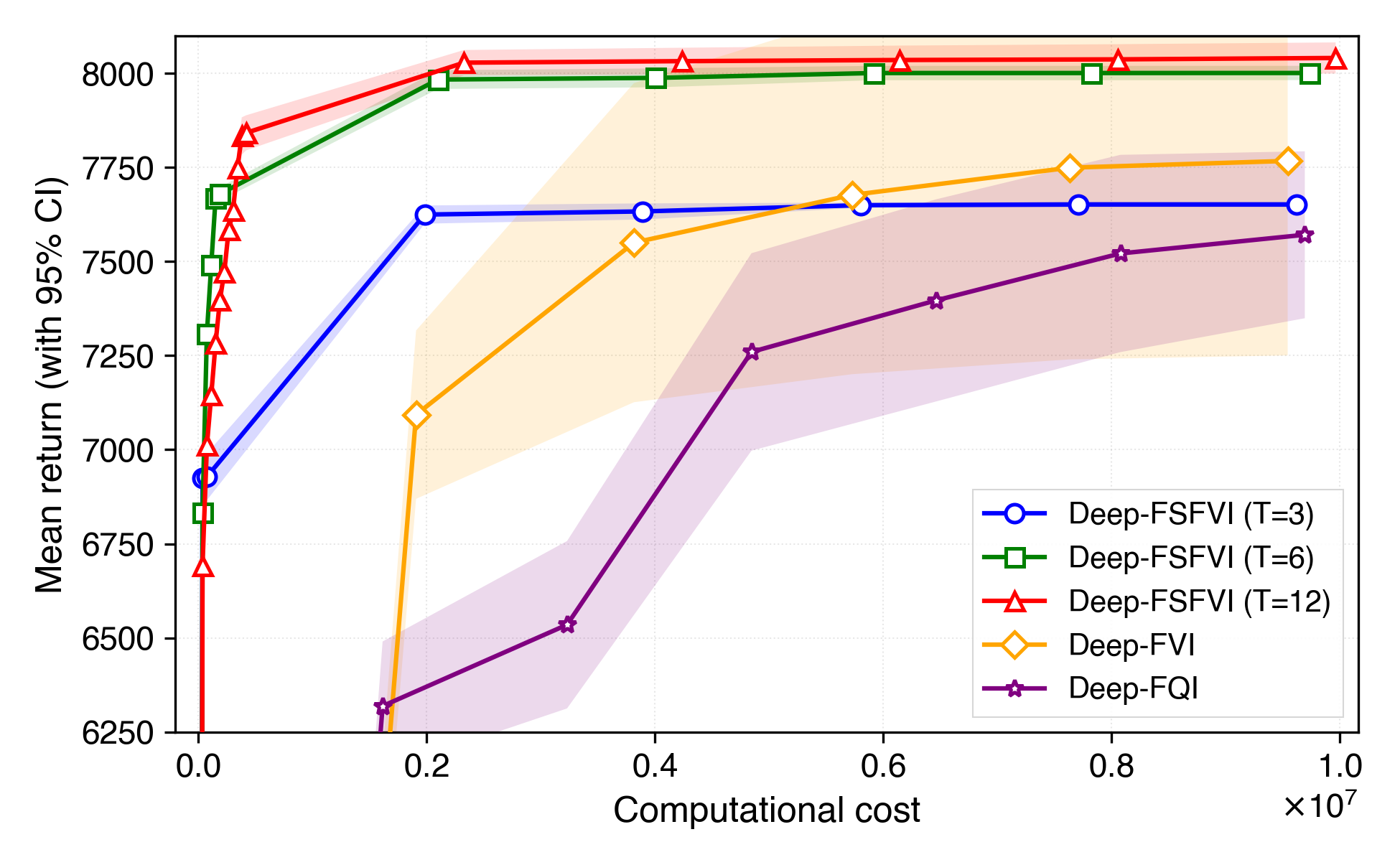}
  \caption{Mean return (with 95\% CI) versus computational cost for the dynamic pricing problem using linear network architectures across five independent replications. For E-FSVI, we plot the performance of the policy after each lower-level value iteration and once per upper-level value iteration. For the other methods, we plot once per value iteration.}
  \label{fig:pricing_deep}
\end{figure}

\looseness-1 As shown in Table \ref{tab:domains}, this is the largest problem out of the three. We therefore use this domain to test the ``fitted'' variants of our algorithms. We set $N$, the number of state samples, to be 10\% of the size of the state space. As before, we use $M = M_u = 50$ for all methods, and $M_l=1$ for E-FSVI.

Figures \ref{fig:pricing_lin} and \ref{fig:pricing_deep} compare the performance of frozen-state and baseline methods under two different architectures: linear function approximation (Figure~\ref{fig:pricing_lin}) and neural network function approximation (Figure~\ref{fig:pricing_deep}). We point out several takeaways:

\begin{enumerate}
    \item \looseness-1 Frozen-state methods improve upon baselines in both settings, but the gains are larger in the setting using neural network function approximation. This is largely due to the relatively poor performance of Deep-FVI and Deep-FQI, which appear to suffer from higher variance and instability during training. We hypothesize that this is due to the difficulty of fitting high-capacity models under long-horizon dynamics, where bootstrapping can amplify noise. In contrast, the structured nature of Deep-FSFVI, specifically the reduced effective discount factor and frozen slow states, likely contributes to improved stability and more reliable learning.

    \item Again, we see that the $T=3$ version of Deep-FSFVI performed poorly under both architectures, but $T=6$ and $T=12$ were adequate.
\end{enumerate}
Overall, the results suggest that frozen-state methods are broadly effective, and surprisingly, seem to enjoy some additional robustness even in the face of neural network architectures that can be difficult to fit.

\section{Conclusion}
In this paper, we studied a new class of MDPs with a type of structure called \emph{fast-slow} structure, motivated by practical applications where some components of the state evolve slowly over time. Based on this structure, we propose a set of new algorithms based on the idea of \emph{freezing the slow state} for several periods at a time to ease the computational burden of value iteration algorithms. We analyze the regret of the resulting policy using a novel analysis of various Bellman operators. Empirically, on three example applications, we show that our new frozen-state methods converge significantly faster to good policies than standard methods, and notably, ignoring the slow state leads to low-performing policies. Therefore, our method can be viewed as a viable compromise between solving the full MDP (often computationally intractable) and completely ignoring states during the modeling process (computationally easy but potentially highly suboptimal in the true model).

\section*{Funding}
This research is based upon work supported by the U.S. National Science Foundation under Grant No. 1807536.

\newpage
\bibliographystyle{abbrvnat}
{\small 
	\bibliography{bibliography}}

\clearpage
\newpage
\appendix

\begin{center}
    {\huge Online Supplement:\\Faster Reinforcement Learning by Freezing Slow States}\\[1.5em]
    {\Large Yijia Wang$^1$ and Daniel R. Jiang$^{2,1}$\\[0.5em]
    $^1$University of Pittsburgh, $^2$AI at Meta}
\end{center}
\bigskip\bigskip

\section{Empirical Versions of VI and FSVI}
\label{app:empirical_vi}

\begin{algorithm}
	\SetKwInput{Input}{Input}
	\SetKwInput{Output}{Output}
	\BlankLine
	\Input{Initial values $U^0$, number of iterations $k$, and number of next state samples $M$.}
		\medskip
	\Output{Approximation to the optimal policy $\nu^k$.}
	\BlankLine
	\For{$k' = 1, 2, \ldots, k $}{
	    \vspace{0.5em}
	    \For{$s$ in the state space $\mathcal{S}$}{
	        \vspace{0.5em}
    $\displaystyle U^{k'}(s) = \max_{a} r(s,a) + \gamma\, \frac{1}{M} \sum_{j=1}^M \, U^{k'-1}(s^{(a,j)})$, where $s^{(a,j)}$ is sampled from the generative model, conditional on state $s$ and action $a$.
        	\vspace{0.5em}
    	}
	}
	\vspace{0.5em}
	\For{$s$ in the state space $\mathcal{S}$}{
        \vspace{0.5em}
    	$\displaystyle \nu^{k}(s) = \argmax_{a} r(s,a) + \gamma\, \frac{1}{M} \sum_{j=1}^M \, U^{k}(s^{(a,j)})$, where $s^{(a,j)}$ is sampled from the generative model, conditional on state $s$ and action $a$. 
    	\vspace{0.5em}
	}
	\caption{Empirical VI for the Base Model (Base E-VI)}
	\label{alg:e-vi}
\end{algorithm}%

\begin{algorithm}
	\SetKwInput{Input}{Input}
	\SetKwInput{Output}{Output}
	\BlankLine
	\Input{Initial values $J_T^* \equiv 0$ and $V^0$, number of iterations $k$, and number of next-state samples $M_l$ (lower level), $M_u$ (upper level).}
		\medskip
	\Output{Approximation of the $T$-periodic frozen-state policy $(\tilde{\mu}^k, \tilde{\boldsymbol{\pi}}^*)$ and $J_1^*$.}
	\BlankLine
    	\For{each slow state $x \in \mathcal{X}$}{
		\vspace{0.5em}
        	\For{$t = T-1, T-2, \ldots, 1$}{
		\vspace{0.5em}
		    \For{each fast state $y \in \mathcal{Y}$}{
        		\vspace{0.5em}
        		$\displaystyle J^*_t(x,y) = \max_a r(x,y,a) + \gamma  \,  \frac{1}{M_l} \sum_{j=1}^{M_l} J^*_{t+1}(x,y^{(a,j)})$, where $y^{(a,j)}$ is sampled from the generative model, conditional on state $(x, y)$ and action $a$. \\

        		\vspace{0.5em}
        		$\displaystyle \tilde{\pi}_t^*(x,y) = \argmax_a \, r(x,y,a) + \gamma \, \frac{1}{M_l} \sum_{j=1}^{M_l} J^*_{t+1}(x,y^{(a,j)})$, with $y^{(a,j)}$ defined the same way as above.
        		\vspace{0.5em}
		    }
		}
	}
	\BlankLine
	\For{$k' = 1, 2, \ldots, k $}{
	    \vspace{0.5em}
	    \For{$s_0=(x_0,y_0)$ in the state space $\mathcal{X}\times \mathcal{Y}$}{
	        \vspace{0.5em}
        	$\displaystyle V^{k'}(x_0,y_0,J^*_1,\tilde{\boldsymbol{\pi}}^*) = \max_{a} \, \frac{1}{M_u} \sum_{j=1}^{M_u} \, \Bigl[\tilde{R}^{(a,j)} + \gamma^T V^{k'-1}(x^{(a,j)}_T, y^{(a,j)}_T,J^*_1,\tilde{\boldsymbol{\pi}}^*)\Bigr]$, where $x_T^{(a,j)}$, $y_T^{(a,j)}$, and $\tilde{R}^{(a,j)}$ are sampled using a $T$-step transition through the generative model, conditioned on $x_0$, $y_0$, $a$, $\tilde{\boldsymbol{\pi}}^*$ and $J^*_1$. 
        	\vspace{0.5em}
           
    	}
	}
	\vspace{0.5em}
	\For{$s_0=(x_0,y_0)$ in the state space $\mathcal{X}\times \mathcal{Y}$}{
        \vspace{0.5em}
    	$\displaystyle \tilde{\mu}^k(x_0,y_0) = \argmax_{a} \frac{1}{M_u} \sum_{j=1}^{M_u} \Bigl[\tilde{R}^{(a,j)} + \gamma^T V^k(x^{(a,j)}_T, y^{(a,j)}_T,J^*_1,\tilde{\boldsymbol{\pi}}^*)\Bigr]$, where $x_T^{(a,j)}$, $y_T^{(a,j)}$, and $\tilde{R}^{(a,j)}$ are sampled as above, conditioned on $x_0$, $y_0$, $a$, $\tilde{\boldsymbol{\pi}}^*$ and $J^*_1$.
    	\vspace{0.5em}
	}

	\caption{Empirical Frozen-State Value Iteration (E-FSVI)}
	\label{alg:e-vi_appr}
\end{algorithm}%

\clearpage

\section{Implementation Details for Experimental Results}
\label{app:exp_details}

\subsection{Ignoring the Slow State}
\label{sec:app:ignore}
For Slow-Agn Base VI, our goal is to simulate the situation where a decision maker neglects to consider the slow state during modeling. In other words, we model the problem as if the state consists only of the fast component, intentionally ignoring the slow state altogether. We implement this in our experiments as follows:
\begin{itemize}
    \item \textbf{Full-to-partial state conversion.} The environment always operates on the full state space, which includes both fast and slow components. However, the Slow-Agn E-VI algorithm only observes the fast component (i.e., a partial state). Whenever the environment returns a state to the algorithm, it extracts and passes only the fast component.
    \item \textbf{Partial-to-full state conversion.} When the algorithm queries the environment (e.g., for rewards or next states) using a partial state and an action, we construct a corresponding full state by randomly sampling a slow state uniformly and appending it to the given fast component. This synthesized full state is then used by the environment to transition forward.
\end{itemize}

\subsection{Gridworld with Spatial Tasks}
\label{app:gridworld}
We describe the rewards used in the environment. The agent receives a reward of 2.0 for picking up the object for every task. We specify the completion reward $r(i,w)$ for task $i$ in reward state $w$ as follows. For $i =1,\ldots,4$, we have $(r(i,0), r(i,1)) = (80,6)$; for $i=5, 6$, we have $(r(i,0), r(i,1)) = (1,1)$; and for $i=7,8$, we have $(r(i,0),r(i,1)) = (2,30)$. Note that the high rewards switch from the outer ring to inner ring depending on the value of $w_t$, which changes state from either $0$ to $1$ or $1$ to $0$ with probability 0.02.

\section{Value Function Approximation with a Linear Architecture} \label{sec:avi}

In this section, we explore the use of a linear architecture for a more compact representation of the value function. We show how \emph{approximate} VI (AVI) can be combined with the frozen-state approximation, resulting in an algorithm that can scale to fast-slow MDPs with much larger state spaces. The form of AVI that we use is based on the technique first proposed in \cite{tsitsiklis1996feature} and later also used in \cite{zanette2019limiting}. A technical contribution we make here is to prove error bounds when this approximation architecture is used in a hierarchical setting that combines finite-horizon and infinite-horizon components (i.e., our frozen-state VI).

\subsection{The Approximation Architecture}
Let $\boldsymbol{\phi}(s) = \bigl(\phi_1(s), \phi_2(s), \ldots, \phi_M(s)\bigr)^\intercal \in \R^M$ be an $M$-dimensional feature vector evaluated at state $s \in \mathcal S$. An approximation $\{\hat{J}_t(\boldsymbol{\omega_t})\}_{t=1}^{T}$ of the lower-level value functions $\{J_t^*\}_{t=1}^{T}$ of the frozen-state approximation is given by a sequence of parameter vectors $\{\boldsymbol{\omega}_t\}_{t=1}^T$ with $\boldsymbol{\omega}_t \in \mathbb R^M$, where the component of $\hat{J}_t(\boldsymbol{\omega_t})$ associated with $s$ is given by
\begin{equation*}
    \hat{J}_t(s,\boldsymbol{\omega}_t) = \boldsymbol{\phi}^\intercal(s) \, \boldsymbol{\omega}_t, \quad \text{for } t = 1, 2, \ldots, T.
\end{equation*}
(Note that since $J_T^* \equiv 0$, we can set $\boldsymbol{\omega}_T = \boldsymbol{0}$.) The approximation $\hat{V}(\boldsymbol{\beta})$ to the upper-level value function (of the frozen-state model) $V^*(J_1^*,\tilde{\boldsymbol{\pi}}^*)$ is given by a parameter vector $\boldsymbol{\beta} \in \mathbb R^M$, where
\begin{equation*}
    \hat{V}(s,\boldsymbol{\beta}) = \boldsymbol{\phi}^\intercal(s) \, \boldsymbol{\beta}.
\end{equation*}
For simplicity, we have used the same features in both the upper and lower levels.

It is well-known that naive specifications of approximate value iteration applied to linear architectures can produce divergent behavior \citep{bertsekas1996neuro}. To circumvent this potential issue, our algorithmic approach depends on a set of \emph{pre-selected} states $\tilde{\mathcal{S}} = \{s_1,s_2,\ldots,s_M\}$, an idea popularized in \cite{tsitsiklis1996feature}, who showed that if certain assumptions on these states and the feature vectors are satisfied, then divergence is avoided. A similar algorithm is also described more recently in \cite{zanette2019limiting}. 

In this section, we need an ordering of the state space, so without loss of generality, we assume that $\mathcal S = \{1, 2, \ldots, N\}$ and that the first $M$ are the pre-selected states, i.e., $s_m = m$ for $m = 1, 2, \ldots, M$. 
We make the following assumption on the feature vectors; this is essentially Assumption 2 of \cite{tsitsiklis1996feature}, adapted to our setting.

\begin{assumption}\label{assumption.basis_function}
Let $\tilde{\mathcal{S}} = \{s_1,s_2,\ldots,s_M\}$ be a set of pre-selected anchor states. Suppose the following conditions on the features $\boldsymbol{\phi}$ are satisfied.
    
    \begin{enumerate}
        \item The vectors $\boldsymbol{\phi}(s_1), \boldsymbol{\phi}(s_2), \ldots, \boldsymbol{\phi}(s_M)$ are linearly independent.
        
        \item There exists some $\gamma' \in [\gamma, 1)$ such that for any state $s\in \mathcal{S}$, there are coefficients $\theta_m(s) \in \R$ for $m=1,2, \ldots, M$  satisfying:
        \[
            \sum_{m=1}^M |\theta_m(s)| \leq 1 \quad \text{and} \quad \boldsymbol{\phi}(s) = \frac{\gamma'}{\gamma} \sum_{m=1}^M \theta_m(s) \, \boldsymbol{\phi}(s_m).
        \]
    \end{enumerate}
The interpretation of this assumption is that the feature space $\{\phi(s) \, | \, s \in \mathcal S\}$ lies in the convex hull of the points defined by the pre-selected states: $\bigl\{\pm (\gamma' / \gamma) \,  \phi(s_m)\bigr\}_{m=1}^M$. To reduce notional clutter, we will define $\kappa = \gamma' / \gamma$ to be the \emph{amplification factor} induced by the features.
\end{assumption}

Following \cite{tsitsiklis1996feature}, we let $\Phi \in \R^{N\times M}$ be a matrix with the $s$-th row equal to $\boldsymbol{\phi}^\intercal (s)$ and let $L\in \R^{M\times M}$ be a matrix with the $m$-th row equal to $\boldsymbol{\phi}^\intercal (s_m)$. If we let $G$ be the remaining rows of $\Phi$, then we see that $\Phi = \bigl[L; G\bigl]$. Next, by Assumption \ref{assumption.basis_function}, the matrix $L$ has a unique matrix inverse $L^{-1} \in \R^{M\times M}$. We define $\Phi^{\dagger} \in \R^{M\times N}$ as follows: for $m \in \{1, 2, \ldots, M\}$, suppose the $m$-th column of $\Phi^{\dagger}$ is equal to the $m$-th column of $L^{-1}$, and let all the other entries of $\Phi^{\dagger}$ be zero. In other words, $\Phi^{\dagger} = [L^{-1} \; 0]$. Therefore, we see that $\Phi^{\dagger}$ is a left inverse of $\Phi$:
\begin{equation*}
    \Phi^{\dagger} \Phi = [L^{-1} \; 0] \begin{bmatrix}L\\ G\end{bmatrix} = L^{-1} L = \mathbf{I},
\end{equation*}
where $\mathbf{I} \in \R^{M\times M}$ is the identity matrix.

\subsection{Frozen-State Approximate Value Iteration}
Recall the lower-level Bellman operator $\bar{H}$ from (\ref{eq.Vlower_operator}) and the upper-level Bellman operator $F_{\hat{J}_1, \hat{\boldsymbol{\pi}}}$ defined in (\ref{eq.Vupper_operator}). The high-level idea behind our new approach, \emph{frozen-state approximate value iteration} (FSAVI) is as follows:
\begin{itemize}
\item \textbf{Lower-level AVI.} We first run approximate value iteration (under basis functions $\Phi$) for the lower-level problem. Letting $\boldsymbol{\omega}_T^* = \mathbf{0}$, the parameter $\boldsymbol{\omega}^*_t$ is estimated by first evaluating $\bar{H} \hat{J}_{t+1}(\boldsymbol{\omega}_{t+1}^*) $ at the pre-selected states, and then computing $\boldsymbol{\omega}_t^*$ so that $\hat{J}_t(s,\boldsymbol{\omega}_t^*) = \bigl(\bar{H} \hat{J}_{t+1}(\boldsymbol{\omega}_{t+1}^*) \bigr) (s)$ for $s\in \tilde{\mathcal{S}}$.
\item \textbf{Upper-level AVI.} Suppose that after solving the lower level, we have parameter vectors $\boldsymbol{\omega}^* = (\boldsymbol{\omega}_1^*, \boldsymbol{\omega}_2^*, \ldots, \boldsymbol{\omega}_T^*)$, implying lower-level value functions $\hat{J}_t(\boldsymbol{\omega}_t^*) = \Phi \boldsymbol{\omega}_t^*$ and an associated greedy policy $\hat{\boldsymbol{\pi}}(\boldsymbol{\omega}^*) = (\hat{\pi}_1(\boldsymbol{\omega}^*), \ldots, \hat{\pi}_{T-1}(\boldsymbol{\omega}^*))$:
\begin{equation}
\hat{\pi}_t(x, y, \boldsymbol{\omega}^*) = \argmax_{a} \; r(x,y,a) + \gamma \, \E \bigl[\hat{J}_{t+1}(x,y', \boldsymbol{\omega}_{t+1}^*)\bigr].
\label{eq.greedy_linear}
\end{equation}
For the upper level, the parameter $\boldsymbol{\beta}_k$ is updated to $\boldsymbol{\beta}_{k+1}$ in iteration $k+1$ by first evaluating $F_{\hat{J}_1(\boldsymbol{\omega}_1^*), \hat{\boldsymbol{\pi}}(\boldsymbol{\omega}^*)} \hat{V} (\boldsymbol{\beta}_{k})$ at the pre-selected states, then computing $\boldsymbol{\beta}_{k+1}$ so that $\hat{V}(s,\boldsymbol{\beta}_{k+1}) = \bigl(F_{\hat{J}_1(\boldsymbol{\omega}_1^*), \hat{\boldsymbol{\pi}}(\boldsymbol{\omega}^*)} \hat{V} (\boldsymbol{\beta}_{k})\bigr)(s)$ for $s\in \tilde{\mathcal{S}}$. Note that, taking $\Phi$ as fixed, the dependence of the upper level on the lower level can be represented succinctly through $\boldsymbol{\omega}^*$. Therefore, we will use the simplified notation $F_{\boldsymbol{\omega}} := F_{\hat{J}_1(\boldsymbol{\omega}_1), \hat{\boldsymbol{\pi}}(\boldsymbol{\omega})}$ going forward.
\end{itemize}
To start, we define two new Bellman operators for the parameter space:
\begin{equation*}
    \bar{H}' = \Phi^{\dagger} \circ \bar{H} \circ \Phi \quad \text{and} \quad F'_{\boldsymbol{\omega}} = \Phi^{\dagger} \circ F_{\boldsymbol{\omega}} \circ \Phi.
\end{equation*}
To understand the definition of $H'$, consider the lower level. Suppose we start with a parameter vector $\boldsymbol{\omega}_{t+1}^*$, representing an approximate value function at time period $t+1$ given by $\hat{J}_{t+1}(\boldsymbol{\omega}_{t+1}^*) = \Phi \boldsymbol{\omega}_{t+1}^*$. The update to the next parameter vector $\boldsymbol{\omega}_t^*$ is obtained by applying $\bar{H}$ to $\hat{J}_{t+1}(\boldsymbol{\omega}_{t+1}^*)$, as we would normally do, and then applying $\Phi^{\dagger}$ to project back to the parameter space. For the upper level, a similar logic holds to go from $\boldsymbol{\beta}_i$ to $\boldsymbol{\beta}_{i+1}$: we first have the approximate upper-level value function $\hat{V}(\boldsymbol{\beta}_{i}) = \Phi \boldsymbol{\beta}_k$ and then apply the normal Bellman update $F_{\boldsymbol{\omega}^*}$, before lastly obtaining the updated parameter $\boldsymbol{\beta}_{i+1}$ using $\Phi^{\dagger}$. Therefore, we have 
\begin{equation*}
\boldsymbol{\omega}_t^* = \bar{H}' (\boldsymbol{\omega}_{t+1}^*) \quad \text{and} \quad \boldsymbol{\beta}_{i+1} = F'_{\boldsymbol{\omega}^*} (\boldsymbol{\beta}_{i}).
\end{equation*}
We show in the Lemma \ref{lemma.upper_Hprime_contraction} of the Appendix that $F_{\boldsymbol{\omega}^*}'$ is an $(\kappa \gamma^{T})$-contraction in the norm $\|\cdot\|_\Phi$ on $\R^M$ defined by $\|\boldsymbol{\beta}\|_\Phi = \| \Phi \boldsymbol{\beta} \|_\infty$ and therefore has a fixed point $\boldsymbol{\beta}^*$. We now define two quantities related to the approximation error of the linear architecture.
\begin{definition}

Define the linear architecture approximation error for the lower level as
\begin{equation}
    \varepsilon_{\textnormal{low}} = \textstyle \max_{t \in \{1, 2, \ldots, T\}} \, \inf_{\boldsymbol{\omega}_t \in \mathbb R^M} \, \bigl \| J^*_t - \hat{J}_t(\boldsymbol{\omega}_t) \bigr \|_\infty.
    \label{eq:eps_low}
\end{equation}
Let $V^*_{\boldsymbol{\omega}}$ be the fixed point of $F_{\boldsymbol{\omega}}$. For the upper level, we define error $\epsilon_{\textnormal{up}}$ as
\begin{equation}
    \varepsilon_{\textnormal{up}} = \textstyle \sup_{\boldsymbol{\omega}} \, \inf_{\boldsymbol{\beta} \in \mathbb{R}^M} \bigl \| V^*_{\boldsymbol{\omega}} - \hat{V}(\boldsymbol{\beta}) \bigr \|_\infty.
    \label{eq:eps_high}
\end{equation}
Both (\ref{eq:eps_low}) and (\ref{eq:eps_high}) are related to the approximation errors defined in \cite{tsitsiklis1996feature}; moreover, taking a uniform bound over quantities that need to be approximated resembles the errors defined in \cite{munos2008finite}.
\end{definition}

\begin{algorithm}
	\SetKwInput{Input}{Input}\SetKwInput{Output}{Output}
	\Input{$\tilde{\mathcal{S}} = \{s_1, s_2, \ldots, s_M\}$, $\boldsymbol{\phi}$, initial weights $\boldsymbol{\omega}_T = \boldsymbol{\beta}_0 = \mathbf{0}$, number of iterations $k$.
	}
	\medskip
	\Output{Approximation of the $T$-periodic frozen-state policy $\bigl(\hat{\mu}_{(\boldsymbol{\beta}^k,\, \boldsymbol{\omega}^*)}, \hat{\boldsymbol{\pi}}_{\boldsymbol{\omega}^*}\bigr)$ and $\hat{J}_1(\boldsymbol{\omega}^*)$} %
	\BlankLine
	\For{$t = T-1, T-2, \ldots, 1$}{
		\vspace{0.5em}
		\For{each pre-selected state $s = (x, y) \in \tilde{\mathcal{S}}$}{
    		\vspace{0.5em}
    		$J_t(x,y) = \max_{a} r(x, y, a) + \gamma \, \E \bigl[ \hat{J}_{t+1}(x, f_\mathcal{Y}(x,y,a,w), \boldsymbol{\omega}_{t+1}) \bigr]$.\\
    		\vspace{0.5em}
		}
		\vspace{0.5em}
		Set remaining entries of $J_t$ to zero. Update parameter vector: $\boldsymbol{\omega}_t^* = \Phi^\dagger J_t$.\\
	}
    	Let $\hat{\boldsymbol{\pi}}_{\boldsymbol{\omega}^*}$ be greedy with respect to $\hat{J}_t(\boldsymbol{\omega}_t^*) = \Phi \boldsymbol{\omega}_t^*$, similar to \eqref{eq.greedy_linear}.

	\BlankLine
	\For{$i=1,2,\ldots,k$}{
		\vspace{0.5em}
		\For{each pre-selected state $s_0 \in \tilde{\mathcal{S}}$}{
			\vspace{0.5em}
			$V^i(s_0) = \max_a \E  \bigl[ \tilde{R}(s, a, \hat{J}_1(\boldsymbol{\omega}_1^*)) + \gamma^T \, \hat{V}(s_T(a,\tilde{\boldsymbol{\pi}}_\text{avi}) , \boldsymbol{\beta}_{i-1} )\bigr].$\\ \vspace{0.5em}
			Set remaining entries of $V^i$ to zero. Update parameter vector: $\boldsymbol{\beta}_i = \Phi^\dagger \, V^{i}$.\\
		}
	}
	\For{$s_0$ in the state space $\mathcal{S}$}{
        \vspace{0.5em}
    	$\hat{\mu}_{(\boldsymbol{\beta}^k,\, \boldsymbol{\omega}^*)}(s_0) = \argmax_{a} \E\bigl[\tilde{R}(s_0, a, \hat{J}_1(\boldsymbol{\omega}_1^*)) + \gamma^T
    	\hat{V}(s_T(a,\tilde{\boldsymbol{\pi}}_{\boldsymbol{\omega}^*}) ,\boldsymbol{\beta}_k)\bigr]$. 
    	\vspace{0.5em}
	}
	\caption{Frozen-State Approximate Value Iteration (FSAVI)}
	\label{alg:avi_appr}
\end{algorithm}\DecMargin{1em}

\begin{theorem} \label{thm:fsavi_regret}
    Let $(\hat{\mu}_{(\boldsymbol{\beta}^k,\, \boldsymbol{\omega}^*)}, \hat{\boldsymbol{\pi}}_{\boldsymbol{\omega}^*})$ be the result after running FSAVI for $k$ iterations for a given $\tilde{\mathcal{S}}$ and $\boldsymbol{\phi}$. The regret incurred when running $(\hat{\mu}_{(\boldsymbol{\beta}^k,\, \boldsymbol{\omega}^*)}, \hat{\boldsymbol{\pi}}_{\boldsymbol{\omega}^*})$ in the base model satisfies
    \begin{align*}
        &\mathcal R\bigl(\hat{\mu}_{(\boldsymbol{\beta}^k,\, \boldsymbol{\omega}^*)}, \hat{\boldsymbol{\pi}}_{\boldsymbol{\omega}^*}\bigr) \leq \biggl(\frac{2\gamma^T}{(1-\gamma^T)^2} + \frac{2}{1-\gamma^T}\biggr) \epsilon_{r, \textnormal{avi}}(\gamma,\alpha,d_\mathcal{Y},\mathbf{L},T, \gamma', \varepsilon_\textnormal{low})\\
&+ \biggl(\frac{2\gamma^{2T}}{(1-\gamma^T)^2} + \frac{2\gamma^T}{1-\gamma^T}\biggr) L_U \, d(\alpha,d_\mathcal{Y},T)+ \biggl(\frac{1 + \kappa}{1- \kappa \gamma^T} \biggr)\,\varepsilon_{\textnormal{up}} + (\kappa \gamma^T)^k \biggl( \frac{\kappa^2 - \kappa^2 (\kappa\gamma)^{T+1}}{(1-\kappa \gamma^T) (1-\kappa\gamma)} \biggr) r_\textnormal{max},
    \end{align*}
    where the reward error is given by
    \begin{align*}
\epsilon_{r, \textnormal{avi}}(\gamma,\alpha,d_\mathcal{Y},\mathbf{L},T, \gamma', \varepsilon_\textnormal{low}) = \epsilon_r(\gamma,\alpha,d_\mathcal{Y},\mathbf{L},T) + \biggl(\frac{1+\kappa}{1-\kappa\gamma} - \frac{(\kappa\gamma)^T(1+\gamma)}{\gamma-\kappa\gamma^2}\biggr) \varepsilon_{\textnormal{low}}.
\end{align*}
\end{theorem}
\begin{proof}
See Appendix \ref{sec:proof_fsavi_regret}.
\end{proof}

\section{Proofs from Sections \ref{sec:GeneralModel} and \ref{sec:hierarchical_approx}}
\subsection{Technical Lemmas}
\begin{lemma}
    \label{lemma:xt_yt}
Consider a $(\alpha, d_\mathcal{Y})$-fast-slow MDP. For any states $(x_0,y_0)$ and $(\tilde{x}_0,\tilde{y}_0)$, let $(x_t, y_t)$ and $(\tilde{x}_t,\tilde{y}_t)$ be the states reached after $t$ transitions under a policy $\boldsymbol{\pi} = (\pi_0,\ldots,\pi_{t-1})$, i.e., $(x_t, y_t) = f^{\boldsymbol{\pi}}(x_{t-1}, y_{t-1}, w_t)$ and $(\tilde{x}_t, \tilde{y}_t) = f^{\boldsymbol{\pi}}(\tilde{x}_{t-1}, \tilde{y}_{t-1}, \tilde{w}_t)$. Then, for any policy $\boldsymbol{\pi}$, we have
\begin{enumerate}[(i)]
    \item $\|x_t - \tilde{x}_0\|_2 \leq t \alpha  d_\mathcal{Y} + \|x_0-\tilde{x}_0\|_2$,
    \item $ \|x_t - \tilde{x}_t\|_2 \leq 2 t \alpha d_\mathcal{Y} + \|x_0-\tilde{x}_0\|_2$,
    \item  $\|y_t - \tilde{y}_t\|_2 \leq 2 t d_\mathcal{Y} + \|y_0-\tilde{y}_0\|_2$.
\end{enumerate}
\end{lemma}
\begin{proof}
    Lemma~\ref{lemma:xt_yt} is a consequence of Assumption~\ref{assumption.fast_slow_transition}.
\end{proof}

The following two lemmas are about properties of the Bellman operators $H$ and $\tilde{H}$ (recall that $\tilde{H}$ is the frozen-state version).
\begin{lemma}
    \label{lemma:HV_HV'}
    For any state $(x,y)$ and any two value functions $V, \; V': \mathcal{X}\times \mathcal{Y} \rightarrow \R$, we have
    \[
    |(\tilde{H}^t V) (x,y) - (\tilde{H}^t V') (x,y)| \leq \gamma^t \max_{y \in \mathcal Y_x^t} \, \bigl| V(x,y) - V'(x,y) \bigr|,
    \]
    where $\mathcal Y_s^t$ is the set of fast states reachable from $s=(x,y)$ after $t$ transitions of $f_\mathcal{Y}(x, \cdot, \cdot, \cdot)$.
\end{lemma}
\begin{proof}
The result follows by the contraction property of the Bellman operators.
\end{proof}

\begin{lemma}[Discrepancy between $H$ and $\tilde{H}$]
    \label{lemma:HV_H'V}
    Consider a value function $V:\mathcal X \times \mathcal Y \rightarrow \mathbb R$. Suppose there exists $L_V>0$ such that for any states $(x,y)$ and $(\tilde{x},\tilde{y})$, it holds that $
    |V(x,y) - V(\tilde{x},\tilde{y})| \leq L_V \|(x,y) - (\tilde{x},\tilde{y})\|_2$.
    Then,
    \begin{align*}
       \bigl | (H^t &V)(x,y) - (\tilde{H}^t V) (\tilde{x},\tilde{y}) \bigr | \\
       &\leq \bigl\|(x,y)-(\tilde{x}, \tilde{y}) \bigl\|_2 \biggl(L_r \sum_{i=0}^{t-1} (\gamma L_f)^i + L_V (\gamma L_f)^t\biggr) + \alpha d_\mathcal{Y} \biggl( L_r \sum_{i=1}^{t-1} \gamma^i \, \sum_{j=0}^{i-1} L_f^j  + L_V \gamma^t \sum_{j=0}^{t-1} L_f^j \biggr).
    \end{align*}
\end{lemma}

\begin{proof}
We need to show that for each $t \ge 1$,
\begin{equation}
    \bigl | (H^t V)(x,y) - (\tilde{H}^t V) (\tilde{x},\tilde{y}) \bigr | \le \phi_{t, 1} \, \bigl\|(x,y)-(\tilde{x}, \tilde{y}) \bigr \|_2 + \phi_{t, 2} \, (\alpha d_\mathcal{Y}),
    \label{eq:induction_hyp_discrep}
\end{equation}
for coefficients $\phi_{t,1}$ and $\phi_{t,2}$ defined as
\[
\phi_{t,1} = \biggl(L_r \sum_{i=0}^{t-1} (\gamma L_f)^i + L_V (\gamma L_f)^t\biggr) \quad \text{and} \quad \phi_{t,2} = \biggl( L_r \sum_{i=1}^{t-1} \gamma^i \, \sum_{j=0}^{i-1} L_f^j  + L_V \gamma^t \sum_{j=0}^{t-1} L_f^j \biggr).
\]
Let $(x', y') = (f_\mathcal{X}(x,y,a,w), f_\mathcal{Y}(x,y,a,w))$ and $(\tilde{x}', \tilde{y}') = (f_\mathcal{X}(\tilde{x},\tilde{y},a,w), f_\mathcal{Y}(\tilde{x},\tilde{y},a,w))$ be one-step transitions starting from $(x,y)$ and $(\tilde{x},\tilde{y})$, according to the true system dynamics. For $t=1$:
\begin{align*}
    \bigl | (H &V)(x,y) - (\tilde{H} V) (\tilde{x},\tilde{y}) \bigr |  \\
    &= \Bigl| \max_{a} \Bigl \{ r(x,y, a) + \gamma \, \E [V (x',y')] \Bigr \} - \max_{\tilde{a}} \Bigl \{ r(\tilde{x},\tilde{y},\tilde{a}) + \gamma \, \E [V (\tilde{x}, \tilde{y}')]\Bigr \} \Bigr| \nonumber\\
    &\le  L_r \bigl \|(x, y) - (\tilde{x}, \tilde{y}) \bigr\|_2 + \gamma \max_{a}  \mathbb{E} \bigl | V(x', y') - V(\tilde{x}, \tilde{y}')  \bigr|\\
    &\le L_r \bigl \|(x, y) - (\tilde{x}, \tilde{y}) \bigr\|_2 + L_V\gamma \max_{a}  \mathbb{E} \, \| (x', y') - (\tilde{x}, \tilde{y}')  \|_2 \\
    &= L_r \bigl \|(x, y) - (\tilde{x}, \tilde{y}) \bigr\|_2 + L_V\gamma \max_{a}  \mathbb{E} \bigl [  \| (x', y') - (\tilde{x}', \tilde{y}')\|_2 + \|(\tilde{x}', \tilde{y}')- (\tilde{x}, \tilde{y}') \|_2  \bigr]\\
    &\le L_r \bigl \|(x, y) - (\tilde{x}, \tilde{y}) \bigr\|_2 + L_V\gamma \bigl [  L_f \| (x, y) - (\tilde{x}, \tilde{y})\|_2 + \alpha d_\mathcal{Y}  \bigr]\\
    &=\phi_{1, 1} \, \bigl\|(x,y)-(\tilde{x}, \tilde{y}) \bigr \|_2 + \phi_{1, 2} \, (\alpha d_\mathcal{Y}),
\end{align*}
which verifies the base case. Let us now assume that (\ref{eq:induction_hyp_discrep}) holds for $t-1$.
\begin{align*}
    \bigl | (H^t &V)(x,y) - (\tilde{H}^t V) (\tilde{x},\tilde{y}) \bigr |  \\
    &= \Bigl| \max_{a} \Bigl \{ r(x,y, a) + \gamma \, \E [ (HV)^{t-1} (x',y')] \Bigr \} - \max_{\tilde{a}} \Bigl \{ r(\tilde{x},\tilde{y},\tilde{a}) + \gamma \, \E [ (\tilde{H}V)^{t-1} (\tilde{x}, \tilde{y}')]\Bigr \} \Bigr| \nonumber\\
    &\le  L_r \bigl \|(x, y) - (\tilde{x}, \tilde{y}) \bigr\|_2 + \gamma \max_{a}  \mathbb{E} \bigl | (HV)^{t-1}(x', y') - (\tilde{H}V)^{t-1}(\tilde{x}, \tilde{y}')  \bigr|\\
    &\le  L_r \bigl \|(x, y) - (\tilde{x}, \tilde{y}) \bigr\|_2 + \gamma 
    \Bigl[ \phi_{t-1, 1} \, \bigl\|(x',y')-(\tilde{x}, \tilde{y}') \bigr \|_2 + \phi_{t-1, 2} \, (\alpha d_\mathcal{Y}) \Bigr]\\
    &\le  L_r \bigl \|(x, y) - (\tilde{x}, \tilde{y}) \bigr\|_2 + \gamma 
    \Bigl[ \phi_{t-1, 1} \, \bigl(L_f \bigl \| (x, y) - (\tilde{x}, \tilde{y}) \bigr\|_2 + \alpha d_\mathcal{Y}  \bigr) + \phi_{t-1, 2} \, (\alpha d_\mathcal{Y}) \Bigr]\\
    &\le \bigl(L_r + \gamma L_f \phi_{t-1, 1} \bigr) \bigl \|(x, y) - (\tilde{x}, \tilde{y}) \bigr\|_2 +  \bigl(\gamma\phi_{t-1, 1} + \gamma \phi_{t-1, 2}\bigr)(\alpha d_\mathcal{Y}),
\end{align*}
where the second inequality follows by the induction hypothesis and the third inequality follows by the same steps as in the case of $t=1$. It is straightforward to verify that
\[
\phi_{t, 1} = L_r + \gamma L_f \phi_{t-1, 1} \quad \text{and} \quad \phi_{t, 2} = \gamma\phi_{t-1, 1} + \gamma \phi_{t-1, 2},
\]
which completes the induction step and the proof.
\end{proof}

\subsection{Proof of Proposition~\ref{thm:stationary_opt_policy}}
\label{sec:stationary_opt_policy_proof}
    We consider an MDP $\langle \mathcal{S}, \mathcal{A}, \mathcal{W}, f, r, \gamma \rangle$ and note that $U^*$ is the unique optimal solution of the base model \eqref{eq.Bellman_original}, and there exists a stationary optimal policy $\nu^*(x,y)=\argmax \, U^*(x,y)$ that attains this optimal value \citep[Proposition 4.3]{bertsekas2004stochastic}. Fix a state $s_0 \in \mathcal S$ and for $t > 0$ and a sequence of policies $\pi_0, \ldots, \pi_{t-1}$, define the notation:
    \[
    s_1(\pi_0) = f^{\pi_0}(s_0, w_1) \quad \text{and} \quad s_{t'+1}(\pi_0, \ldots, \pi_{t'}) = f^{\pi_{t'}}(s_{t'}(\pi_0, \ldots, \pi_{t'-1}), w_{t'+1})
    \]
    for $t' \ge 1$. Therefore, we have
    \begin{align} 
        U^*(s_0) &= \max_{\pi_0} \, r(s_0,\pi_0) + \gamma \, \E \bigl[ U^*(s_1(\pi_0)) \bigr] =  r(s_0,\nu^*) + \gamma \, \E\bigl[ U^*(s_1(\nu^*)) \bigr] \label{eq.Ustar1}.
    \end{align}
    By expanding the $U^*(s_1(\pi_0))$ and $U^*(s_1(\nu^*))$ terms in \eqref{eq.Ustar1}, we have the following:
    \begin{equation*}
        \begin{aligned}
        U^*(s_0) &= \max_{\pi_0,\pi_1} \, \E \Bigl[ r(s_0,\pi_0) + \gamma \, r(s_1(\pi_0),\pi_1) + \gamma^2 \, U^*(s_2(\pi_0, \pi_1)) \Bigr]\\
        &= \E \Bigl[r(s_0,\nu) + \gamma \, r(s_1(\nu^*),\nu^*) + \gamma^2 \, U^*(s_2(\nu^*, \nu^*))\Bigr].
        \end{aligned}
    \end{equation*}
    Let $\boldsymbol{\pi} = (\pi_0, \pi_1, \ldots, \pi_{T-1})$. Repeating the expansion, we obtain:
    \begin{align}
        U^*(s_0) &= \max_{\boldsymbol{\pi}} \, \E \left[ \sum_{t=0}^{T-1} \gamma^t \, r\bigl(s_t(\pi_0, \ldots, \pi_{t-1}), \pi_t\bigr) + \gamma^T U^*\bigl(s_T(\pi_0, \ldots, \pi_{T-1})\bigr) \right] \label{eq.UstarT1}\\
        &= \E \left[ \sum_{t=0}^{T-1} \gamma^t \, r\bigl(s_t(\nu^*, \ldots, \nu^*),\nu^*\bigr) + \gamma^T \, U^*\bigl(s_T(\nu^*, \ldots, \nu^*)\bigr)\right]. \label{eq.UstarT2}
    \end{align}
    Observe that \eqref{eq.UstarT1} is in same form as the Bellman equation \eqref{eq.Bellman_T_original} for the hierarchical reformulation (with $T$-horizon reward function $R$ and value function $\bar{U}$), which has a unique optimal solution $\bar{U}^*$. Therefore $U^*(s_0) = \bar{U}^*(s_0)$ and (i) is proved when we recall that $s_0$ was chosen arbitrarily. Part (ii) follows because by \eqref{eq.UstarT2}, it is clear that $(\nu^*, \ldots, \nu^*)$ solves \eqref{eq.UstarT1} and hence also \eqref{eq.Bellman_T_original}.

\section{Proofs for Section~\ref{sec:theory} and Section~\ref{sec:interpreting}}
\subsection{Technical Lemmas}

\begin{lemma}
    \label{lemma:U_V}
    Consider two MDPs, $\mathcal M_1$ and $\mathcal M_2$, who differ in their transition and reward functions: $\mathcal M_i = \langle \mathcal{S}, \mathcal{A}, \mathcal{W}, f_i, r_i, \gamma \rangle$. Let $U_i^*$ be the optimal value function of $\mathcal M_i$. Suppose that
    \begin{enumerate}[(a)]
        \item $|r_1(s,a) - r_2(s,a)| \leq \epsilon_r$ for all $s\in\mathcal{S}$ and $a\in\mathcal{A}$;
        \item $\|f_1(s,a,w) - f_2(s,a,w)\|_2 \leq d$ for all $s \in \mathcal S$, $a \in \mathcal A$, and $w \in \mathcal W$; and
        \item there exists $L_1 > 0$ such that $|U_1^*(s) - U_1^*(s')| \leq L_1 \|s-s'\|_2$ for all $s, s'\in\mathcal{S}$.
    \end{enumerate}
    Then, the difference in optimal values of the two MDPs can be bounded as follows:
    \[
        \bigl|U_1^*(s)-U_2^*(s)\bigr| \leq \frac{\epsilon_r + \gamma L_1 d}{1-\gamma}
    \]
    for all $s \in \mathcal S$.
\end{lemma}
\begin{proof}
    Let $\hat{s} = \argmax_{s\in\mathcal{S}} | U_1^*(s)-U_2^*(s)|$. We will analyze $\bigl| U_1^*(\hat{s})-U_2^*(\hat{s}) \bigr|$. 
    \begin{align*}
        \bigl| U_1^*(\hat{s}) - U_2^*(\hat{s}) \bigr| &= \Bigl| \max_{a_1\in\mathcal{A}} \, \Bigl \{ r_1(\hat{s},a_1) + \gamma\, \E \bigl [U_1^* \bigl( f_1(\hat{s},a_1,w) \bigr) \bigr] \Bigr\} - \max_{a_2\in\mathcal{A}} \, \Bigl\{ r_2(\hat{s},a_2) + \gamma \, \E \bigl [U_2^*(f_2(\hat{s},a_2,w)) \bigr] \Bigr \} \Bigr|\\
        &\leq \max_{a\in\mathcal{A}} \,  \bigl| r_1(\hat{s},a) + \gamma\, \E \bigl [U_1^* \bigl( f_1(\hat{s},a,w) \bigr) \bigr] - r_2(\hat{s},a) - \gamma \, \E \bigl [U_2^*(f_2(\hat{s},a,w)) \bigr] \bigr| \\
        &\leq \max_{a\in\mathcal{A}} \, \bigl| r_1(\hat{s},a) - r_2(\hat{s},a)\bigr| + \gamma \max_{a\in\mathcal{A}}  \, \bigl|\E \bigl [U_1^* \bigl( f_1(\hat{s},a,w) \bigr) \bigr] - \E \bigl [U_2^* \bigl( f_2(\hat{s},a,w) \bigr) \bigr]\bigr| \\
        &\leq \begin{aligned}[t]\epsilon_r &+ \gamma \max_{a\in\mathcal{A}}  \, \bigl|\E \bigl [U_1^* \bigl( f_1(\hat{s},a,w) \bigr) - U_1^* \bigl( f_2(\hat{s},a,w) \bigr) \bigr]\bigr|\\
                                &+ \gamma \max_{a\in\mathcal{A}}  \, \bigl|\E \bigl [U_1^* \bigl( f_2(\hat{s},a,w) \bigr) - U_2^* \bigl( f_2(\hat{s},a,w) \bigr) \bigr]\bigr|\end{aligned} \\
        &\leq \epsilon_r + \gamma L_1 \max_{a, w}  \|f_1(\hat{s},a,w) - f_2(\hat{s},a,w)\|_2 + \gamma \max_{s\in\mathcal{S}} \, \bigl|U_1^*(s)-U_2^*(s)\bigr| \\
        &\leq \epsilon_r + \gamma L_1 d + \gamma \bigl|U_1^*(\hat{s})-U_2^*(\hat{s})\bigr|.
    \end{align*}
    Rearranging, we have
    \begin{equation*}
        \bigl| U_1^*(\hat{s})-U_2^*(\hat{s}) \bigr| \leq \frac{\epsilon_r + \gamma L_1 d}{1-\gamma},
    \end{equation*}
    which completes the proof if we recall how $\hat{s}$ was chosen.
\end{proof}

\begin{lemma}
    \label{lemma:loss_bound}
    Consider two MDPs, $\mathcal M_1$ and $\mathcal M_2$, who differ in their transition and reward functions: $\mathcal M_i = \langle \mathcal{S}, \mathcal{A}, \mathcal{W}, f_i, r_i, \gamma \rangle$. Let $U_i^*$ be the optimal value function of $\mathcal M_i$. Suppose that
    \begin{enumerate}[(a)]
        \item $|r_1(s,a) - r_2(s,a)| \leq \epsilon_r$ for all $s\in\mathcal{S}$ and $a\in\mathcal{A}$;
        \item $\|f_1(s,a,w) - f_2(s,a,w)\|_2 \leq d$ for all $s \in \mathcal S$, $a \in \mathcal A$ and $w \in \mathcal W$;
        \item there exists $L_1 > 0$ such that $|U_1^*(s) - U_1^*(s')| \leq L_1 \|s-s'\|_2$ for any $s, s'\in\mathcal{S}$; and
        \item $|U_1^*(s)-U_2^*(s)|\leq \epsilon_U$ for all $s\in\mathcal S$.
    \end{enumerate}
    Let $\tilde{\pi}_2$ be a policy that is an approximation of the optimal policy for $\mathcal M_2$, in the sense that:
    \begin{equation}
    \label{eq.pi_2_greedy}
    \tilde{\pi}_2(s) = \argmax_{a \in \mathcal A} \, \Bigl\{ \tilde{r}_2(s, a) + \gamma \, \E \bigl [\tilde{U}_2\bigl(\tilde{f}_2(s,a,w)\bigr) \bigr] \Bigr\},
    \end{equation}
    where $|r_2(s,a) - \tilde{r}_2(s,a)| \leq \tilde{\epsilon}_r$, $\|f_2(s,a,w) - \tilde{f}_2(s,a,w)\|_2 \leq \tilde{d}$, and $|U_2^*(s)-\tilde{U}_2(s)|\leq \tilde{\epsilon}_U$ for all $s \in \mathcal S$, $a \in \mathcal A$, and $w \in \mathcal W$.
    Then, the value of $\tilde{\pi}_2$ when implemented in $\mathcal M_1$ has regret bounded by:
    \begin{equation*}
         \bigl \|U_1^* - U_1^{\tilde{\pi}_2}\bigr\|_\infty \leq \frac{2(\epsilon_r + \tilde{\epsilon}_r) + 2\gamma (\epsilon_U + \tilde{\epsilon}_U) + 2\gamma L_1 (d + \tilde{d})}{1-\gamma}.
    \end{equation*}
    This lemma is a generalization and extension of Corollary 1 of \cite{singh1994upper}.
\end{lemma}
\begin{proof}
    Let $\pi_1^*$ be an optimal policy for $\mathcal M_1$. Using (\ref{eq.pi_2_greedy}), it follows that
    \begin{equation}
        \tilde{r}_2(s,\pi_1^*(s)) + \gamma \, \E \bigl [\tilde{U}_2(\tilde{f}_2(s,\pi_1^*(s),w)) \bigr]
        \leq \tilde{r}_2(s,\tilde{\pi}_2(s)) + \gamma \, \E \bigl [\tilde{U}_2(\tilde{f}_2(s,\tilde{\pi}_2(s),w) ) \bigr].
    \label{eq:lemma11:1}
    \end{equation}
    Set $\boldsymbol{\epsilon}_U = \epsilon_U + \tilde{\epsilon}_U$, $\boldsymbol{\epsilon}_r = \epsilon_r + \tilde{\epsilon}_r$. Combining parts (a) and (d) in the statement of the lemma with the approximation errors of $\tilde{U}_2$ and $\tilde{r}_2$, we know that $U_1^*(s) - \boldsymbol{\epsilon}_U \leq \tilde{U}_2(s) \leq U_1^*(s) +\boldsymbol{\epsilon}_U$ and $r_1(s,a) - \boldsymbol{\epsilon}_r \leq \tilde{r}_2(s,a) \leq r_1(s,a) + \boldsymbol{\epsilon}_r$ for any $s$ and $a$. Using these, we can lower bound both terms on the left-hand-side of (\ref{eq:lemma11:1}), upper bound both terms on the right-hand-side of (\ref{eq:lemma11:1}), and then rearrange to obtain
    \begin{align} 
        r_1(s,\pi_1^*(s)) - r_1(s,\tilde{\pi}_2(s))
        &\leq 2\boldsymbol{\epsilon}_r + 2\gamma\boldsymbol{\epsilon}_U + \gamma\, \E \bigl [ U_1^*(\tilde{f}_2(s,\tilde{\pi}_2(s),w)) - U_1^*(\tilde{f}_2(s,\pi_1^*(s),w)) \bigr ].\label{eq.rdot_diff}
    \end{align}
    Let state $\hat{s} = \argmax_{s\in\mathcal{S}} \, U_1^*(\hat{s}) - U_1^{\tilde{\pi}_2}(\hat{s}) $ be the state that achieves the largest regret (when using $\tilde{\pi}_2$ in $\mathcal M_1$). Substituting from \eqref{eq.rdot_diff} gives
    \begin{equation*}
        \begin{aligned}
        U_1^*(\hat{s}) - U_1^{\tilde{\pi}_2}(\hat{s})
        &= r_1(\hat{s},\pi_1^*(\hat{s})) - r_1(\hat{s},\tilde{\pi}_2(\hat{s})) + \gamma \, \E \bigl[ U_1^*(f_1(\hat{s},\pi_1^*(\hat{s}),w)) - U_1^{\tilde{\pi}_2}(f_1(\hat{s},\tilde{\pi}_2(\hat{s}),w)) \bigr]\\
        &\leq \begin{aligned}[t] 2\boldsymbol{\epsilon}_r &+ 2\gamma \boldsymbol{\epsilon}_U + \gamma\, \E \bigl [ U_1^*(\tilde{f}_2(\hat{s},\tilde{\pi}_2(\hat{s}),w)) - U_1^*(\tilde{f}_2(\hat{s},\pi_1^*(\hat{s}),w)) \bigr ]\\
        &+ \gamma \, \E \bigl[ U_1^*(f_1(\hat{s},\pi_1^*(\hat{s}),w)) - U_1^{\tilde{\pi}_2}(f_1(\hat{s},\tilde{\pi}_2(\hat{s}),w)) \bigr]
        \end{aligned}\\
        &= \begin{aligned}[t] 2 \boldsymbol{\epsilon}_r &+ 2\gamma \boldsymbol{\epsilon}_U + \gamma\, \E \bigl [ U_1^*(\tilde{f}_2(\hat{s},\tilde{\pi}_2(\hat{s}),w)) - U_1^*(f_1(\hat{s},\tilde{\pi}_2(\hat{s}),w)) \bigr ]\\
        &+ \gamma \, \E \bigl[ U_1^*(f_1(\hat{s},\pi_1^*(\hat{s}),w)) -U_1^*(\tilde{f}_2(\hat{s},\pi_1^*(\hat{s}),w)) \bigr]\\
        &+ \gamma \, \E \bigl[ U_1^*(f_1(\hat{s},\tilde{\pi}_2(\hat{s}),w)) - U_1^{\tilde{\pi}_2}(f_1(\hat{s},\tilde{\pi}_2(\hat{s}),w)) \bigr]
        \end{aligned}\\
        &\leq 2\boldsymbol{\epsilon}_r + 2\gamma\boldsymbol{\epsilon}_U + 2\gamma L_1 (d+\tilde{d}) + \gamma \bigl(U_1^*(\hat{s}) - U_1^{\tilde{\pi}_2}(\hat{s}) \bigr),
        \end{aligned}
    \end{equation*}
    where we have used property (c) and that $\|f_1(s, a, w) - \tilde{f}_2(s,a,w)\|_2 \le d + \tilde{d}$.
    Therefore, we rearrange to see that
    \begin{equation*}
        U_1^*(\hat{s}) - U_1^{\tilde{\pi}_2}(\hat{s}) \leq \frac{2\boldsymbol{\epsilon}_r + 2\gamma\boldsymbol{\epsilon}_U + 2\gamma L_1 (d+\tilde{d})}{1-\gamma},
    \end{equation*}
    completing the proof.
\end{proof}

\begin{lemma}
    \label{lemma:Uk_Uopt}
    Consider an MDP $\langle \mathcal{S}, \mathcal{A}, \mathcal{W}, f, r, \gamma \rangle$ with reward function $r$ taking values in $[0, r_\textnormal{max}]$. Suppose the optimal value function is $U^*$ and the associated Bellman operator is $F$. Fix any initial value function such that $0\le U_0(s) \le r_\textnormal{max}/(1-\gamma)$ for all $s$ and let $U^k = F^k \, U_0$ be the result after iteration $k$ of value iteration. Then, it holds that
    \begin{equation*}
        \|U^k - U^*\|_\infty \leq \frac{ \gamma^k \, r_\textnormal{max}}{1-\gamma}.
    \end{equation*}
\end{lemma}

\begin{proof}
    This is a standard result that follows from the contraction property of $F$ and the fact that $U^k = F \, U^{k-1}$. Therefore,
    \begin{align*}
        \|U^k - U^*\|_\infty 
        = \|HU^{k-1} - HU^*\|_\infty &\leq \gamma\|U^{k-1} - U^*\|_\infty \leq \gamma^k \|U^0 - U^*\|_\infty \leq \frac{\gamma^k \, r_\textnormal{max}}{1-\gamma},
    \end{align*}
    where in the last step, we used $0 \le U^*(s) \le r_\textnormal{max}/(1-\gamma)$ for all $s$.
\end{proof}

\begin{lemma}[Proposition 6.1 of \cite{bertsekas1996neuro}]
    \label{lemma:Ugreedy_Uopt}
    Consider an MDP $\langle \mathcal{S}, \mathcal{A}, \mathcal{W}, f, r, \gamma \rangle$ with optimal value function $U^*$. Suppose that $\nu$ is a policy that is greedy with respect to another value function $U$:
    \[
    \nu(s) = \argmax_{a} \, \bigl \{ r(s,a) + \E\bigl[ U(f(s,a,w)) \bigr] \bigr\}.
    \]
    If $\|U - U^*\|_\infty \leq \varepsilon$, then the performance of $\nu$ is bounded as follows:
    \[
    \|U^{\nu} - U^*\|_\infty \leq \frac{2 \gamma \varepsilon}{1-\gamma}.
    \]
\end{lemma}

\begin{lemma}
    \label{prop:K2}
    Let $V^k(J_1^*,\boldsymbol{\pi}^*)$ be the value function approximation obtained from running FSVI for $k$ iterations. Then, the ``value iteration error'' is given by
    \begin{align*}
    \bigl \|V^{k}(J_1^*,\boldsymbol{\pi}^*) - V^*(J_1^*,\boldsymbol{\pi}^*) \bigr\|_\infty 
    &\leq \frac{\gamma^{kT} r_\textnormal{max}}{1-\gamma}.
    \end{align*}
\end{lemma}
\begin{proof}
Consider the upper-level MDP. Note that the discount factor is $\gamma^T$ and the $T$-horizon reward function
\[
\tilde{R}(s_0, a, J_1^*) \in \biggl[0, \frac{1-\gamma^T}{1-\gamma}r_\textnormal{max}\biggr].
\]
The result follows by Lemma~\ref{lemma:Uk_Uopt}.
\end{proof}

\subsection{Proof of Proposition~\ref{prop:cumreward_diff}} 
\label{sec:proof_cumreward_diff}
Using \eqref{eq.R1_a} and \eqref{eq.R2_a}, the difference between the two reward functions can be expanded as follows:
\begin{align*}
    \bigl|\E &[R(s_0, a, \boldsymbol{\pi}^*)] - \E [\tilde{R}(s_0, a, J_1^*)]\bigr| \\
    &= \bigl|r(x_0,y_0,a) + \gamma \, \E [(H^{T-1} U^*) (x_1,y_1)] - \gamma^T \, \E [U^*(x_T,y_T)] - r(x_0,y_0,a) - \gamma \, \E [(\tilde{H}^{T-1} \mathbf{0}) (x_1,y_1)]\bigr|\\
    &= \gamma \, \bigl|\E [(H^{T-1} U^*) (x_1,y_1)] - \E [(\tilde{H}^{T-1} 0) (x_1,y_1)] - \gamma^{T-1} \E [U^*(x_T,y_T)]\bigr|\\
    &\leq \begin{aligned}[t]&\underbrace{\gamma \, \E \bigl| (H^{T-1} U^*) (x_1,y_1) - (\tilde{H}^{T-1} U^*)(x_1,y_1)\bigr|}_{\text{Term A}} \\
    &+ \underbrace{\gamma \, \E  \bigl|(\tilde{H}^{T-1} U^*)(x_1,y_1) - (\tilde{H}^{T-1} 0) (x_1,y_1) - \gamma^{T-1} \, U^*(x_T,y_T)\bigr|}_{\text{Term B}},\end{aligned}
\end{align*}
where $(x_t, y_t)$ is the state obtained after transitioning from $(x_0, y_0)$ according to the true dynamics $f = (f_\mathcal{X}, f_\mathcal{Y})$ for $t$ steps. We now work on Terms A and B separately. 

Noting that $U^*$ has Lipschitz constant $L_U$ by Assumption \ref{assumption:lipschitz}, we can apply Lemma \ref{lemma:HV_H'V} to Term A to obtain
\begin{equation}
\text{Term A} \; \le \; \alpha d_\mathcal{Y} \biggl( L_r \sum_{i=1}^{T-2} \gamma^i \, \sum_{j=0}^{i-1} L_f^j  +\gamma^{T-1} L_U  \sum_{j=0}^{T-2} L_f^j \biggr).
\label{eq:termA}
\end{equation}

Moving on to Term B, since the reward function $r \ge 0$, it follows that $U^*(s) \ge 0$ for all $s$. Also, the monotonicity of $\tilde{H}$ implies that $(\tilde{H}^{T-1} U^*) \ge (\tilde{H}^{T-1} 0)$. Therefore, applying Lemma \ref{lemma:HV_HV'},
\begin{align*}
(\tilde{H}^{T-1} U^*)(x_1,y_1) - (\tilde{H}^{T-1} 0) (x_1,y_1) &= \bigl|(\tilde{H}^{T-1} U^*)(x_1,y_1) - (\tilde{H}^{T-1} 0) (x_1,y_1)\bigr|\\
&\le \gamma^{T-1} \textstyle \max_{y \in \mathcal Y_{s_1}^{T-1}} \bigl| U^*(x_1,y) - 0 \bigr| \\
&= \gamma^{T-1} \textstyle \max_{y \in \mathcal Y_{s_1}^{T-1}} U^*(x_1,y),
\end{align*}
where $\mathcal Y_{s_1}^{T-1}$ is the set of fast states reachable from $s_1=(x_1,y_1)$ after $T-1$ transitions of $f_\mathcal{Y}(x_1, \cdot, \cdot, \cdot)$. Let $\tilde{y}_{s_1} = \argmax_{y \in \mathcal Y_{s_1}^{T-1}} U^*(x_1,y)$ be the fast state that attains the maximum. Note that $\tilde{y}_{s_1}$ depends on $s_1$, which is random. Combining with the rest of Term B, we have
\begin{align}
\text{Term B} \; &\le \; \gamma \, \E  \bigl|\gamma^{T-1} U^*(x_1,\tilde{y}_{s_1}) - \gamma^{T-1} \, U^*(x_T,y_T)\bigr|\nonumber\\
&\le \; \max_{\omega \in \Omega} \; \gamma^T\, \bigl| U^*\bigl(x_1(\omega), \tilde{y}_{s_1}(\omega)\bigr) - U^*\bigl(x_T(\omega),y_T(\omega)\bigr) \bigr|\nonumber\\
&\le \max_{\omega \in \Omega} \; \gamma^T L_U \Bigl(\|x_1(\omega) - x_T(\omega)\|_2 + \|\tilde{y}_{s_1}(\omega) - y_T(\omega)\|_2\Bigr)\label{eq:termB_1}\\
&\le \gamma^T L_U d_\mathcal{Y} (\alpha + 2) (T-1),\label{eq:termB_2}
\end{align}
where \eqref{eq:termB_1} follows by Assumption \ref{assumption:lipschitz} and \eqref{eq:termB_2} comes from Lemma~\ref{lemma:xt_yt} (we use that $x_T(\omega)$ is $T-1$ transitions from $x_1(\omega)$ and both $\tilde{y}_{s_1}(\omega)$ and $y_T(\omega)$ are both $T-1$ transitions from $y_1(\omega)$). Finally, we have
\begin{align*}
\text{Terms A} + \text{B} \; &\le \; \alpha d_\mathcal{Y} \biggl( L_r \sum_{i=1}^{T-2} \gamma^i \, \sum_{j=0}^{i-1} L_f^j  + \gamma^{T-1} L_U \sum_{j=0}^{T-2} L_f^j \biggr) +  \gamma^T L_U d_\mathcal{Y} (\alpha + 2) (T-1)\\
&= \alpha d_\mathcal{Y} \biggl( L_r \sum_{i=1}^{T-2} \gamma^i \, \sum_{j=0}^{i-1} L_f^j \biggr) + \gamma^{T-1} L_U \Biggl[ \alpha d_\mathcal{Y} \sum_{j=0}^{T-2} L_f^j + \gamma d_\mathcal{Y} (\alpha + 2) (T-1)\Biggr]
\end{align*}
which completes the proof.

\subsection{Proof of Lemma \ref{lemma:main}}
\label{sec:proof_lemma_main}
To analyze $\mathcal R(\mu, \boldsymbol{\pi}) = \|\bar{U}^* - \bar{U}^{\mu}(\boldsymbol{\pi})\|_\infty$, we will consider two MDPs that operate on the $T$-period timescale, one with optimal value $\bar{U}^*$ and the other with optimal value $V^*(J_1, \boldsymbol{\pi})$. The reason to study an MDP with optimal value $V^*(J_1, \boldsymbol{\pi})$ is because $\mu$ can be viewed as an \emph{approximation} to the optimal policy for the second MDP, as suggested in (\ref{eq.mu_approx_opt}). Since both MDPs are defined on the $T$-period timescale, the transition functions are defined using $T$-period noise sequences $\boldsymbol{w} = (w_1,w_2,\ldots,w_{T})$. 
\begin{itemize}
    \item For $\bar{U}^*$, let $\mathcal M_1 = \langle \mathcal{S}, \mathcal{A}, \mathcal{W}, f_1, r_1, \gamma^T \rangle$ be the MDP associated with the base model reformulation \eqref{eq.Bellman_T_original}, but with the lower-level policy fixed to be $\boldsymbol{\pi}^*$. The reward function $r_1$ is $r_1(s,a) = \E [R(s, a, \boldsymbol{\pi}^*)]$. Given a $T$-period noise sequence $\boldsymbol{w}$, an initial state $s$, and action $a$, the ``next'' state $f_1(s,a,\boldsymbol{w})=s_T(a, \boldsymbol{\pi}^*)$ is the state obtained by first taking action $a$ in state $s$ and then following policy $\boldsymbol{\pi}^*$ for the next $T-1$ steps.
    \item For $V^*(J_1, \boldsymbol{\pi})$, let $\mathcal M_2 = \langle \mathcal{S}, \mathcal{A}, \mathcal{W}, f_2, r_2, \gamma^T \rangle$ be the MDP associated with the frozen-state hierarchical approximation \eqref{eq.Vupper}, where $r_2$ is defined as $r_2(s,a) = \E [\tilde{R}(s, a, J_1)]$. The transition function $f_2$ is defined in the same way as $f_1$ except we replace $\boldsymbol{\pi}^*$ by $\boldsymbol{\pi}$.
\end{itemize}
Let $\epsilon_r(\boldsymbol{\pi}^*, J_1) = \max_{s,a} \, |\E [R(s, a, \boldsymbol{\pi}^*)] - \E [\tilde{R}(s, a, J_1)]|$, so that we have $|r_1(s,a) - r_2(s,a)| \leq \epsilon_r(\boldsymbol{\pi}^*, J_1)$.
Noting that the first steps of $f_1$ and $f_2$ are the same (action $a$ in state $s$ with $w_1$ revealed), applying parts (ii) and (iii) of Lemma~\ref{lemma:xt_yt}, the maximum discrepancy between $f_1$ and $f_2$ is:
    \begin{align*}
        \|f_1(s,a,\boldsymbol{w}) - f_2(s,a,\boldsymbol{w})\|_2 \; \leq \; d(\alpha,d_\mathcal{Y},T) := 2 (\alpha+1) d_\mathcal{Y} (T-1).
    \end{align*}
    Applying Lemma~\ref{lemma:U_V}, we see that
\begin{align*}
    \bigl\|\bar{U}^* - V^*(J_1, \boldsymbol{\pi})\bigr\|_\infty \leq \frac{1}{1-\gamma^T} \Bigl( \epsilon_r(\boldsymbol{\pi}^*, J_1) +   \gamma^T L_{U} d(\alpha,d_\mathcal{Y},T)  \Bigr).
\end{align*}
We also need to account for the fact that $\mu$ is greedy with respect to $V$, an approximation of the optimal value of $\mathcal M_2$. More precisely, $\mu$ is greedy with respect to $r_2(s,a) = \E\bigl[\tilde{R}(s_0, a, J_1)\bigr]$, $f_2(s,a,\boldsymbol{w}) = s_T(a, \boldsymbol{\pi})$, and value function $V$. We can thus apply Lemma \ref{lemma:loss_bound} with $\epsilon_r = \epsilon_r(\boldsymbol{\pi}^*, J_1)$, $d=d(\alpha, d_\mathcal{Y}, T)$, $L_1 = L_U$, $\epsilon_U = \bigl\|\bar{U}^* - V^*(J_1, \boldsymbol{\pi})\bigr\|_\infty$, and $\tilde{\epsilon}_U = \bigl\|V - V^*(J_1,\tilde{\boldsymbol{\pi}})\bigr\|_\infty$. Collecting terms completes the proof.

\subsection{Proof of Theorem~\ref{thm:fsvi_regret}}
\label{sec:proof_thm:fsvi_regret}

We apply Lemma \ref{lemma:main} with $\boldsymbol{\pi} = \tilde{\boldsymbol{\pi}}^*$, $J_1 = J_1^*$, and $V = V^{k}(J_1^*,\boldsymbol{\pi}^*)$. The result follows by combining it with the result of Lemma \ref{prop:K2} and noting that by Proposition \ref{prop:cumreward_diff}, $\epsilon_r(\boldsymbol{\pi}^*, J^*_{1}) \le \epsilon_r(\gamma,\alpha,d_\mathcal{Y},\mathbf{L},T)$.

\subsection{Proof of Proposition~\ref{prop:standard_VI_error}}
\label{proof:standard_VI_error}
Since $\nu_k$ is greedy with respect to $U_k$, we can apply Lemmas~\ref{lemma:Uk_Uopt} and \ref{lemma:Ugreedy_Uopt} to obtain
\begin{equation*}
    \|U^{\nu_{k}} - U^*\|_\infty \leq \frac{2\gamma}{1-\gamma} \frac{ \gamma^k r_\textnormal{max}}{1-\gamma} = \frac{2 r_\textnormal{max} \gamma^{k+1}}{(1-\gamma)^2}.
\end{equation*}

\section{Bounds on $L_U$ in Terms of $L_r$ and $L_f$}
\label{appendix:bounds_on_LU}

We start with an assumption that, if true, leads to a simple bound on the Lipschitz constant $L_U$. The main result is in Proposition \ref{prop.lipschitz_value}.
\begin{assumption}
\label{assumption:smallLf}
Suppose that $\gamma L_f < 1$, where the constant $L_f$, as defined in \eqref{eq.lipschitz_f}, is the sensitivity of the transition function to small changes in $(s,a)$.
\end{assumption}

\begin{lemma} \label{lemma.lipschitz_U_to_Q}
    Consider an $(\alpha, d_\mathcal{Y})$-fast-slow MDP $\langle \mathcal{S}, \mathcal{A}, \mathcal{W}, f, r, \gamma \rangle$ and let $U:\mathcal{S} \rightarrow \R$ be a value function such that there exists $L_U > 0$ where for any states $s$ and $\tilde{s}$,
    \begin{equation}
        |U(s) - U(\tilde{s})| \leq L_U \, \|s - \tilde{s}\|_2.
        \label{eq:lipschitz_U_to_Q_assumption}
    \end{equation}
    Define the state-action value function $Q(s,a) = r(s,a) + \gamma \, \E\bigl[U(f(s,a,w))\bigr]$. Then, for any state-action pairs $(s,a)$ and $(\tilde{s},\tilde{a})$, the state-action value function $Q$ satisfies
    \begin{align*}
        \bigl| Q(s,a) - Q(\tilde{s},\tilde{a}) \bigr|
        \leq (L_r + \gamma L_U L_f) \Bigl(\|s-\tilde{s}\|_2 + \|a-\tilde{a}\|_2 \Bigr).
    \end{align*}
\end{lemma}

\begin{proof}
    For any state-action pairs $(s,a),(\tilde{s},\tilde{a}) \in \mathcal S \times \mathcal A$, we have
    \begin{align}
        \bigl| Q(s,a) - Q(\tilde{s},\tilde{a}) \bigr| %
        &\leq | r(s,a) - r(\tilde{s},\tilde{a})| + \gamma \, \bigl| \E\bigl[U(f(s,a,w)) - U(f(\tilde{s},\tilde{a}, w))\bigr] \bigr| \nonumber\\
        &\leq L_r \bigl(\|s-\tilde{s}\|_2 + \|a-\tilde{a}\|_2\bigr) + \gamma L_U \max_{w} \bigl\|f(s,a, w) - f(\tilde{s},\tilde{a}, w))\bigr\|_2  \label{eq.lipschitz_value_1}\\
        &\leq L_r \bigl(\|s-\tilde{s}\|_2 +\|a-\tilde{a}\|_2\bigr) + \gamma L_U L_f \bigl( \|s - \tilde{s}\|_2 + \|a-\tilde{a}\|_2\bigr)  \label{eq.lipschitz_value_2}\\
        &\leq (L_r + \gamma L_U L_f) \bigl( \|s-\tilde{s}\|_2 + \|a-\tilde{a}\|_2\bigr), \nonumber
    \end{align}
    where \eqref{eq.lipschitz_value_1} follows by \eqref{eq:lipschitz_U_to_Q_assumption} and \eqref{eq.lipschitz_value_2} follows by the definition of $L_f$ in \eqref{eq.lipschitz_f}.
\end{proof}

\begin{lemma} \label{lemma.lipschitz_Q_to_U}
    Consider an $(\alpha, d_\mathcal{Y})$-fast-slow MDP $\langle \mathcal{S}, \mathcal{A}, \mathcal{W}, f, r, \gamma \rangle$. Let $Q:\mathcal{S} \times \mathcal A \rightarrow \R$ be a state-action value function and assume there exists $L_Q > 0$ where for any states $(s, a)$ and $(\tilde{s}, \tilde{a})$,
    \begin{align}
        \bigl| Q(s,a) - Q(\tilde{s},\tilde{a}) \bigr|
        \leq L_Q \Bigl(\|s-\tilde{s}\|_2 +\|a-\tilde{a}\|_2 \Bigr).
        \label{eq:lipschitz_Q_to_U_assumption}
    \end{align}
    Define $U(s) = \max_a Q(s,a)$. Then, for any states $s$ and $\tilde{s}$, the value function $U$ satisfies
    \begin{align*}
        \bigl| U(s) - U(\tilde{s}) \bigr|
        \leq L_Q \, \|s-\tilde{s}\|_2.
    \end{align*}
\end{lemma}

\begin{proof}
Note that:
    \begin{align*}
        \bigl| U(s) - U(\tilde{s}) \bigr| &= \bigl| \max_a \, Q(s,a) - \max_{\tilde{a}}\, Q(\tilde{s},\tilde{a}) \bigr| \\
        &\leq \max_a \, \bigl| Q(s,a) - Q(\tilde{s},a) \bigr| \\
        &\leq L_Q \|s-\tilde{s}\|_2,
    \end{align*}
    where the last inequality is by \eqref{eq:lipschitz_Q_to_U_assumption}.
\end{proof}

\begin{lemma} \label{lemma.lipschitz_Q_U_recursive}
Consider an $(\alpha, d_\mathcal{Y})$-fast-slow MDP $\langle \mathcal{S}, \mathcal{A}, \mathcal{W}, f, r, \gamma \rangle$. Starting with $U_0=0$, recursively define $Q_{k+1}$ and $U_{k+1}$ as follows:
    \begin{equation*}
        Q_{k+1}(s,a) = r(s,a) + \gamma \,\E\bigl[U_k(f(s,a,w)\bigr] \quad \text{and} \quad U_{k+1}(s) = \max_a \, Q_{k+1}(s,a).
    \end{equation*}
    Then $U_k$ is Lipschitz continuous and its Lipschitz constant $L_{U_k}$ satisfies
    \begin{equation}
        L_{U_k} = L_r + \gamma L_f L_{U_{k-1}}.
        \label{eq:lipschitz_Q_U_recursive_result}
    \end{equation}
\end{lemma}

\begin{proof}
    The proof is by induction. For $k=1$, note that $Q_1(s,a) = r(s,a)$ and therefore has Lipschitz constant $L_r$
    by \eqref{assumption.lipschitz_reward}. By Lemma~\ref{lemma.lipschitz_Q_to_U}, it follows that $U_1$ also has Lipschitz constant $L_r$. Since $L_{U_0} = 0$, we see that $L_{U_1} = L_r$ satisfies \eqref{eq:lipschitz_Q_U_recursive_result}.
    Now, assume that $L_{U_k}$ satisfies \eqref{eq:lipschitz_Q_U_recursive_result} for $k \ge 1$.
    Then, by Lemma~\ref{lemma.lipschitz_U_to_Q}, $Q_{k+1}$ is $(L_r + \gamma L_f L_{U_{k}})$-Lipschitz continuous and by Lemma~\ref{lemma.lipschitz_Q_to_U}, $U_{k+1}$ is $(L_r + \gamma L_f L_{U_{k}})$-Lipschitz continuous.
\end{proof}

\begin{proposition} \label{prop.lipschitz_value}
    Consider an $(\alpha, d_\mathcal{Y})$-fast-slow MDP $\langle \mathcal{S}, \mathcal{A}, \mathcal{W}, f, r, \gamma \rangle$ and suppose Assumption \ref{assumption:smallLf} holds. Then, the optimal value $U^*$, as defined in \eqref{eq.Bellman_original}, satisfies:
    \begin{align*}
        \bigl |U^*(s) - U^*(\tilde{s}) \bigr| 
        \leq \frac{L_r}{1 - \gamma L_f} \, \bigl\|s-\tilde{s}\bigr\|_2
    \end{align*}
    for any states $s, \tilde{s} \in \mathcal S$.
\end{proposition}

\begin{proof} According to Proposition~7.3.1 of \citet{bertsekas2012dynamic}, the value $U_k$ in Lemma~\ref{lemma.lipschitz_Q_U_recursive} converges to the optimal value $U^*$ (value iteration). The recursion \eqref{eq:lipschitz_Q_U_recursive_result} can be written as:
\[
L_{U_k} = L_r + \gamma L_f L_r + \cdots +  (\gamma L_f)^{k-1} L_r =  \sum_{i=0}^{k-1} (\gamma L_f)^i L_r,
\]
a convergent sequence since Assumption \ref{assumption:smallLf} is satisfied. Letting $k\rightarrow \infty$, we see that $U^*$ has Lipschitz constant 
\[
\lim_{k\rightarrow\infty} L_{U_k} = \sum_{i=0}^{\infty} (\gamma L_f)^i L_r = \frac{L_r}{1 - \gamma L_f},
\]
completing the proof.
\end{proof}

\section{Proofs for Appendix~\ref{sec:avi}}

\subsection{Technical Lemmas}
\begin{lemma} \label{lemma.Phi_Phi_dagger}
    For any vectors $J \in \mathbb{R}^N$ and $J' \in \mathbb{R}^N$, it holds that
    \begin{equation*}
        \bigl \| (\Phi \Phi^\dagger)(J) - (\Phi \Phi^\dagger)(J') \bigr \|_\infty \le \kappa \, \| J - J' \|_\infty %
    \end{equation*}
\end{lemma}
\begin{proof}
For simplicity, let $D = \Phi \bigl[ \Phi^\dagger(J) - \Phi^\dagger(J')\bigr]$ be the term inside the norm on the left hand side. Then, for any state $s$, we have $|D(s)| = \boldsymbol{\phi}^\intercal(s) \bigl[ \Phi^\dagger(J) - \Phi^\dagger(J')\bigr]$.
    We select $\theta_1(s), \theta_1(s), \ldots, \theta_M(s) \in \R$ that satisfy Assumption~\ref{assumption.basis_function}, obtaining
    \begin{align*}
        |D(s)| &= \Biggl| \kappa \Biggl( \sum_{m=1}^M \theta_m(s) \boldsymbol{\phi}^\intercal(s_m) \Biggr) \bigl(\Phi^\dagger(J) - \Phi^\dagger(J')\bigr) \Biggr| \\
        &\leq \kappa \max_m \, \bigl| \boldsymbol{\phi}^\intercal(s_m) \bigl(\Phi^\dagger(J)) - \Phi^\dagger(J')\bigr) \bigr| \\
        &= \kappa \max_m  \, |D(s_m)| \\
        &= \kappa \max_m \, |J(s_m) - J'(s_m)| \\
        &\leq \kappa \, \| J - J' \|_\infty %
    \end{align*}
    where the third equality uses the fact that $s_m$ is a pre-selected state.
\end{proof}

\begin{lemma}
\label{lemma.expansion_lower}
Given a lower-level value function $\hat{J}(\boldsymbol{\omega}_{t+1})$, recall that one approximate Bellman step in the lower level of FSAVI yields $\hat{J}(\boldsymbol{\omega}_{t}) = \Phi \Phi^{\dagger} \bar{H} \hat{J}(\boldsymbol{\omega}_{t+1})$ in the value function space.
If $\boldsymbol{\omega}_{T} = \mathbf{0}$,
\[
\bigl \|\hat{J}(\boldsymbol{\omega}_{1}) \bigr\|_\infty \le \kappa r_\textnormal{max} \sum_{i=0}^{T-1} (\kappa \gamma)^i = \frac{(\kappa\gamma)^T - 1}{\kappa\gamma - 1} \, \kappa r_\textnormal{max}.
\]
Moreover, the upper-level reward function can be bounded as follows:
\[
\bigl|\E\bigl[\tilde{R}(s_0, a, \hat{J}_1(\boldsymbol{\omega}_1))\bigr]\bigr| \le \frac{(\kappa\gamma)^{T+1} - 1}{\kappa\gamma - 1} \, \kappa r_\textnormal{max}.
\]
\end{lemma}
\begin{proof}
The proof follows by Assumption \ref{assumption.basis_function} and some manipulation:
\begin{align*}
\bigl|\hat{J}(\boldsymbol{\omega}_{t})(s) \bigr|
&= \Biggl| \kappa \Biggl( \sum_{m=1}^M \theta_m(s) \boldsymbol{\phi}^\intercal(s_m) \Biggr) \bigl(\Phi^{\dagger} \bar{H} \hat{J}(\boldsymbol{\omega}_{t+1})\bigr) \Biggr| \\
&\le \kappa \max_m \,  \bigl| \boldsymbol{\phi}^\intercal(s_m) \bigl(\Phi^{\dagger} \bar{H} \hat{J}(\boldsymbol{\omega}_{t+1}) \bigr) \bigr|\\
&= \kappa \, \max_m \, \bigl|
 \bigl(\bar{H}\hat{J}(\boldsymbol{\omega}_{t})\bigr)(s_m) \bigr|\\
&= \kappa \, r_\textnormal{max} + \kappa \gamma \, \bigl \| \hat{J}(\boldsymbol{\omega}_{t+1}) \bigr \|_\infty.
\end{align*}
Applying the above inequality $T-1$ times yields the first result. Next, we see that
\begin{align*}
\bigl|\E\bigl[\tilde{R}(s_0, a, \hat{J}_1(\boldsymbol{\omega}_1))\bigr]\bigr| &\le r_\textnormal{max} + \gamma \bigl \|\hat{J}(\boldsymbol{\omega}_{1}) \bigr\|_\infty\\
&\le \kappa r_\textnormal{max} + \kappa \gamma \bigl \|\hat{J}(\boldsymbol{\omega}_{1}) \bigr\|_\infty\\
&\le \kappa r_\textnormal{max} \sum_{i=0}^{T} (\kappa \gamma)^i,
\end{align*}
where we used the fact that $\kappa \ge 1$ in the second inequality.
\end{proof}

\begin{lemma} \label{lemma.lower_w_bellman}
    Suppose $(\boldsymbol{\omega}_1^*, \boldsymbol{\omega}_2^*, \ldots, \boldsymbol{\omega}_{T}^*)$ satisfies $\boldsymbol{\omega}_t^* = \bar{H}' \boldsymbol{\omega}_{t+1}^*$ for all $t$. Then, we have
    \begin{equation*}
        \bigl \|J_t^* - \hat{J}_t(\boldsymbol{\omega}_t^*) \bigr\|_\infty
        \leq \biggl(1 + \frac{\gamma+1}{\gamma} \sum_{i=1}^{T-t}(\kappa\gamma)^i\biggr) \varepsilon_{\textnormal{low}} = \biggl(\frac{1+\kappa}{1-\kappa\gamma} - \frac{(\kappa\gamma)^T(1+\gamma)}{\gamma-\kappa\gamma^2}\biggr) \varepsilon_{\textnormal{low}}
    \end{equation*}
    a bound on the error of the value function approximation.
\end{lemma}
\begin{proof}
Let $\varepsilon' = \varepsilon_{\textnormal{low}} + \delta$ for some $\delta>0$. For each $t$, choose a parameter vector $\bar{\boldsymbol{\omega}}_t \in \R^M$ such that $\|J_t^* - \hat{J}_t(\bar{\boldsymbol{\omega}}_t)\|_\infty < \varepsilon'$, which is possible by the definition of $\varepsilon_{\textnormal{low}}$. Then, it holds that
    \begin{align}
        \bigl \|\hat{J}_t(\bar{\boldsymbol{\omega}}_t) - \Phi \bar{H}' (\bar{\boldsymbol{\omega}}_{t+1}) \bigr\|_\infty
        &= \bigl \|\Phi \bar{\boldsymbol{\omega}}_t - \Phi \Phi^{\dagger} \bar{H} \hat{J}_{t+1}(\bar{\boldsymbol{\omega}}_{t+1})\bigr\|_\infty \nonumber\\
        &= \bigl \|\Phi \Phi^{\dagger} \Phi \bar{\boldsymbol{\omega}}_t - \Phi \Phi^{\dagger} \bar{H} \hat{J}_{t+1}(\bar{\boldsymbol{\omega}}_{t+1}) \bigr\|_\infty \nonumber\\
        &\leq \kappa \bigl \|\Phi \bar{\boldsymbol{\omega}}_t - \bar{H} \hat{J}_{t+1}(\bar{\boldsymbol{\omega}}_{t+1}) \bigr\|_\infty \label{eq.apply_phi_phi_dagger_lower}\\
        &= \kappa \bigl \|\hat{J}_t(\bar{\boldsymbol{\omega}}_t) - \bar{H} \hat{J}_{t+1}(\bar{\boldsymbol{\omega}}_{t+1}) \bigr\|_\infty \nonumber\\
        &\leq\kappa \Bigl( \bigl \|\hat{J}_t(\bar{\boldsymbol{\omega}}_t) - J_t^* \bigr\|_\infty + \bigl\|J_t^* - \bar{H} \hat{J}_{t+1}(\bar{\boldsymbol{\omega}}_{t+1})\bigr\|_\infty \Bigr) \nonumber\\
        &< \kappa\Bigl(\varepsilon' + \|\bar{H} J_{t+1}^* - \bar{H} \hat{J}_{t+1}(\bar{\boldsymbol{\omega}}_{t+1})\|_\infty \Bigr)\nonumber\\
        &\leq \kappa\Bigl( \varepsilon' + \gamma \, \bigl\|J_{t+1}^* - \hat{J}_{t+1}(\bar{\boldsymbol{\omega}}_{t+1})\bigr\|_\infty \Bigr) \label{eq.apply_contraction}\\
        &< \kappa (\gamma+1)\, \varepsilon'. \nonumber
    \end{align}
    where \eqref{eq.apply_phi_phi_dagger_lower} is by Lemma~\ref{lemma.Phi_Phi_dagger} and \eqref{eq.apply_contraction} follows by the contraction property of $\bar{H}$. The next step is the quantify the difference between $\hat{J}_t(\bar{\boldsymbol{\omega}}_t)$ and $\hat{J}_t(\boldsymbol{\omega}_t^*)$. Let $\varepsilon'' = \kappa (\gamma+1)\, \varepsilon'$.
    \begin{align}
        \bigl \|\hat{J}_t(\bar{\boldsymbol{\omega}}_t) - \hat{J}_t(\boldsymbol{\omega}_t^*) \bigr\|_\infty
        &\leq \bigl \|\hat{J}_t(\bar{\boldsymbol{\omega}}_t) - \Phi \bar{H}' (\bar{\boldsymbol{\omega}}_{t+1}) \bigr\|_\infty + \bigl \|\Phi \bar{H}' (\bar{\boldsymbol{\omega}}_{t+1}) - \hat{J}_t(\boldsymbol{\omega}_t^*) \bigr\|_\infty\nonumber\\
        &\le \varepsilon'' + \bigl \|\Phi \Phi^{\dagger} \bar{H} \hat{J}_{t+1}(\bar{\boldsymbol{\omega}}_{t+1}) - \Phi \Phi^{\dagger} \bar{H} \hat{J}_{t+1}(\boldsymbol{\omega}_{t+1}^*) \bigr\|_\infty\nonumber\\
        &\leq \varepsilon'' + \frac{\gamma'}{\gamma} \bigl \|\bar{H} \hat{J}_{t+1}(\bar{\boldsymbol{\omega}}_{t+1}) - \bar{H} \hat{J}_{t+1}(\boldsymbol{\omega}_{t+1}^*)\bigr\|_\infty\label{eq.apply_phi_phi_dagger_lower2}\\
        &\leq \varepsilon'' + \gamma' \bigl \|\hat{J}_{t+1}(\bar{\boldsymbol{\omega}}_{t+1}) - \hat{J}_{t+1}(\boldsymbol{\omega}_{t+1}^*) \bigr\|_\infty\label{eq.apply_contraction2}\\
        &\le \varepsilon'' + \gamma' \Bigl( \epsilon'' + \gamma' \bigl \|\hat{J}_{t+2}(\bar{\boldsymbol{\omega}}_{t+2}) - \hat{J}_{t+2}(\boldsymbol{\omega}_{t+2}^*) \bigr\|_\infty \Bigr) \nonumber \\
        &\le \cdots \nonumber \\
        &\le \varepsilon'' \sum_{i=0}^{T-t-1}(\gamma')^i + (\gamma')^{T-t} \,  \bigl \|\hat{J}_{T}(\bar{\boldsymbol{\omega}}_{T}) - \hat{J}_{T}(\boldsymbol{\omega}_{T}^*) \bigr\|_\infty \nonumber\\
        &= \frac{\gamma+1}{\gamma} \varepsilon' \, \sum_{i=1}^{T-t}(\gamma')^i,\label{eq.apply_terminal_zero}
    \end{align}
    where \eqref{eq.apply_phi_phi_dagger_lower2} is by Lemma~\ref{lemma.Phi_Phi_dagger}, \eqref{eq.apply_contraction2} is by the contraction property of $\bar{H}$, and \eqref{eq.apply_terminal_zero}
    because $\boldsymbol{\omega}_{T}^* = \bar{\boldsymbol{\omega}}_T = \boldsymbol{0}$ (since $J_T(s) = 0$ for all $s$). Therefore,
    \begin{align*}
        \bigl \|J_t^* - \hat{J}_t(\boldsymbol{\omega}_t^*) \bigr\|_\infty
        &\leq \bigl \|J_t^* - \hat{J}_t(\bar{\boldsymbol{\omega}}_t) \bigr\|_\infty + \bigl \|\hat{J}_t(\bar{\boldsymbol{\omega}}_t) - \hat{J}_t(\boldsymbol{\omega}_t^*)\bigr\|_\infty\\
        &\leq \varepsilon' \biggl ( 1 + \frac{\gamma+1}{\gamma} \sum_{i=1}^{T-t}(\gamma')^i \biggr)\\
        &= \varepsilon' \biggl ( 1 + \frac{\gamma+1}{\gamma} \sum_{i=1}^{T-t}(\kappa \gamma)^i \biggr)\\
        &= \varepsilon' \biggl ( 1 + \frac{\gamma+1}{\gamma} \frac{\kappa \gamma - (\kappa \gamma)^T}{1-\kappa \gamma}\biggr)\\
        &= \varepsilon' \biggl ( \frac{1+\kappa}{1-\kappa\gamma} - \frac{(\kappa\gamma)^T(1+\gamma)}{\gamma-\kappa\gamma^2}\biggr).
    \end{align*}
    Since $\delta$ can be arbitrarily small, the proof is complete.
\end{proof}

\begin{lemma}
\label{lemma.expansion_upper}
Given an upper-level value function $\hat{V}(\boldsymbol{\beta}_{i})$, recall that one approximate Bellman step in the upper level of FSAVI yields $\hat{V}(\boldsymbol{\beta}_{i+1}) = \Phi \Phi^{\dagger} F_{\boldsymbol{\omega}^*} \hat{V}(\boldsymbol{\beta}_{i})$ in the value function space. We have
\[
\bigl \|\hat{V}(\boldsymbol{\beta}^*) \bigr\| \le \frac{\kappa^2 - \kappa^2 (\kappa\gamma)^{T+1}}{(1-\kappa \gamma^T) (1-\kappa\gamma)} \, r_\textnormal{max},
\]
where $\boldsymbol{\beta}^*$ is a fixed point of $F'_{\boldsymbol{\omega}^*}$.
\end{lemma}
\begin{proof}
Again, the proof follows by Assumption \ref{assumption.basis_function} and some manipulation:
\begin{align*}
\bigl|\hat{V}(\boldsymbol{\beta}_{i+1})(s) \bigr|
&= \Biggl| \kappa \Biggl( \sum_{m=1}^M \theta_m(s) \boldsymbol{\phi}^\intercal(s_m) \Biggr) \bigl(\Phi^{\dagger} F_{\boldsymbol{\omega}} \hat{V}(\boldsymbol{\beta}_{i})\bigr) \Biggr| \\
&\le \kappa \max_m \,  \bigl| \boldsymbol{\phi}^\intercal(s_m) \bigl(\Phi^{\dagger} F_{\boldsymbol{\omega}} \hat{V}(\boldsymbol{\beta}_{i}) \bigr) \bigr|\\
&= \kappa \, \max_m \, \bigl|
 \bigl(F_{\boldsymbol{\omega}} \hat{V}(\boldsymbol{\beta}_{i})\bigr)(s_m) \bigr|\\
&\le \kappa \, R_\textnormal{max} + \kappa \gamma^T \, \bigl \| \hat{V}(\boldsymbol{\beta}_{i}) \bigr \|_\infty,
\end{align*}
where $R_\textnormal{max}$ is an upper bound on $\bigl|\E\bigl[\tilde{R}(s_0, a, \hat{J}_1(\boldsymbol{\omega}_1))\bigr]\bigr|$ from Lemma \ref{lemma.expansion_lower}.
Starting with $\boldsymbol{\beta}_{i} = \mathbf{0}$ and applying the inequality repeatedly, we see that
\[
\bigl\|\hat{V}(\boldsymbol{\beta}^*)\bigr\|_\infty \le \kappa R_\textnormal{max} \sum_{j=0}^\infty (\kappa \gamma^T)^j \le \frac{\kappa R_\textnormal{max}}{1-\kappa \gamma^T},
\]
which completes the proof.
\end{proof}

\begin{lemma} \label{lemma.upper_Hprime_contraction}
    For any $\boldsymbol{\omega}$, the parameter space Bellman operator for the upper-level problem $F'_{\boldsymbol{\omega}} = \Phi^{\dagger} \circ F_{\boldsymbol{\omega}} \circ \Phi$ is a $\gamma'$-contraction with respect to a norm $\|\cdot\|_\Phi$ on $\R^M$ defined by $\|\boldsymbol{\beta}\|_\Phi = \| \Phi \boldsymbol{\beta} \|_\infty$, i.e.,
    \begin{equation*}
        \bigl \| F'_{\boldsymbol{\omega}}(\boldsymbol{\beta}) - F'_{\boldsymbol{\omega}}(\boldsymbol{\beta}') \bigr \|_\Phi \leq \kappa \gamma^{T}  \bigl \| \boldsymbol{\beta} - \boldsymbol{\beta}' \bigr \|_\Phi,
    \end{equation*}
    where $\boldsymbol{\beta}, \boldsymbol{\beta}' \in \mathbb{R}^M$. Therefore, there exists a fixed point $\boldsymbol{\beta}^*$ of $F_{\boldsymbol{\omega}}'$.
\end{lemma}
\begin{proof}The proof follows Theorem 3a of \cite{tsitsiklis1996feature}. We include the steps here in our notation for completeness:
    \begin{align*}
        \bigl \| F'_{\boldsymbol{\omega}}(\boldsymbol{\beta}) - F'_{\boldsymbol{\omega}}(\boldsymbol{\beta}') \bigr \|_\Phi
        &= \bigl \| (\Phi^{\dagger} \circ F_{\boldsymbol{\omega}} \circ \Phi) (\boldsymbol{\beta}) - (\Phi^{\dagger} \circ F_{\boldsymbol{\omega}} \circ \Phi) (\boldsymbol{\beta}') \bigr \|_\Phi \\
        &= \bigl \| \Phi (\Phi^{\dagger} \circ F_{\boldsymbol{\omega}} \circ \Phi) (\boldsymbol{\beta}) - \Phi (\Phi^{\dagger} \circ F_{\boldsymbol{\omega}} \circ \Phi) (\boldsymbol{\beta}') \bigr \|_\infty \\
        &\le \kappa \, \bigl\| F_{\boldsymbol{\omega}}(\Phi \boldsymbol{\beta}) - F_{\boldsymbol{\omega}} (\Phi \boldsymbol{\beta}') \bigr\|_\infty \\
        &\leq \kappa \gamma^{T} \, \bigl \| \Phi \boldsymbol{\beta} - \Phi \boldsymbol{\beta}' \bigr \|_\infty \\
        &= \kappa \gamma^{T} \, \bigl \| \boldsymbol{\beta} - \boldsymbol{\beta}' \bigr \|_\Phi,
    \end{align*}
    where first inequality follows by Lemma \ref{lemma.Phi_Phi_dagger} and the second inequality follows by the $\gamma^T$-contraction property of $F_{\boldsymbol{\omega}}$.
\end{proof}

\begin{lemma}
\label{lemma:value_iteration_avi}
Let $\boldsymbol{\omega}^*$ be the solution of the lower level of FSAVI and let $\boldsymbol{\beta}^*$ be the fixed point of $F_{\boldsymbol{\omega}^*}'$. The approximate value iteration of the upper level of FSAVI, which produces $\boldsymbol{\beta}_k$, has a ``value iteration'' error of:
\[
 \bigl\|\hat{V}(\boldsymbol{\beta}_k) - \hat{V}(\boldsymbol{\beta}^*)  \bigr\|_\infty \le (\kappa \gamma^T)^k \frac{\kappa^2 - \kappa^2 (\kappa\gamma)^{T+1}}{(1-\kappa \gamma^T) (1-\kappa\gamma)} \, r_\textnormal{max}. 
\]
\end{lemma}
\begin{proof}
We have:
\begin{align*}
    \bigl\|\hat{V}(\boldsymbol{\beta}_k) - \hat{V}(\boldsymbol{\beta}^*)  \bigr\|_\infty 
    &= \bigl\|\Phi F_{\boldsymbol{\omega}^*}' \boldsymbol{\beta}_{k-1} - \Phi F_{\boldsymbol{\omega}^*}' \boldsymbol{\beta}^* \bigr\|_\infty\\
    &= \bigl\| F_{\boldsymbol{\omega}^*}' \boldsymbol{\beta}_{k-1} - F_{\boldsymbol{\omega}^*}' \boldsymbol{\beta}^* \bigr\|_\Phi\\
    &\le \kappa \gamma^T \bigl \| \Phi \boldsymbol{\beta}_{k-1} - \Phi \boldsymbol{\beta}^*  \bigr\|_\infty\\
    &\le \kappa \gamma^T \bigl \| \hat{V}(\boldsymbol{\beta}_{k-1}) - \hat{V}(\boldsymbol{\beta}^*)\bigr\|_\infty\\
    &\le (\kappa \gamma^T)^k \bigl \| \hat{V}(\boldsymbol{\beta}_{0}) - \hat{V}(\boldsymbol{\beta}^*)\bigr\|_\infty\\
    &\le (\kappa \gamma^T)^k \frac{\kappa^2 - \kappa^2 (\kappa\gamma)^{T+1}}{(1-\kappa \gamma^T) (1-\kappa\gamma)} \, r_\textnormal{max}.
\end{align*}
The first inequality is by Lemma \ref{lemma.upper_Hprime_contraction} and the last inequality follows from $\boldsymbol{\beta}_0=0$ and Lemma \ref{lemma.expansion_upper}.
\end{proof}

\begin{lemma} \label{lemma.upper_w_bellman}
    Consider any $\boldsymbol{\omega}$. If $\boldsymbol{\beta}^*$ is the fixed point of $F'_{\boldsymbol{\omega}}$, i.e., $\boldsymbol{\beta}^* = F'_{\boldsymbol{\omega}} \, \boldsymbol{\beta}^*$, which exists by Lemma \ref{lemma.upper_Hprime_contraction}, then it holds that
    \begin{equation*}
        \bigl \| V^*_{\boldsymbol{\omega}} - \hat{V}(\boldsymbol{\beta}^*) \bigr\|_\infty \leq \biggl(\frac{1 + \kappa}{1- \kappa \gamma^T} \biggr)\,\varepsilon_{\textnormal{up}} %
    \end{equation*}
    where $ V^*_{\boldsymbol{\omega}}$ is the fixed point of $F_{\boldsymbol{\omega}}$.
\end{lemma}
\begin{proof}
    Let $\varepsilon' = \varepsilon_{\textnormal{up}} + \delta$ for some $\delta>0$. Choose $\bar{\boldsymbol{\beta}} \in \R^M$ such that $\bigl\|V^*_{\boldsymbol{\omega}} - \hat{V}( \bar{\boldsymbol{\beta}}) \bigr\|_\infty < \varepsilon'$. Then,
    \begin{align}
        \bigl \|\hat{V}(\bar{\boldsymbol{\beta}}) - \Phi F'_{\boldsymbol{\omega}} (\bar{\boldsymbol{\beta}}) \bigr\|_\infty
        &= \bigl \|\Phi \bar{\boldsymbol{\beta}} - \Phi \Phi^{\dagger} F_{\boldsymbol{\omega}} \hat{V}(\bar{\boldsymbol{\beta}}) \bigr\|_\infty \nonumber\\
        &=\bigl \|\Phi \Phi^{\dagger} \Phi \bar{\boldsymbol{\beta}} - \Phi \Phi^{\dagger} F_{\boldsymbol{\omega}} \hat{V}(\bar{\boldsymbol{\beta}}) \bigr\|_\infty \nonumber\\
        &< \kappa \bigl \|\Phi \bar{\boldsymbol{\beta}} - F_{\boldsymbol{\omega}} \hat{V}(\bar{\boldsymbol{\beta}}) \bigr\|_\infty \label{eq.apply_phi_phi_dagger}\\
        &= \kappa \bigl \|\hat{V}(\bar{\boldsymbol{\beta}}) - F_{\boldsymbol{\omega}} \hat{V}(\bar{\boldsymbol{\beta}}) \bigr\|_\infty \nonumber\\
        &\leq \kappa \Bigl( \bigl \|\hat{V}(\bar{\boldsymbol{\beta}}) - V^*_{\boldsymbol{\omega}} \bigr\|_\infty + \bigl \|V^*_{\boldsymbol{\omega}} - F_{\boldsymbol{\omega}} \hat{V}(\bar{\boldsymbol{\beta}}) \bigr\|_\infty  \Bigr)\nonumber\\
        &< \kappa \Bigl( \varepsilon' +  \bigl\|F_{\boldsymbol{\omega}} V^*_{\boldsymbol{\omega}} - F_{\boldsymbol{\omega}} \hat{V}(\bar{\boldsymbol{\beta}})\bigr\|_\infty\Bigr) \nonumber\\
        &< \kappa \, \bigl( \varepsilon' +  \gamma^T \epsilon' \bigr) = \kappa \, \bigl(1 + \gamma^T\bigr) \varepsilon',\label{eq.apply_contraction_upper}
    \end{align}
    where \eqref{eq.apply_phi_phi_dagger} is by Lemma~\ref{lemma.Phi_Phi_dagger} and \eqref{eq.apply_contraction_upper} is by the $\gamma^T$-contraction property of $F_{\boldsymbol{\omega}}$.
    Now, we let $\varepsilon'' = \kappa (1+\gamma^T) \, \varepsilon'$ and see that
    \begin{align*}
        \bigl \|\hat{V}(\bar{\boldsymbol{\beta}}) - \hat{V}(\boldsymbol{\beta}^*) \bigr\|_\infty
        &\leq \bigl\|\hat{V}(\bar{\boldsymbol{\beta}}) - \Phi F'_{\boldsymbol{\omega}} (\bar{\boldsymbol{\beta}}) \bigr\|_\infty + \bigl \|\Phi F'_{\boldsymbol{\omega}} (\bar{\boldsymbol{\beta}}) - \hat{V}(\boldsymbol{\beta}^*) \bigr\|_\infty\\
        &< \epsilon'' + \bigl \|\Phi \Phi^{\dagger} F_{\boldsymbol{\omega}} \hat{V}(\bar{\boldsymbol{\beta}}) - \Phi \Phi^{\dagger} F_{\boldsymbol{\omega}} \hat{V}(\boldsymbol{\beta}^*)\bigr\|_\infty\\
        &< \epsilon'' + \kappa \, \bigl \|F_{\boldsymbol{\omega}} \hat{V}(\bar{\boldsymbol{\beta}}) - F_{\boldsymbol{\omega}} \hat{V}(\boldsymbol{\beta}^*) \bigr\|_\infty\\
        &\leq \epsilon'' + \kappa\gamma^T \bigl\|\hat{V}(\bar{\boldsymbol{\beta}}) - \hat{V}(\boldsymbol{\beta}^*) \bigr\|_\infty.
    \end{align*}
    It thus follows that 
    \begin{equation*}
        \bigl \|\hat{V}(\bar{\boldsymbol{\beta}}) - \hat{V}(\boldsymbol{\beta}^*) \bigr\|_\infty
        \leq \frac{\kappa + \kappa \gamma^T}{1- \kappa \gamma^T} \varepsilon'. 
    \end{equation*}
    Putting the pieces together, we have
    \begin{align*}
        \bigl\|V^*_{\boldsymbol{\omega}} - \hat{V}(\boldsymbol{\beta}^*)\bigr\|_\infty
        &\leq \bigl\|V^*_{\boldsymbol{\omega}} - \hat{V}(\bar{\boldsymbol{\beta}}) \bigr\|_\infty + \|\hat{V}(\bar{\boldsymbol{\beta}}) - \hat{V}(\boldsymbol{\beta}^*)\bigr\|_\infty\\
        &\leq \varepsilon' + \frac{\kappa + \kappa \gamma^T}{1- \kappa \gamma^T} \, \varepsilon' \\
        &\leq  \frac{1 + \kappa}{1- \kappa \gamma^T}\,\varepsilon'.
    \end{align*}
    Since $\delta$ can be arbitrarily small, the proof is complete.
\end{proof}

\subsection{Proof of Theorem \ref{thm:fsavi_regret}}
\label{sec:proof_fsavi_regret}
We apply Lemma \ref{lemma:main} with $\boldsymbol{\pi} = \hat{\boldsymbol{\pi}}_{\boldsymbol{\omega}^*}$, $J_1 = \hat{J}_1(\boldsymbol{\omega}^*)$, and $V = \hat{V}(\boldsymbol{\beta}_k)$.
First, to compute the reward error $\epsilon_r(\boldsymbol{\pi}^*, \hat{J}_1(\boldsymbol{\omega}_1^*))$, we have
\begin{align*}
\epsilon_r(\boldsymbol{\pi}^*, \hat{J}_1(\boldsymbol{\omega}_1^*)) &= \max_{s,a} \bigl|\E [R(s, a, \boldsymbol{\pi}^*)] - \E [\tilde{R}(s, a, \hat{J}_1(\boldsymbol{\omega}_1^*))] \bigr|\\
&\le \epsilon_r(\gamma,\alpha,d_\mathcal{Y},\mathbf{L},T) + \max_{s,a}\, \bigl|\E [\tilde{R}(s, a, J_1^*)] - \E [\tilde{R}(s, a, \hat{J}_1(\boldsymbol{\omega}_1^*))]\bigr|\\
&\le \epsilon_r(\gamma,\alpha,d_\mathcal{Y},\mathbf{L},T) + \gamma \, \bigl \|J_1^* -\hat{J}_1(\boldsymbol{\omega}_1^*) \bigr\|_\infty\\
&\le \epsilon_r(\gamma,\alpha,d_\mathcal{Y},\mathbf{L},T) +  \biggl(\gamma + (\gamma+1) 
\sum_{i=1}^{T-1}(\gamma')^i\biggr) \varepsilon_{\textnormal{low}},
\end{align*}
where the last inequality follows from Lemma \ref{lemma.lower_w_bellman}. The other term to analyze is $\bigl \| V^*_{\boldsymbol{\omega}^*} - \hat{V}(\boldsymbol{\beta}_k) \bigr\|_\infty$, where we remind the reader of our usage of the shorthand notation $V^*_{\boldsymbol{\omega}^*} = V^*(\hat{J}_1(\boldsymbol{\omega}^*),\hat{\boldsymbol{\pi}}_{\boldsymbol{\omega}^*})$:
\begin{align*}
    \bigl \|V^*_{\boldsymbol{\omega}^*} - \hat{V}(\boldsymbol{\beta}^*) \bigr\|_\infty 
    &\le \bigl \| V^*_{\boldsymbol{\omega}^*} - \hat{V}(\boldsymbol{\beta}^*) \bigr\|_\infty + \bigl \| \hat{V}(\boldsymbol{\beta}^*) - \hat{V}(\boldsymbol{\beta}_k) \bigr\|_\infty\\
    &\le \biggl(\frac{1 + \kappa}{1- \kappa \gamma^T} \biggr)\,\varepsilon_{\textnormal{up}} + (\kappa \gamma^T)^k \biggl( \frac{\kappa^2 - \kappa^2 (\kappa\gamma)^{T+1}}{(1-\kappa \gamma^T) (1-\kappa\gamma)} \biggr) r_\textnormal{max},
\end{align*}
which follows by Lemmas \ref{lemma.expansion_upper} and \ref{lemma:value_iteration_avi}. This completes the proof.

\end{document}